%% file: main.tex
\documentclass[numbers,webpdf,imaiai]{ima-authoring-template}%

\input{preamble}

\begin{document}
	
	\DOI{DOI HERE}
	\copyrightyear{2025}
	\vol{00}
	\pubyear{2025}
	\access{Advance Access Publication Date: Day Month Year}
	\appnotes{Paper}
	\copyrightstatement{Published by Oxford University Press on behalf of the Institute of Mathematics and its Applications. All rights reserved.}
	\firstpage{1}
	
	

    	\title[Struggles of Shallow Networks with High Frequencies]{Why Shallow Networks Struggle to Approximate\\ and Learn High Frequencies}
	
	\author{Shijun Zhang*\ORCID{0000-0003-4115-7891}
		\address{\orgdiv{Department of Applied Mathematics}, \orgname{
        Hong Kong Polytechnic University},
			\orgaddress{
				\country{Hong Kong}}}}
                
	\author{Hongkai Zhao\ORCID{0009-0006-0373-7181}
		\address{\orgdiv{Department of Mathematics}, \orgname{Duke University}, \orgaddress{
				\postcode{27708}, \state{NC}, \country{USA}}}}
                
	\author{Yimin Zhong\ORCID{0000-0003-1537-7364}
		\address{\orgdiv{Department of Mathematics and Statistics}, \orgname{Auburn University}, \orgaddress{
				\postcode{36830},
				\state{AL}, 
				\country{USA}}}}
                
	\author{Haomin Zhou\ORCID{0000-0001-7647-2600}
		\address{\orgdiv{School of Mathematics},
			\orgname{Georgia Institute of Technology}, 
			\orgaddress{
				\postcode{30332}, 
				\state{GA}, 
				\country{USA}}}}
	
	\authormark{S. Zhang, H. Zhao, Y. Zhong, and H. Zhou}
	
	\corresp[*]{Corresponding author: 
		\href{mailto:shijun.zhang@polyu.edu.hk}{shijun.zhang@polyu.edu.hk}
	}
	
	\received{Date}{0}{Year}
	\revised{Date}{0}{Year}
	\accepted{Date}{0}{Year}
	
	

    \abstract{In this work, we present a comprehensive study combining mathematical and computational analysis to explain why a two-layer neural network struggles to handle high frequencies in both approximation and learning, especially when machine precision, numerical noise, and computational cost are significant factors in practice. Specifically, we investigate the following fundamental computational issues: (1) the minimal numerical error achievable under finite precision, (2) the computational cost required to attain a given accuracy, and (3) the stability of the method with respect to perturbations. The core of our analysis lies in the conditioning of the representation and its learning dynamics. Explicit answers to these questions are provided, along with supporting numerical evidence.}
	
	\keywords{shallow neural networks; low-pass filter; Gram matrix; generalized Fourier analysis; Radon transform; Rashomon set.\\[5pt]}
	
	\maketitle


\section{Introduction}
Neural networks are now widely used in machine learning, artificial intelligence, and many other areas as a parameterized representation with certain structures for approximating an input-to-output relation, e.g., a function or map in mathematical terms. 
They have achieved notable successes in practice but also encountered significant challenges.
More importantly, many basic and practical questions are still open. Extensive studies have been carried out to understand the properties of neural networks and how they work in different perspectives, such as universal approximation property, representation capacity, and optimization process (mostly based on gradient descent with different variations), often separately. For example, it has been widely known that neural networks can approximate any Lipschitz function with a small error.
The approximation theory has been extensively studied for various types of activation functions and diverse structures of networks~\cite{shijun:4,shijun:5,shijun:net:arc:beyond:width:depth,shijun:arbitrary:error:with:fixed:size,jiao2021deep,yarotsky18a,yarotsky2017,shijun:1,shijun:2,shijun:3,ZHOU2019,10.3389/fams.2018.00014,2019arXiv190501208G,2019arXiv190207896G,suzuki2018adaptivity,Ryumei,Wenjing,Bao2019ApproximationAO}. Recently, there are works~\cite{shijun:2,yarotsky18a,shijun:thesis} discussing the explicit constructions of the ``optimal'' networks of multiple layers.
 However, one daunting issue that has not been studied systematically in the past is whether such ``optimal'' approximations can be possibly attained by training the networks and more importantly, what is the approximation limit in terms of a finite machine precision, computation cost, and the property of the function being approximated? Hence, an effective algorithm needs to consider all these aspects to achieve well-balanced accuracy, efficiency, and stability. Due to the nonlinear nature of neural network representations, this is a challenging task. 

In this work, instead of a mathematical study of approximation theory, which usually does not consider the practical constraint of finite machine precision or the cost and stability of finding a good solution, we study a few basic questions from a practice point of view for both approximation and optimization for two-layer neural networks. Our consideration includes both asymptotic/continuous and non-asymptotic/discrete regimes in terms of network width:
\begin{itemize}
\item the minimal numerical error one can achieve given a finite machine precision;
\item
the computation time (cost) to achieve a certain accuracy for the training process;
\item 
stability to perturbations, e.g., noise in the data, or over-fitting. 
\end{itemize}
Our study shows that, in practice, a shallow neural network is essentially a ``low-pass filter'' due to the ill-conditioning of the representation which is explicitly characterized by the spectral decay of the Gram matrix, composed of pairwise correlation of the parameterized activation functions, and asymptotic equivalence of the eigenmodes to the eigenfunctions of the Laplace operator (generalized Fourier modes) in arbitrary dimensions. 
More specifically, ill-conditioning of the representation means smooth modes, the number of which depends on the spectral decay rate of the Gram matrix and machine precision, can be captured and stably used for approximation. Although the universal approximation property of two-layer neural networks is proved in theory, conditioning of the representation and the finite machine precision determine the achievable numerical accuracy which may be far less than the machine precision in practice, for example, when approximating functions with significant high-frequency components, such as functions with rapid changes and/or fast oscillations. Moreover, the numerical accuracy can not be further improved by increasing the amount of data or the network's width after a certain threshold since the number of eigenmodes that can be stably captured with a given machine precision does not increase. 
On the other hand, the low pass filter nature leads to certain stability with respect to perturbations in the high modes, e.g., noises, or over-parametrization. 

One of the most important features when using neural networks for approximation is the capability of learning, i.e., optimizing the parameters to adapt to the underlying function manifested by data. However,  the initial representation with randomized parameters can not capture those high frequency components needed to represent fine features due to the ill-conditioning and hence can not guide the optimization to achieve the adaptivity effectively. Moreover, we show that ill-conditioning of the representation causes slow learning dynamics for high frequencies, which are needed in an adaptive representation, for gradient-based optimization. Furthermore, the adaptive distribution of parameters can lead to even worse conditioning of the representation and hence even slower learning dynamics. These difficulties make approximation of high frequencies based on learning challenging if not impossible.

From a probabilistic perspective, we show that the Rashomon set, the set of parameters where accurate approximations can be achieved, for a two-layer neural network has a small measure for highly oscillatory functions. The measure decreases exponentially with respect to the oscillation frequency. It implies, in practice, both low probabilities of being close to a good approximation for a random initial guess and high computational cost for finding one. 

These understandings of the limit of a one-hidden layer network prompt us to study how to use multi-layer to circumvent the limit through effective smooth decomposition and composition in our future work.



\subsection{Literature review}
The approximation theory of shallow neural networks has been well-known since the universal approximation theorem~\cite{barron1993, Cybenko1989ApproximationBS,chui1992approximation,park1991universal,klusowski2018approximation,breiman1993hinging}. The error bounds of approximation in $L^{\infty}$ norm have been demonstrated either through explicit construction or by probabilistic proofs, see~\cite{domingo2021tighter,klusowski2018approximation,hornik1989multilayer} and the references therein. However, it has been observed widely in practice that shallow neural networks cannot approximate highly oscillatory functions effectively~\cite{chen2022bridging, grossmann2023can}. Several explanations have been proposed in the past few years. 

The authors of~\cite{luo2019theory,xu2019frequency} summarized such phenomenon as a heuristic law called \emph{frequency principle}, that is, the training process of the neural network recovers lower Fourier frequencies first. Several explanations based on the frequency principle are proposed. When the activation function $\sigma$ is analytic, e.g., $\texttt{Tanh}$ or $\texttt{Sigmoid}$, the authors in~\cite{xu2019frequency} have shown that the training of the network at the final stage can be slow due to the inability of smooth activations to pick up the high-frequency components. While for non-smooth activation functions, e.g., $\texttt{ReLU}$, $\texttt{LReLU}$, and $\texttt{ELU}$, the authors of~\cite{luo2019theory} provided an interpretation for the \emph{ frequency principle} based on the smoothing effect of the activation functions.  However, the theory cannot be applied to general cases since it demands strong regularity assumptions for the activation functions and the objective function.
In high dimensions, ~\cite{eldan2016power} and~\cite{safran2017depth} claimed the bottleneck of the training of shallow networks may come from the so-called ``depth separation'' of the capacity between shallow and deep neural networks and explained that shallow neural networks demand a width that grows exponentially in dimension to fit discontinuous functions in $L^2$ norm while deep neural networks can fit well with a much smaller width.

 Another explanation attributed the difficulty to the configuration of the training dynamics. Most of the literature focuses on two regimes: the Neural Tangent Kernel (NTK) regime and the Mean-Field regime. (1) In the NTK regime, the definition of network slightly differs from the classical one, each layer is scaled by a key factor $\frac{1}{\sqrt{n_l}}$ with $n_l$ being the width. If the network is sufficiently wide~\cite{zhang2019fast,li2020learning}, the parameters of the neural network almost freeze, except for the last layer. This configuration can sometimes be viewed as the final stage of training when the parameters are nearly optimal. Under such circumstances, training dynamics has been extensively investigated~\cite{du2019gradient,li2018learning,jacot2018neural}. With well-distributed labeled data, the slow convergence rate relates to the fast decay rate of the eigenvalues of the neuron-tangent kernel~\cite{cao2019towards,su2019learning,velikanov2021explicit,nitanda2020optimal}. Studies of the eigenvalues of practical kernels have shown that the decay rates of the leading eigenvalues are closely related to the regularity near the diagonal of the kernel~\cite{birman1970asymptotic}, while a fully explicit characterization of all eigenvalues is difficult. (2) In the mean-field regime, the shallow network is commonly used where an additional factor of $\frac{1}{n}$ is applied to the classical one with $n$ being the width. Parameters can be treated as an empirical distribution function or a particle system~\cite{mei2018mean,rotskoff2018trainability,sirignano2020mean}. As the width of the network becomes infinity, the limiting distribution obeys a gradient flow under the Wasserstein metric, and the convergence is proved in~\cite{rotskoff2018trainability} for $C^1$ activation function under the assumption that the empirical measure of particle system converges. A more general class of particle systems has been explored in~\cite{chizat2018global}. However, the corresponding convergence rate is not mentioned. 
 
 In addition to the above possible explanations, initialization of parameters may also play a vital role in understanding the difficulty of training neural networks to fit highly oscillatory functions. A recent work~\cite{holzmuller2020training} considered a special setting for the two-layer \texttt{ReLU} network that the labeled points and biases are not well distributed and proved that the trained network will not converge to the desired objective function.

\subsection{Contributions}

Our main contributions are summarized below.
\begin{itemize}
    \item We explicitly characterize the decay rate of the eigenvalues of the Gram kernel corresponding to \texttt{ReLU} activation functions (and others) in any dimensions and show that the corresponding eigenfunctions are equivalent to generalized Fourier modes. The corresponding discrete Gram matrix is also analyzed. The study implies that the approximation by a two-layer neural network can only maintain a finite number of leading (smooth) modes accurately given a finite machine precision.    
    \item 
    We investigate the nonlinear learning dynamics based on the gradient flow for two-layer \texttt{ReLU} networks with finite width in a bounded domain. We show slow learning dynamics for high-frequency modes. 
    \item The measure of Rashomon set, the set of parameters in parameter space that renders an approximation with a given tolerance, for two-layer neural networks is characterized. The result shows that oscillatory functions are difficult to represent and learn from a probability perspective.
\end{itemize}



In this work, we mainly focus on using $\texttt{ReLU}$ as the activation function. Our study can be extended to other activation functions as shown in the appendix. Here is the outline of this paper. First, we present a spectral analysis of the Gram matrix and least square approximation in Section~\ref{sec:gram:and:ls}. Then we study the training dynamics based on gradient descent in Section~\ref{sec:dynamics}.
In Section~\ref{sec:Rashomon}, the probability framework and Rashomon set are employed to show why oscillatory functions are difficult to represent and learn. Extension of the current work is briefly discussed in Section~\ref{sec:discussions}. 


\section{Gram matrix and least square approximation}
\label{sec:gram:and:ls}

We first introduce some notations and the general setup for two-layer neural networks. Denote $[n]=\{1,2,\cdots,n\}$ and $C(D)$ the continuous functions on a compact domain $D\subseteq\bbR^d$.
Let $\calH_n$ be the hypothesis space generated by shallow feed-forward networks of width $n$. Each $h\in \calH_n$ has the following classical form
\begin{equation}
\label{eq:def:shallow:net}
    h(\bmx) = \sum_{i=1}^n a_i \sigma(\bmw_i \cdot \bmx - b_i)+c\quad \tn{for any $\bmx\in \bbR^d$},
\end{equation}
where $n\in\bbN^+$ is the width of the network, $\bmw_i\in \bbR^d$, $a_i, b_i,c\in\bbR$ are parameters for each $i\in[n]$.
Finding the best $h\in \calH_n$ to approximate the objective function $f(\bmx)\in C(D)$ is usually converted into minimizing the expected (or true) risk
\begin{equation*}
    \calL(h,f)\coloneqq \bbE_{\bmx\sim  \calU(D)}\left[\ell\big(h(\bmx),f(\bmx)\big)\right],
\end{equation*}
where $\calU$ is some data distribution over $D$ and $\ell(\cdot,\cdot)$ is a loss function. In practice, only finitely many samples $\big\{\big(\bmx_i,f(\bmx_i)\big)\big\}_{i=1}^N$ are available and the data distribution is unknown. However, one could approximate the expected risk by  the empirical risk $\calL_{\tn{emp}}(h, f)$, which is given by
\begin{equation*}
    \calL_{\tn{emp}} (h, f) \coloneqq \frac{1}{N}\sum_{i=1}^N \ell\big(h(\bmx_i),f(\bmx_i)\big).
\end{equation*}

In this paper, we let $\calU$ be the uniform distribution and $\ell(y,y^\prime)=|y-y^\prime|^2$, implying
\begin{equation*}
   \calL(h, f) =  \int_{D}| h(\bmx) -  f(\bmx)|^2 d\bmx\quad \tn{and}\quad \calL_{\tn{emp}} (h, f) = \frac{1}{N}\sum_{i=1}^N |h(\bmx_i) - f(\bmx_i)|^2.
\end{equation*}

A learning/training process is to identify $h^{\ast}\in\calH_n$ or $\hath\in\calH_n$ such that
\begin{equation*}
    h^{\ast}\in \argmin_{h\in\calH_n}\calL(h,f)\quad \tn{or}\quad 
    \hath\in \argmin_{h\in\calH_n}\calL_{\tn{emp}}(h,f).
\end{equation*}
This study investigates the approximation capabilities of two-layer neural networks within the mentioned framework. We aim to address the three basic questions outlined in the abstract that commonly arise in practical settings.


We start with a study on approximation properties of a two-layer neural network as a linear representation, i.e., where the weights and biases in the hidden neurons are fixed, using least squares. In this setting, the solution can be found by solving a linear system involving the Gram matrix, the normal equation. The most fundamental question is the basis of the representation, i.e., $\{\sigma(\bmw_i \cdot \bmx - b_i), i\in[n]\}$: space the basis span and the correlation among the basis. Desirable features of a good basis for computational efficiency, accuracy, and stability in practice are sparsity and well-conditioning of the Gram matrix. In other words, global interactions and strong correlations should be avoided. 

We start with the most used activation function in neural networks is \texttt{ReLU}: $\sigma(x) \coloneqq \max(x, 0)$. The general form of a shallow network in \eqref{eq:def:shallow:net} can be simplified to 
\begin{equation}
\label{eq:def:shallow:net:sphere}
    h(\bmx) = \sum_{i=1}^n a_i \sigma(\bmw_i \cdot \bmx - b_i)+c, \quad \bmx\in D\subset\bbR^d,~\bmw_i\in\bbS^{d-1}, ~a_i, b_i\in \bbR.
\end{equation}
\subsection{One-dimensional case}
In one dimension, we let $D = [-1, 1]$. Due to the affinity of \texttt{ReLU} and the transform $\sigma(x) = x - \sigma(-x)$, the 
 form of a shallow network in \eqref{eq:def:shallow:net:sphere} can be further simplified to
\begin{equation*}
    h(x) = c + vx+ \sum_{i=1}^n a_i \sigma(x  - b_i), \quad x\in D,~c, v,a_i, b_i\in\bbR.
\end{equation*}
Since we only consider the approximation inside $D$, we may further reduce the network into $ h(x) = c + \sum_{i=1}^n a_i \sigma(x  - b_i) $ by setting
$c = f(-1)$ and $b_i\in D$ as well.
With fixed $b_i\in D$, \texttt{ReLU} functions $\sigma( x  -  b_i )$ span the same continuous piecewise linear (in each sub-intervals between $b_i$'s) function space as the linear finite element basis (hat functions) with nodes $\{b_i\}_{i=1}^n$. Mathematically, they are the same when used in approximating a function in the domain $D$ in the least square setting. The minimizer is the $L^2$ projection of $f(x)$ onto the continuous piecewise linear function space. However, the major difference in practice is the Gram matrix (mass matrix in finite element terminology or normal matrix in linear algebra terminology) of the basis, which defines the linear system one needs to solve numerically to find the best approximation. Using the finite element basis, which is local and decorrelated, the Gram matrix is sparse and well-conditioned. The condition number is proportional to the ratio between the sizes of the maximal sub-interval and the minimal sub-interval \cite{FEM-conditioning}. While using \texttt{ReLU} functions, which are non-local and can be highly correlated, the Gram matrix is dense and ill-conditioned, as we will show below. As a consequence, 1) computation and memory costs involving the Gram matrix can be very expensive, and 2) only those functions close to the linear space spanned by the leading eigenvectors of the Gram matrix can be approximated well. The number of the leading eigenvectors that can be used stably and accurately depends on the decay rate of the eigenvalues, machine precision, and/or noise level. Now we present a spectral analysis of the Gram matrix for a set of \texttt{ReLU} functions.

Denote the Gram matrix $\bmG \coloneqq (\bmG_{i,j})\in\bbR^{n\times n}$, where
\begin{equation*}
   \bmG_{i,j} \coloneqq  \int_{D} \sigma(x - b_i) \sigma(x - b_j) dx
   =\frac{1}{24}(2 - b_i - b_j - |b_i - b_j|)^2 (2 - b_i - b_j + 2 |b_i-b_j|).
\end{equation*}
The correlation between two \texttt{ReLU} functions with close biases $b_i, b_j$ is $1-O(|b_i-b_j|^2)$ and this strong correlation suggests ill-conditioning of the Gram matrix. To fully understand the spectrum property of $\bmG$, we define the corresponding Gram kernel function $\calG: \bbR\times \bbR\mapsto\bbR$ as
\begin{equation}\label{EQ: GRAM KERNEL}
    \calG(x, y) \coloneqq \int_{D}\sigma( z- x)\sigma(z - y) dz.
\end{equation}
In particular, if we restrict $x, y\in D$, 
\begin{equation}\label{eq:explicit}
\begin{aligned}
        \calG(x, y) &\coloneqq\frac{1}{24}(2 - x - y - |x - y|)^2 (2 - x - y + 2 |x-y|) \\
        &= \frac{1}{12}|x-y|^3 + \frac{1}{12}(2-x-y)\left(2(1-x)(1-y) - (x-y)^2\right).
    \end{aligned}
\end{equation}
First, we provide an explicit spectral characterization of the Gram kernel \eqref{eq:explicit}. Define the operator $K: L^2[-1, 1]\to L^2[-1, 1]$ 
    \begin{equation*}
        K h(x) = \int_{-1}^1 \calG(x, y) h(y) dy. 
    \end{equation*}
Let $\mu_k,\ k=1, 2, \ldots$, be the eigenvalue of $K$ in descending order and $\phi_k$ be the corresponding eigenfunction which satisfies 
    \begin{equation}\label{EQ: EIGEN}
        \int_{-1}^1 \calG(x, y) \phi_k(y) dy = \mu_k \phi_k(x).
    \end{equation}
    Taking derivatives four times on both sides, using $\frac{\partial^4}{\partial x^4}\calG(x,y)=\delta(x-y)$ we obtain the following differential equation, 
    \begin{equation}\label{eq:ODE}
        \phi_k^{(4)}(x) = \frac{1}{\mu_k} \phi_k(x),
    \end{equation}
   which implies that there are constants $A_k, B_k, C_k, D_k \in \bbC$ such that
     \begin{equation*}
         \phi_k(x) = A_k \cosh(w_k x) + B_k \sinh(w_k x) + C_k \cos(w_k x) + D_k \sin(w_k x)
     \end{equation*} 
     for certain $w_k = \mu_k^{-\frac{1}{4}}> 0$. Furthermore, one can easily check that $\calG(1, y)=\calG_x(1, y)=\calG_{xx}(-1, y)=\calG_{xxx}(-1, y)=0$, which implies the following boundary conditions for $\phi_k$, 
\begin{equation*}
    \phi_k(1)=\phi_k'(1)=\phi_k''(-1)=\phi_k'''(-1)=0,
\end{equation*}
which determine $A_k, B_k, C_k, D_k$ explicitly. We show that $\phi_k$ are asymptotically Fourier modes from low frequencies to high frequencies. Here we summarize the results while the detailed calculations and proofs can be found in Appendix~\ref{sec:GF}. 
\begin{itemize}
 \item 
$w_{2j+1}\in ((j+\frac{1}{4})\pi, (j+\frac{1}{2})\pi)$,  $w_{2j+2}\in ((j + \frac{1}{2})\pi, (j+\frac{3}{4})\pi)$, $j\ge 0$, and $\lambda_k=w_k^{-4}\sim (\frac{k\pi}{2})^{-4}$,
\item
if $k=2j+1$, $\phi_k(x) = C_k(-\frac{\cos(w_k)}{\sinh(w_k)} \sinh(w_k x)  +  \cos(w_k x))$, $C_k=\cO(1)$,\\  
if $k=2j+2$, $\phi_k(x) = D_k ( - \frac{\sin(w_k)}{\cosh(w_k)}\cosh(w_k x) + \sin(w_k x) )$, $D_k=\cO(1)$,
\item 
$\{\phi_k\}_{k\ge1}$ forms an orthonormal basis of $L^2(D)$ and
\[
\|\phi_k\|_{L^{\infty}(D)}=\cO(1), ~\|\phi^{'}_k\|_{L^{\infty}(D)}=\cO(k), ~\|\phi^{''}_k\|_{L^{\infty}(D)}=\cO(k^2),
\]
\[
\|\phi_{2j+1}(x)-\cos (w_{2j+1}x)\|_{L^2(D)}=\cO(j^{-1/2}), ~\|\phi_{2j+2}(x)-\sin (w_{2j+2}x)\|_{L^2(D)}=\cO(j^{-1/2}).
\]
\end{itemize}

Next, we study the spectral properties of the discrete Gram matrix. Earlier work \cite{hong2022activation} studied discrete Gram matrix on the uniform grid in one dimension. We will prove results in more general settings and higher dimensions. Denote the vectors $\bma \coloneqq (a_i)_{i=1}^n$ and $\bmf \coloneqq (f_i)_{i=1}^n$ that $f_i  = \int_{D} f(x) \sigma(x-b_i) dx $. The least-square solution $\bma\in \bbR^n$, when the 
 biases are fixed,  is $\bma = \bmG^{\dagger} \bmf$, where $\bmG^{\dagger}$ is the pseudo-inverse of $\bmG$. Without loss of generality, we assume that $b_i\neq b_j$ for all $i\neq j$ and the biases are sorted in ascending order, that is, $b_1<b_2 <\cdots < b_n$. We first provide an estimate for the eigenvalue estimates of the Gram matrix $\bmG$. The rescaled matrix $\bmG_n = \frac{1}{n}\bmG$ is the so-called \emph{kernel matrix} for $\calG$, which plays an important role in kernel methods~\cite{10.2307/3318636}. 
\begin{theorem}\label{thm:spectrum1D}
    Suppose $\{b_i\}_{i=1}^n$ are quasi-evenly spaced on $D$, $b_i = -1 + \frac{2(i -1)}{n}+ o\left(\frac{1}{n}\right)$. Let $\lambda_1 \ge \lambda_2 \ge \cdots \ge \lambda_n \ge 0$ be the eigenvalues of the Gram matrix $\bmG$, then $|\lambda_k - \frac{n}{2}\mu_k |\le C$ for some constant $C = \calO(1)$, where $\mu_k = \Theta(k^{-4})$ is the $k$-th eigenvalue of $\calG$.
\end{theorem}
\begin{proof}
The idea of proof comes from~\cite{widom1958eigenvalues}. Define the operator ${K}^{\ast}$ by the kernel
\begin{equation*}
    {\calG^{\ast}}(x, y) = \calG\left(-1 + \frac{2}{n}\left\lfloor\frac{x+1}{2}\right\rfloor, -1 + \frac{2}{n}\left\lfloor\frac{y+1}{2}\right\rfloor\right)
\end{equation*}
where $\lfloor\cdot \rfloor$ is the floor function. We denote the eigenvalues of $K^{\ast}$ as $\mu^{\ast}_1\ge \mu^{\ast}_2 \ge \cdots$, 
then for the equispaced biases $b_i^{\ast} = -1 + \frac{2(i-1)}{n}$, the corresponding Gram matrix $\bmG^{\ast}$ has the eigenvalues exactly $\lambda_i^{\ast} = \frac{n}{2} \mu^{\ast}_i$. Using Weyl's inequality for self-adjoint compact operators
\begin{equation}\label{EQ: WEYL BOUND}
    |\mu_i - \mu_i^{\ast}|\le \|K - K^{\ast}\| \le \sqrt{\int_{-1}^1\int_{-1}^1 \left|\calG(x, y) - \calG^{\ast}(x, y)\right|^2 dx dy} = \calO(n^{-1}).
\end{equation}
Now we consider perturbed $\tilde{b}_i = b_i^{\ast} + o(\frac{1}{n})$ as mentioned in Theorem~\ref{thm:spectrum1D}, let the corresponding Gram matrix be $\Tilde{\bmG}$, then by Weyl's inequality for Hermitian matrices
\begin{equation*}
    \| \lambda_i(\tilde{\bmG}) - \lambda_i(\bmG^{\ast}) \| \le \|\tilde{\bmG} - \bmG^{\ast}\| \le \sqrt{\sum_{i=1}^n \sum_{j=1}^n \left|\calG(\tilde{b}_i, \tilde{b}_j) - \calG(b_i^{\ast}, b_j^{\ast}) \right|^2} = o(1).
\end{equation*}
Therefore $|\lambda_i(\tilde{\bmG}) - \frac{n}{2} \mu_i| \le C$ for some positive constant $C > 0$.
\end{proof}

\begin{theorem}\label{THM: Cond Num}
    Suppose $\{b_i\}_{i=1}^n$ are chosen as Theorem~\ref{thm:spectrum1D}, then the condition number of the Gram matrix $\bmG$ satisfies $$\kappa = \lambda_1/\lambda_n = \Omega(n^3),$$
    where $\Omega$ is the big Omega notation.
\end{theorem}
\begin{proof}
    The kernel $\cG(x, y)$ permits the expansion
    \begin{equation}
        \cG(x, y) = \sum_{k=1}^{\infty} \mu_{k}\phi_k(x)\phi_k(y).
    \end{equation}
    Let $\cG_m(x, y): = \sum_{k=1}^m \mu_k \phi_k(x)\phi_k(y)$ be the truncated expansion, where $m\ge 1$ is the truncation parameter. Then the Gram matrix can be decomposed into
    \begin{equation}
       \bmG_{i,j} = \cG_m(b_i, b_j) + (\cG(b_i, b_j) - \cG_m(b_i, b_j)).
    \end{equation}
    Define the matrix $\Phi_m\in\bbR^{n\times m}$ with entries $(\Phi_m)_{i,j} = \phi_j(b_i)$, $1\le j\le m$, then 
    \begin{equation}
        \bmG = \Phi_m \Lambda_m \Phi_m^T + \bmE_m, 
    \end{equation} 
    where $(\Lambda_m)_{i,j} = \delta_{i,j} \mu_i$ and $(\bmE_m)_{i,j}= \sum_{k > m} \mu_k \phi_k(b_i) \phi_k(b_j)$. 
    Let $\sigma_{i}(\Phi_m \Lambda_m \Phi_m^T)$ denote the $i$-th eigenvalue of $\Phi_m \Lambda_m \Phi_m^T$, then by Ostrowski's theorem~\cite{horn2012matrix}, 
    \begin{equation}
        \left|\sigma_{i}(\Phi_m \Lambda_m \Phi_m^T) - \frac{n}{2}\mu_i\right| \le |\mu_i| \left\| \Phi_m^T\Phi_m - \frac{n}{2}\mathrm{Id}_m\right\|_{{\rm  op}}, \quad 1\le i\le m \le n.
    \end{equation}
    Therefore, using Weyl's inequality
    \begin{equation}\label{EQ: EIGEN GAP}
    \begin{aligned}
        \left|\lambda_i - \frac{n}{2}\mu_i\right| &\le |\lambda_i - \sigma_{i}(\Phi_m \Lambda_m \Phi_m^T) | + \left| \sigma_{i}(\Phi_m \Lambda_m \Phi_m^T)  - \frac{n}{2}\mu_i\right| \\
        &\le  \|\bmE_m\|_{{\rm  op}} + |\mu_i| \left\| \Phi_m^T\Phi_m - \frac{n}{2}\mathrm{Id}_m\right\|_{{\rm  op}},\quad \quad 1\le i\le m \le n.
    \end{aligned}
    \end{equation}
  Since the eigenfunctions $\phi_k$ are uniformly bounded, see Theorem~\ref{lem:B4} in Appendix~\ref{sec:GF}, then 
  \begin{equation}\label{EQ: Uni Bound}
       \|\bmE_m\|_{{\rm  op}}\le C n \sum_{k > m} \mu_k = \cO\left(\frac{n}{m^3}\right). 
  \end{equation}
The entries in $\Phi_m\Phi_m^T - \frac{n}{2}\mathrm{Id}_m$ can be estimated by the standard numerical quadrature analysis on abscissas $\{b_i\}_{i=1}^n$. Indeed,
    \begin{equation}\label{EQ: Quadrature}
        \frac{2}{n}\sum_{i=1}^n \phi_j(b_i)\phi_k(b_i) - \int_{-1}^1 \phi_j(x) \phi_k(x) dx  = \cO\left(\frac{1}{n} \sup_{[-1, 1]} (\phi_j\phi_k)'\right).
    \end{equation}
    Hence using the estimate $\|\phi'_k\|_{\infty} = \cO(k)$, $\left\| \Phi_m^T\Phi_m - \frac{n}{2}\mathrm{Id}_m\right\|_{{\rm  op}} = \cO(m^2)$.
    Combine the above estimates into~\eqref{EQ: EIGEN GAP} and reuse Theorem~\ref{thm:spectrum1D}, by selecting $m = n^{1/5}i^{4/5}\ge i$, 
    \begin{equation}
         \left|\lambda_i - \frac{n}{2}\mu_i\right| \le C\min \left( 1, \frac{n}{m^3} + \frac{m^2}{i^4}\right) = \begin{cases}
            \cO(1), & i < n^{\frac{1}{6}}, \\
             \cO(n^{2/5} i^{-12/5}), & n^{\frac{1}{6}}\le  i, \le n.
         \end{cases}
    \end{equation}
    It implies that $\lambda_1 = \Theta(n)$ and $\lambda_n = \cO(n^{-2})$, which leads to the condition number estimate $\kappa = \Omega(n^3)$. 
\end{proof}

\begin{remark}
For evenly spaced biases, explicit computations for the eigenvalues in~\cite{hong2022activation} show that the condition number is $\Omega(n^4)$. However, a sharp lower bound for unevenly distributed biases (grid points) is difficult. Here we use 1) Ostrowski’s theorem to relate the eigenvalues between two symmetric positive definite matrices, and 2) random quadrature points for integral estimation. However, we believe that for a fixed number of points in an interval, non-evenly spaced points will result in a larger condition number for the Gram matrix than equally spaced points, which seems also suggested by our numerical tests, see Figures~\ref{fig:spectrum:uniform} and~\ref{fig:spectrum:adaptive} in Section~\ref{sec:numerics}.
\end{remark}


Generally speaking, if $\{b_i\}_{i=1}^n$ are distributed i.i.d with probability density function $\rho: D\mapsto \bbR$, one can reformulate the matrix-vector multiplication as a $\rho$ weighted integral as the continuous limit. The discrete eigen system in the limit corresponds to that of the modified continuous kernel $\calG_{\rho}(x, y) \coloneqq \sqrt{\rho(x) }\calG(x, y) \sqrt{\rho(y)}$. For this case, we have a similar estimate of eigenvalues if $\rho$ is bounded from below and above by positive constants.
\begin{lemma}\label{LEM: A2}
    Suppose $\rho(x)$ is bounded from below and above by  positive constants and define
    \begin{equation*}
        \calG_{\rho}(x, y)\coloneqq \sqrt{\rho(x)}\calG(x, y) \sqrt{\rho(y)}.
    \end{equation*}
    Let $\widetilde{\mu}_k$ be the $k$-th eigenvalue of $\calG_{\rho}$ in descending order, then $\inf_{[-1,1]} \sqrt\rho \le \widetilde{\mu}_k / \mu_k \le \sup_{[-1,1]}\sqrt\rho$. 
\end{lemma}
\begin{proof}
   By Min-Max theorem for the eigenvalues of the integral kernel $\calG_{\rho}$, 
   \begin{equation*}
       \begin{aligned}
           \widetilde{\mu}_k = \max_{S_k}\min_{z\in S_k, \|z\|=1} \int_{-1}^1\int_{-1}^1 \calG_{\rho}(x, y) z(x) z(y) dx d y, \\
            \widetilde{\mu}_k = \min_{S_{k-1}}\max_{z\in S_{k-1}^{\perp}, \|z\|=1} \int_{-1}^1\int_{-1}^1 \calG_{\rho}(x, y) z(x) z(y) dx d y,
       \end{aligned}
   \end{equation*}
   where $S_k$ is a $k$ dimensional subspace of $L^2[-1,1]$. In the first equation, we choose the space $S_k = \text{span}(\frac{\phi_1}{\sqrt\rho}, \cdots, \frac{\phi_k}{\sqrt\rho})$, where $(\mu_j, \phi_j)$ denotes the $j$th eigenpair of the kernel $\calG$. Let 
   \[
   \hat{z}=\argmin_{z\in S_k, \|z\|=1} \int_{-1}^1\int_{-1}^1 \calG_{\rho}(x, y) z(x) z(y) dx d y, \quad
   \hat{z}= \frac{1}{\sqrt\rho} \sum_{j=1}^k c_j \phi_j, \quad \|\hat{z}\|=1.
   \]
   We have
   \begin{equation*}
       \widetilde{\mu}_k \ge \sum_{j=1}^k \mu_j c_j^2 \ge \mu_k \sum_{j=1}^k c_j^2 = \mu_k \|\sqrt\rho \hat{z}\| \ge \mu_k \inf\sqrt\rho.
   \end{equation*}
   In the second equation, we choose the space $S_{k-1} = \text{span}(\sqrt\rho\phi_1,\cdots, \sqrt\rho\phi_{k-1})$ and let
   \[
   \tilde{z}= \argmax_{z\in S_{k-1}^{\perp}, \|z\|=1} \int_{-1}^1\int_{-1}^1 \calG_{\rho}(x, y) z(x) z(y) dx d y, \quad \tilde{z} = \frac{1}{\sqrt\rho}\sum_{j=k}^{\infty} c_j \phi_j\in S_{k-1}^{\perp}, \quad \|\tilde{z}\|=1,
   \] 
   then
   \begin{equation*}
       \widetilde{\mu}_k \le \sum_{j=k}^{\infty} \mu_j c_j^2 \le \mu_k \sum_{j=k}^{\infty} c_j^2 = \mu_k \|\sqrt\rho \tilde{z}\| \le \mu_k \sup \sqrt\rho.
   \end{equation*}
\end{proof}

If the density function $\rho$ is regular enough, we show a more precise characterization of the eigenvalues and that the eigenfunctions are asymptotically Fourier series. Moreover, the following probabilistic estimate for the eigenvalues of the Gram matrix can be derived.
\begin{theorem}\label{COR: EIGEN 1D}
     Suppose $\{b_i\}_{i=1}^n$ are i.i.d with probability density function $\rho\in C^3[-1,1]$ on $D$ such that $0< \underline{c} \le \rho(x) \le \bar{c} <\infty $. Let $\tilde\lambda_1 \ge \tilde\lambda_2 \ge \cdots \ge \tilde\lambda_n \ge 0$ be the eigenvalues of the corresponding Gram matrix $\bmG := (\cG(b_i,b_j))_{1\le i,j\le n}$, then for sufficiently large $n$,
    \begin{equation}
    \begin{aligned}
    \left|\tilde\lambda_i - \frac{n}{2}\tilde\mu_i\right| = \begin{cases}
             \cO\left(n^{\frac{5}{8}}i^{-3} \sqrt{\log \frac{n}{p}}\right), & i < n^{\frac{7}{8}},\\
             \cO\left(n^{-2}\sqrt{\log\frac{n}{p}}\right), & n^{\frac{7}{8}} \le  i \le n,
         \end{cases}      
    \end{aligned}
    \end{equation}
     with probability $1-p$,  where $\tilde\mu_i = \Theta(i^{-4})$ is the $k$-th eigenvalue of $\calG_{\rho}$.
\end{theorem}
\begin{proof}
Without loss of generality, we assume $b_1\le b_2\le\cdots\le b_n$. Let $\phi_{\rho,k}$ be eigenfunction for the $k$th eigenvalue of $\cG_{\rho}$. Using similar derivation as before, we obtain the differential equation
\begin{equation}\label{EQ: Nonuniform ODE}
    \frac{1}{\tilde\mu_k}\sqrt{\rho(x)}\phi_{\rho,k}(x)=\frac{d^4}{dx^4}\left[\frac{1}{\sqrt{\rho(x)}}\phi_{\rho,k}(x)\right]
\end{equation}
with boundary conditions $\phi_{\rho,k}(1)=\phi_{\rho,k}'(1)=\phi_{\rho,k}''(-1)=\phi_{\rho,k}'''(-1)=0$.
Denote $\psi_{\rho,k}(x):=\frac{1}{\sqrt{\rho(x)}}\phi_{\rho,k}(x)$, then the above equation~\eqref{EQ: Nonuniform ODE} becomes 
\begin{equation}\label{EQ: Nonuniform ODE 2}
\psi_{\rho,k}^{(4)}= \frac{1}{\tilde\mu_k}\rho(x)\psi_{\rho,k}.
\end{equation}
Using a change of variable in the spirit of the Liouville transform,
$$t = -1 + \frac{2}{H}\int_{-1}^{x} {\rho(s)}^{1/4}ds,\quad H=\int_{-1}^{1} {\rho(s)}^{1/4}ds.$$
One can find that the differential equation~\eqref{EQ: Nonuniform ODE 2} reduces to the form
\begin{equation}\label{EQ: Nonuniform ODE 3}
\frac{d^4}{dt^4}\psi_{\rho,k} + p_1(t)\frac{d^3}{dt^3}\psi_{\rho,k} + p_2(t)\frac{d^2}{dt^2}\psi_{\rho,k}  + p_3(t)\frac{d}{dt}\psi_{\rho,k} + p_4(t)\psi_{\rho,k} = \frac{H^4}{\tilde\mu_k} \psi_{\rho,k},
\end{equation}
and the functions $p_k(t)$ are continuous over $[-1,1]$ since 
 $\rho\in C^3[-1,1]$. For $k$ sufficiently large, the asymptotic behaviors of eigenvalues $\tilde\mu_k^{-1}=(\frac{k}{H}\pi)^4(1+\cO(k^{-1}))$ can be derived based on Stone's estimate of linearly independent basis~\cite{stone1926comparison} and Birkhoff's method~\cite{birkhoff1908boundary,naimark1967}. Furthermore, the eigenfunction $\psi_{\rho, k}$ is asymptotically equivalent to Fourier modes in $t$ for sufficiently large $k$  and can be shown uniformly bounded, see detailed discussions in \S 4.10 of~\cite{naimark1967}. 
 
 Similar to Theorem~\ref{THM: Cond Num}, we define the matrix $\Phi_m\in\bbR^{n\times m}$ with entries $(\Phi_m)_{i,j} = \phi_{\rho,j}(b_i)/\sqrt{\rho(b_i)}$, $1\le j\le m$, then the Gram matrix with entry $\bmG_{i,j}= \cG(b_i, b_j)$ equals to
    \begin{equation}
        \bmG = \Phi_m \Lambda_m \Phi_m^T + \bmE_m, 
    \end{equation} 
    where $(\Lambda_m)_{i,j} = \delta_{i,j} \widetilde{\mu_i}$, $(\bmE_m)_{i,j}= \sum_{k > m} \widetilde{\mu_k}  \phi_{\rho,k}(b_i)  \phi_{\rho,k}(b_j)/\sqrt{\rho(b_i)\rho(b_j)}$, and the estimate~\eqref{EQ: Uni Bound} still holds. For each pair of $k,j$, applying the Hoeffding's inequality to~\eqref{EQ: Quadrature}, we get 
 \[
        \left|\sum_{i=1}^n \frac{ \phi_{\rho,j}(b_i) \phi_{\rho,k}(b_i)}{{\rho(b_i)}} - \frac{n}{2}\int_{-1}^1  \phi_{\rho,j}(x) \phi_{\rho,k}(x) dx\right|  = \cO\left(\sqrt{n\log \frac{m^2}{p}}\right)
 \]
 with probability $1-\frac{p}{m^2}$, due to the uniform boundedness of the eigenfunctions $ \phi_{\rho,k}$. Then with probability at least $1-p$, we have $\|\Phi^T_m\Phi_m-\frac{n}{2}\mbox{Id}_m\|_{{\rm  op}}=\cO\left(m\sqrt{n\log \frac{m^2}{p}}\right)$ and 
    \begin{equation}
    \begin{aligned}
    \left|\tilde\lambda_i - \frac{n}{2}\tilde\mu_i\right| &\le C \min_{i\le m \le n}\left(\frac{n}{m^3} + \frac{m}{i^4}\sqrt{n\log \frac{m^2}{p}}\right)
    = \begin{cases}
            \cO\left(n^{\frac{5}{8}}i^{-3} \sqrt{\log \frac{n}{p}}\right), & i < n^{\frac{7}{8}}, \\
\cO\left(n^{-2}\sqrt{\log\frac{n}{p}}\right), & n^{\frac{7}{8}}\le  i \le n,
         \end{cases}      
    \end{aligned}
    \end{equation}
    where $  m = \min(n^{\frac{1}{8}} i, n)$.
\end{proof}

\begin{corollary}
    Under the same assumption of Corollary~\ref{COR: EIGEN 1D}, with at probability $1 -p$, $1>p > ne^{-cn^{3/4}}$ for some $0<c=\cO(1)$, the condition number of Gram matrix $\bmG$ satisfies 
    \begin{equation*}
        \kappa = \lambda_1/\lambda_n = \Omega \left(n^{3}(\log\frac{n}{p})^{-\frac{1}{2}}\right) .
    \end{equation*}
\end{corollary}
\begin{proof}
There exist positive constants $C_1, C_2, C_n$ of $\cO(1)$ such that 
\[
\lambda_1 \ge  C_1 n \left(1 - C_2n^{-3/8}\sqrt{\log\frac{n}{p}}\right), \quad
\lambda_n \le  C_n n^{-2}\sqrt{\log\frac{n}{p}}.
\]
Choose $c=\frac{1}{2C_2}=\cO(1)$, then $\lambda_1 >\frac{C_1 n}{2}$.
\end{proof}

\subsection{Multi-dimensional case}\label{sec:multiD}

Now we provide the spectral analysis for \texttt{ReLU} functions in arbitrary dimensions and give a spectral estimate, although we cannot compute the eigenvalues and eigenfunctions explicitly. In this section, we consider domain $D = B_d(1)$ which is the unit ball in $d$-dimension. The class of neural network $\calH_n$ is 
\begin{equation*}
    h(\bmx) = c+\sum_{i=1}^n a_i \sigma(\bmw_i\cdot \bmx - b_i), \quad
    \bmw_i\in \bbS^{d-1}, ~b_i\in [-1, 1].
\end{equation*}
Denote $V = \bbS^{d-1}\times [-1, 1]$, similar to the one-dimensional setting, we consider the corresponding continuous kernel $G: L^2(V)\mapsto L^2(V)$
\begin{equation*}
    G(\bmw, b, \bmw',b') = \int_{D} \sigma(\bmw\cdot \bmx - b) \sigma(\bmw'\cdot \bmx - b') d\bmx .
\end{equation*}
Let $\phi_k(\bmw, b)$ be an eigenfunction for eigenvalue $\lambda_k$ which satisfies
\begin{equation}\label{EQ: GRAM HIGH DIM}
    \int_{V} G(\bmw, b, \bmw', b') \phi_k(\bmw', b') d\bmw' db' = \lambda_k \phi_k(\bmw, b).
\end{equation}
It is not hard to see that $\phi_k(\bmw, b)$ is supported on $V$ and $\phi_k(\bmw, 1)\!=\!\partial_b\phi_k(\bmw, 1)\!=\!\partial_b^2\phi_k(\bmw, 1)\!=\!0$. 

One of the useful tools to study two-layer \texttt{ReLU} networks of infinite width is the Radon transform~\cite{ongie2019function,savarese2019infinite}. Next, we construct the theory for the Gram matrix in high dimensions using the properties of the Radon transform.
\begin{definition}
    Let $f:\bbR^d\to \bbR$ be an integrable function over all hyperplanes, the Radon transform 
    \begin{equation*}
        \mathcal{R} f(\bmw, b) = \int_{\{\bmx\mid \bmw\cdot \bmx - b = 0\}} f(\bmx) d H_{d-1}(\bmx),\quad\forall (\bmw, b)\in\bbS^{d-1}\times\bbR,
    \end{equation*}
    where
  $dH_{d-1}$ denotes the $(d-1)$ dimensional Lebesgue measure. The adjoint transform $\mathcal{R}^{\ast}:\bbS^{d-1}\times\bbR\to\bbR^d$ is 
\begin{equation*}
    \mathcal{R}^{\ast} \Phi(\bmx) \coloneqq \int_{\bbS^{d-1}} \Phi(\bmw, \bmw\cdot \bmx) d\bmw,\quad \forall \bmx\in\bbR^d.
\end{equation*}
\end{definition}
\begin{theorem}[Helgason~\cite{helgason2011integral}]\label{THM: INV RAD}
    The inversion formula of the Radon transform is
\begin{equation*}
    c_d f = (-\Delta)^{(d-1)/2} \mathcal{R}^{\ast}\mathcal{R} f,
\end{equation*}
where $c_d = (4\pi)^{(d-1)/2}\frac{\Gamma(d/2)}{\Gamma(1/2)}$.
\end{theorem}
\begin{lemma}[Helgason~\cite{helgason2022groups}, Lemma 2.1]\label{LEM: INTERWINE}
    These intertwining relations hold: $\mathcal{R}\Delta = \partial_b^2 \mathcal{R}$ and $\mathcal{R}^{\ast}\partial_b^2 = \Delta \mathcal{R}^{\ast}$. 
\end{lemma}
Then we have the following lemma for the eigenvalues.
\begin{lemma}\label{le:Radon1}
Let $u_k(\bmx) = \int_{V} \sigma(\bmw\cdot \bmx - b) \phi_k(\bmw, b) d\bmw db$ and $\chi_D(\bmx)$ be the characteristic function of $D$, then if $d$ is odd, $$c_d u_k(\bmx) = \lambda_k (-\Delta)^{(d+3)/2} u_k(\bmx) ,\quad \forall \bmx\in D.$$
If $d$ is even 
$$c_d (-\Delta)^{-1/2} \chi_D u_k(\bmx) = \lambda_k (-\Delta)^{(d+3)/2} (-\Delta)^{-1/2}\chi_D u_k(\bmx) ,\quad \forall \bmx\in D.$$
\end{lemma}
\begin{proof}
The function $u_k(\bmx)$ satisfies
\begin{equation}\label{eq:id1}
    \Delta u_k(\bmx) = \int_V \delta(\bmw\cdot \bmx - b) \phi_k(\bmw, b) d\bmw d b = \cR^{\ast} \phi_k(\bmx),\quad \forall \bmx \in D.
\end{equation} We also define
\begin{equation}\label{EQ: PHI_K}
    \widetilde{\phi}_k(\bmw, b) := \frac{1}{\lambda_k}\int_{D} u_k(\bmx) \sigma(\bmw\cdot \bmx - b) d\bmx ,\quad \forall (\bmw, b)\in \bbS^{d-1}\times \bbR, 
\end{equation}
then $\phi_k(\bmw, b) = \widetilde{\phi}_k(\bmw, b)$ on $V$, hence $\cR^{\ast} \phi_k = \cR^{\ast} \widetilde{\phi}_k$ on $D$.
Differentiate~\eqref{EQ: PHI_K} twice in $b$, 
\begin{equation}\label{EQ: SECOND-DER}
    \lambda_k \partial^2_b \widetilde{\phi}_k = \int_D u_k(\bmx) \delta(\bmw\cdot \bmx - b) d\bmx = \cR \chi_D u_k ,\quad \forall (\bmw, b)\in \bbS^{d-1}\times \bbR.
\end{equation}
For odd $d$, apply $(-\Delta)^{(d-1)/2}\cR^{\ast}$ on both sides of~\eqref{EQ: SECOND-DER} and use Lemma~\ref{LEM: INTERWINE}, we obtain
\begin{equation*}
    \lambda_k (-\Delta)^{(d-1)/2} \cR^{\ast} \partial^2_b \widetilde{\phi}_k =   \lambda_k (-\Delta)^{(d-1)/2} \Delta \cR^{\ast} \widetilde{\phi}_k = (-\Delta)^{(d-1)/2} \cR^{\ast} (\cR \chi_D u_k),
\end{equation*}
then with~\eqref{eq:id1} and Theorem~\ref{THM: INV RAD}, it implies
\begin{equation*}
    c_d \chi_D u_k(\bmx) = \lambda_k (-\Delta)^{(d+3)/2} u_k(\bmx) ,\quad \forall \bmx\in D.
\end{equation*}
For even $d$, apply $(-\Delta)^{(d-2)/2}\cR^{\ast}$ on both sides of~\eqref{EQ: SECOND-DER} and follow the same procedure, we obtain
\begin{equation*}
    c_d \psi_k(\bmx) = \lambda_k (-\Delta)^{(d+3
    )/2} \psi_k(\bmx) ,\quad \forall \bmx\in D,
\end{equation*}
where $\psi_k(\bmx) = (-\Delta)^{-1/2}\chi_D u_k(\bmx) $. The eigenfunctions $\phi_k(\bmw, b)$ can be retrieved from the relation~\eqref{EQ: SECOND-DER} and boundary conditions $\phi_k(\bmw, 1) = \partial_b \phi_k(\bmw, 1) = \partial_b^2 \phi_k (\bmw, 1) = 0$.
\end{proof}

The above Lemma~\ref{le:Radon1} shows that $c_d\lambda_k^{-1}$ is the eigenvalue of $(-\Delta)^{(d+3)/2}$. From the Weyl's law for $-\Delta$, we have $\lambda_k = \Theta(k^{-(d+3)/d})$ (using Landau notation) as $k\to\infty$. Suppose the target function $f(\bmx),\, \bmx\in D\subset \mathbb{R}^d$, can be represented by a superposition of \texttt{ReLU} functions with weight $h(\bmw, b),\, (\bmw, b)\in V$:
\[
f(\bmx)=\int_V\sigma(\bmw\cdot \bmx -b)h(\bmw,b)d\bmw db \quad \Longrightarrow\quad \Delta f(\bmx)=\mathcal{R}^{\ast}h(\bmw, b).
\]
In the one-dimensional case, the relation is further simplified to $\frac{d^2}{dx^2}f(x)=h(x)$. Assume $h(\bmw,b)=\sum_{k=1}^{\infty}\alpha_k\phi_k(\bmw,x)$ and use $u_k$  defined in Lemma~\ref{le:Radon1} and~\eqref{eq:id1}, we have
\begin{equation}\label{eq:expansion}
f(\bmx)=\sum_{k=1}^{\infty}\alpha_k \Delta^{-1}\mathcal{R}^{\ast}\phi_k(\bmx)=\sum_{k=1}^{\infty}\alpha_k u_k(\bmx).
\end{equation}
Now we characterize the asymptotic behavior of $u_k$ when $D$ is a unit ball in $\bbR^d$.
 Due to the symmetry,  $u_k$ can be separated into $u_k(\bmx) = V_k(|\bmx|) Y_{k}(\hat{\bmx})$, where $\hat{\bmx} = \bmx/|\bmx|$ and  $Y_{k}$ denotes a spherical harmonics of order $l$ on $\mathbb{S}^{d-1}$ and $V_k(r)$ satisfies the following by Lemma~\ref{le:Radon1}
$$\left(-\frac{d^2}{dr^2} - \frac{d-1}{r}  \frac{d}{dr} + \frac{1}{r^2} l(l+d-2)\right)^{(d+3)/2} V_k = \frac{c_d}{\lambda_k} V_k. $$
The equation can be reduced to a set of Bessel's differential equations, 

$$ \left(-\frac{d^2}{dr^2} - \frac{d-1}{r} \frac{d}{dr}+ \frac{1}{r^2} l(l+d-2) - \mu_{k,p}^2 \right) V_k =  0,$$
where $\mu_{k,p} = \left[\frac{c_d}{\lambda_k}\right]^{{1/(d+3)}} e^{2p \pi i/(d+3)}$, $p=1,2,\cdots, d+3$. Take a change of variable $s = \mu_{k,p} r $, the above equation becomes the standard Bessel's equation
$$\left( s^2 \frac{d^2}{d s^2} + (d-1) s\frac{d}{ds } + (s^2 - l(l+d-2)) \right) V_k = 0 ,$$
which implies $V_k(r) = \sum_{p=1}^{d+3}  {\beta_{k,p}}{{(\mu_{k,p} r)}^{-\nu}} J_{l+\nu}(\mu_{k,p} r)$, $\nu = \frac{d-2}{2}$ and $\beta_j\in\bbC$ are coefficients.
For sufficiently large $k$ that $|\mu_{k,p}|\gg (l+\nu)^2$ with $\mathrm{Im}(\mu_{k,p})\neq 0$, the asymptotical behavior is $\sqrt{\frac{2}{\pi \mu_{k,p} r}} \cos(\mu_{k,p} r - \frac{l+\nu}{2}\pi - \frac{\pi}{4})$, whose magnitude grows exponentially like $\cO(e^{\mathrm{Im}(\mu_{k,p}) r})$ with different rates. 
Since the $L^2$ norm of $u_k(\bmx)$ is uniformly bounded, the coefficients $\beta_{k,p}\to 0$ if $\mathrm{Im}(\mu_{k,p})\neq 0$. That means, asymptotically, the eigenfunction $u_k$ behaves like a usual Bessel function multiplied with spherical harmonics. For a general compact domain $D\subset \bbR^{d}$, all the arguments above are still valid except it is more difficult to explicitly write out the corresponding $u_k$.

\begin{remark}\label{re:relu-k}For the Gram matrix corresponding to activation function $\texttt{ReLU}^m(x)=\frac{1}{m!}[\max(x,0)]^m$ in $d$-dimension with $m\ge 1$, one can modify the above theorems to formally show: 
\begin{itemize}
    \item If $m$ is odd, $(c_{d}\lambda_k^{-1}, \mathcal{R}^{\ast} \phi_k)$ forms an eigenpair of the operator $ (-\Delta)^{\frac{d-1}{2}+(m+1)} $.
    \item If $m$ is even, $( c_d\lambda_k^{-1}, \mathcal{R}^{\ast}\partial_b \phi_k)$ forms an eigenpair of $ (-\Delta)^{\frac{d-1}{2}+(m+1)} $.
\end{itemize}
 Using the inversion formula for the Radon transform, we have $\lambda_k=\Theta(k^{-\frac{d+2m+1}{d}})$.
\end{remark}

\begin{remark}\label{re:highD}
    Although the decay of eigenvalues seems slower in higher dimensions, the number of Fourier modes less than frequency $\nu$ is $\Theta(\nu^d)$.     
In other words, given a threshold $\mathcal{\epsilon}$ for the leading singular value, no matter how wide a two-layer $\texttt{ReLU}$ neural network is, it can only resolve all Fourier modes up to frequency $\mathcal{O}(\mathcal{\epsilon}^{-\frac{1}{d+3}})$. 
\end{remark}


\subsection{Numerical experiments}\label{sec:numerics}
In this section, various numerical experiments are presented to verify our earlier analysis results: 1) the spectral property of the Gram matrix corresponding to \texttt{ReLU} functions, and 2) the low-pass filter nature of two-layer networks in the least square setting. 

\subsubsection{Spectrum and eigenmodes of Gram matrix in one dimension}

The first two numerical experiments show the spectrum of the Gram matrix for \texttt{ReLU} functions in the one-dimensional case. Figure~\ref{fig:spectrum:uniform} and 
the first row of Figure~\ref{fig:eigenmodes:1D} is the plot for eigenvalues in descending order and selected eigenmodes for uniform biases respectively, i.e., $b_j=b_1+2(j-1)/(n-1)$ for $j\in [n]$ with $b_1=-1$. They agree with our analysis perfectly. Figure~\ref{fig:spectrum:adaptive} and the second row of Figure~\ref{fig:eigenmodes:1D} are for non-uniform biases adaptive to the rate of change, i.e., $|f^\prime(x)|$, of $f(x)=\arctan(25x)$ as an example.
Define $F(x)=\int_{-1}^{x} |f^\prime(t)|dt\big/\int_{-1}^{1} |f^\prime(t)|dt$, which is strictly increasing. 
There exists unique $b_i\in [-1,1]$ such that $F(b_i)=(i-1)/(n-1)$ for $i\in [n]$. We see that the conditioning becomes a little worse. However, the leading eigenmodes are more adaptive to the rapid change of $f(x)$ at 0. This is demonstrated further when using a least square approximation based on \texttt{ReLU} functions with uniform and adaptive biases. In Figure~\ref{fig:projection}, we show the projection of $f$ on the leading eigenmodes corresponding to the Gram matrix.
We observe that fewer leading eigenmodes are needed to represent/approximate the target function for adaptive biases compared to uniform biases.  

\begin{figure}
    \centering	
    \begin{minipage}[c]{0.495\textwidth}
    \begin{subfigure}[c]{0.48\linewidth}
    \centering            \includegraphics[width=0.9985\textwidth]{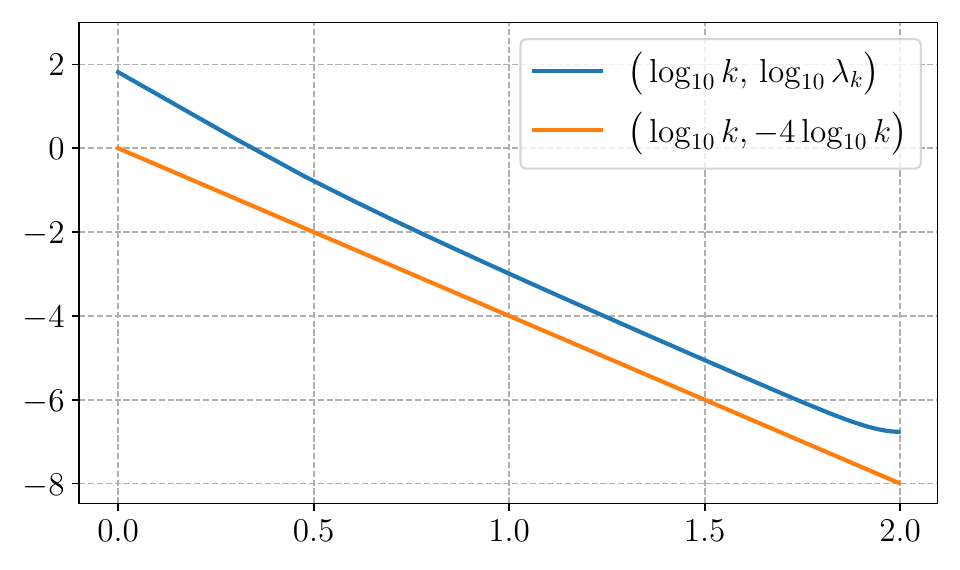}
    \subcaption{$n=100$.}
    \end{subfigure}
    \begin{subfigure}[c]{0.48\linewidth}
    \centering            \includegraphics[width=0.9985\textwidth]{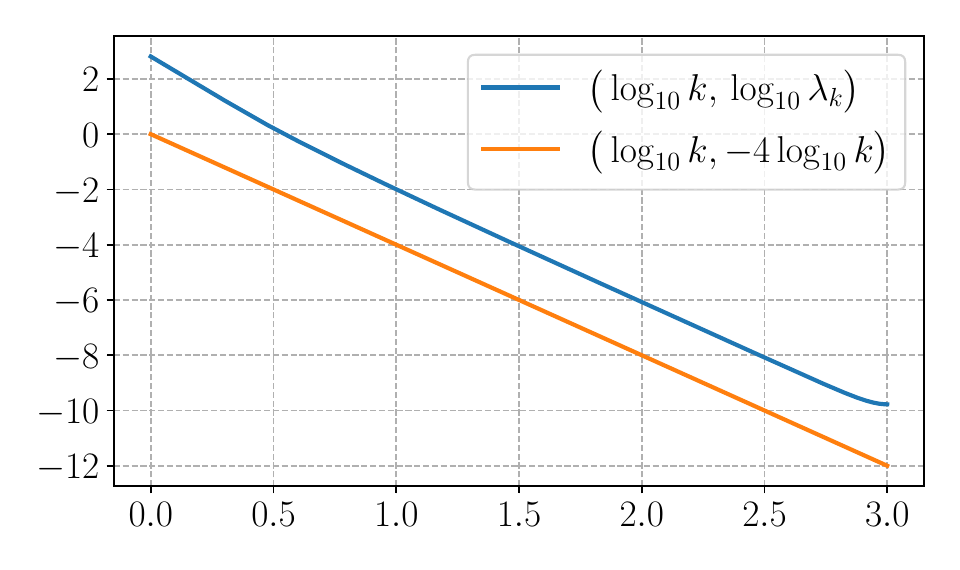}
    \subcaption{$n=1000$.}
    \end{subfigure}
    \caption{Spectrum for uniform $\bmb$.}
    	\label{fig:spectrum:uniform}
    \end{minipage}\hfill
        \begin{minipage}[c]{0.495\textwidth}
    \begin{subfigure}[c]{0.48\linewidth}
    \centering            \includegraphics[width=0.9985\textwidth]{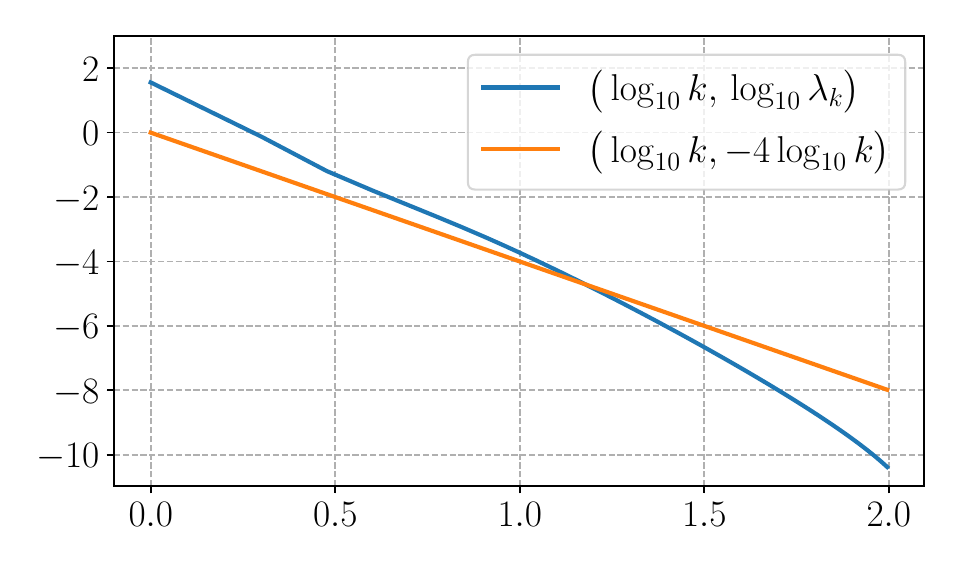}
    \subcaption{$n=100$.}
    \end{subfigure}
    \begin{subfigure}[c]{0.48\linewidth}
    \centering            \includegraphics[width=0.9985\textwidth]{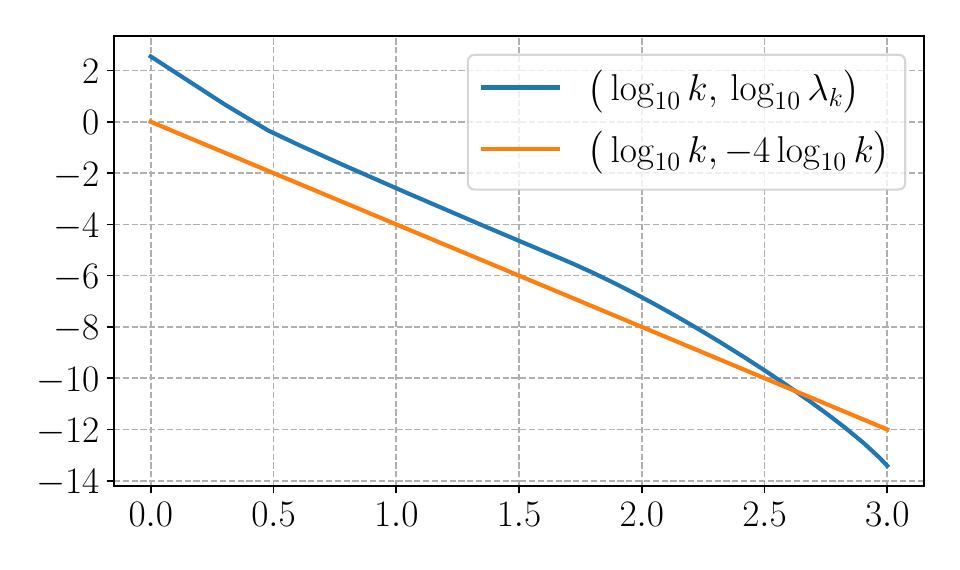}
    \subcaption{$n=1000$.}
    \end{subfigure}
    \caption{Spectrum for adaptive $\bmb$.}
    	\label{fig:spectrum:adaptive}
    \end{minipage}
\end{figure}
\begin{figure}
    \centering	
    \begin{subfigure}[c]{0.245\textwidth}
    \centering            \includegraphics[width=0.9985\textwidth]{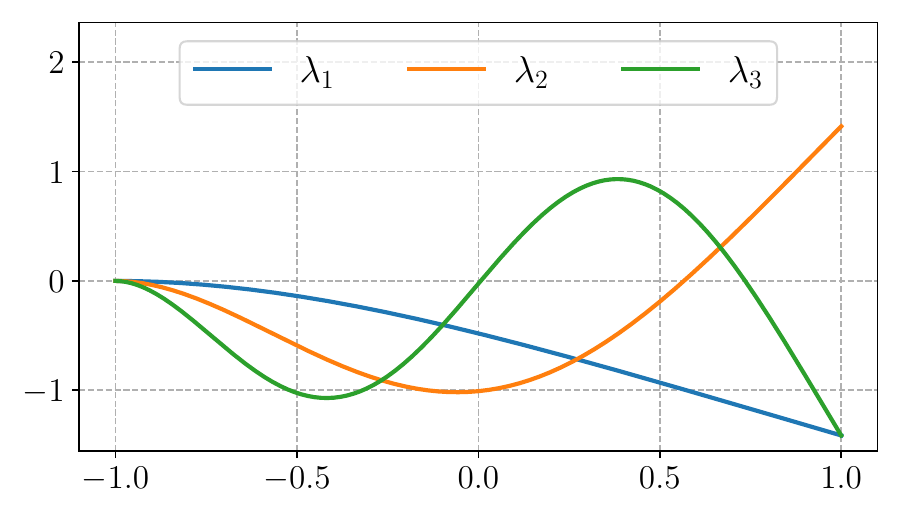}
    \end{subfigure}\hfill
    \begin{subfigure}[c]{0.245\textwidth}
    \centering            \includegraphics[width=0.9985\textwidth]{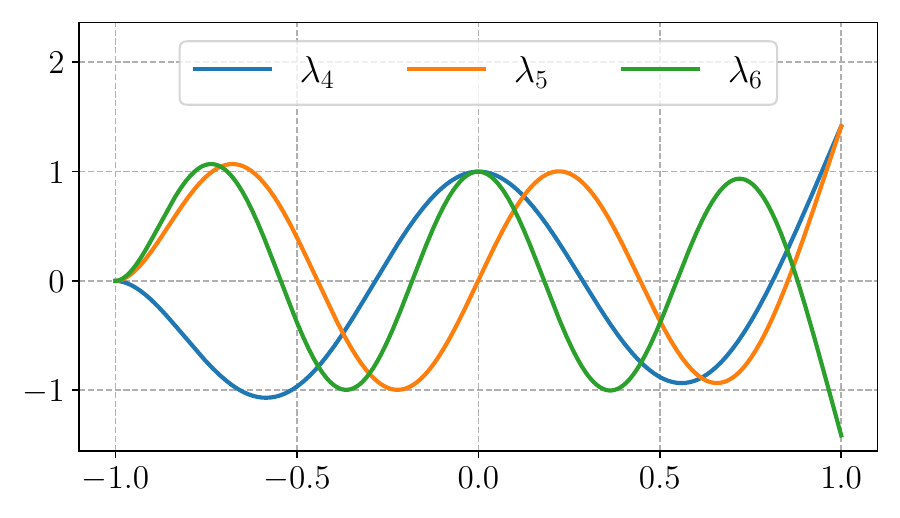}
    \end{subfigure}\hfill
    \begin{subfigure}[c]{0.245\textwidth}
    \centering            \includegraphics[width=0.9985\textwidth]{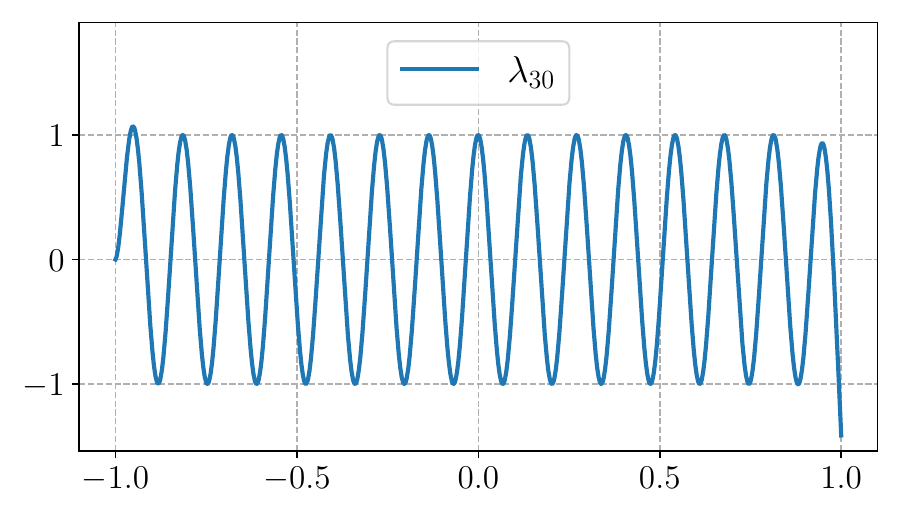}
    \end{subfigure}\hfill
    \begin{subfigure}[c]{0.245\textwidth}
    \centering            \includegraphics[width=0.9985\textwidth]{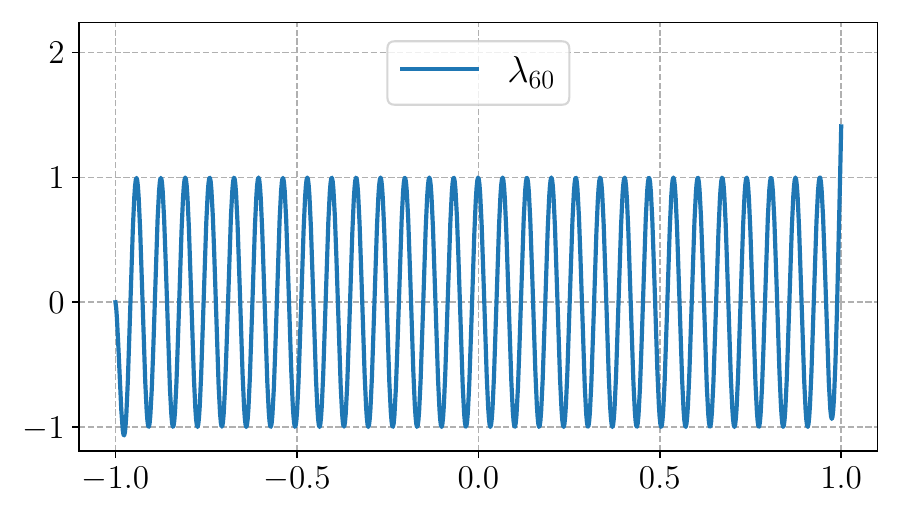}
    \end{subfigure}\\
    \begin{subfigure}[c]{0.245\textwidth}
    \centering            \includegraphics[width=0.9985\textwidth]{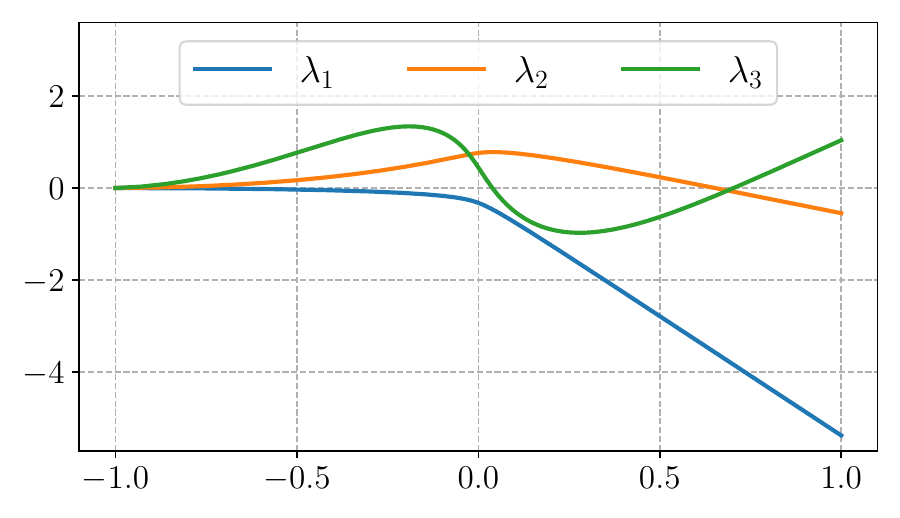}
    \end{subfigure}\hfill
    \begin{subfigure}[c]{0.245\textwidth}
    \centering            \includegraphics[width=0.9985\textwidth]{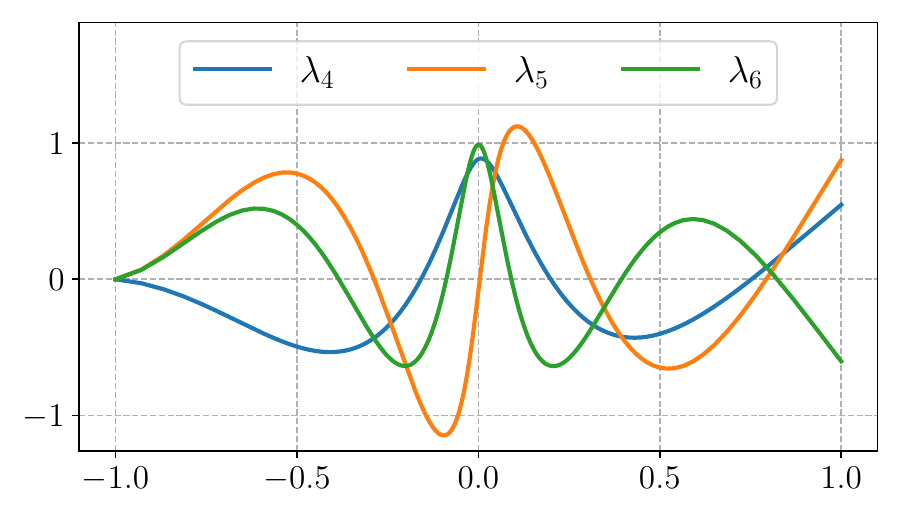}
    \end{subfigure}\hfill
    \begin{subfigure}[c]{0.245\textwidth}
    \centering            \includegraphics[width=0.9985\textwidth]{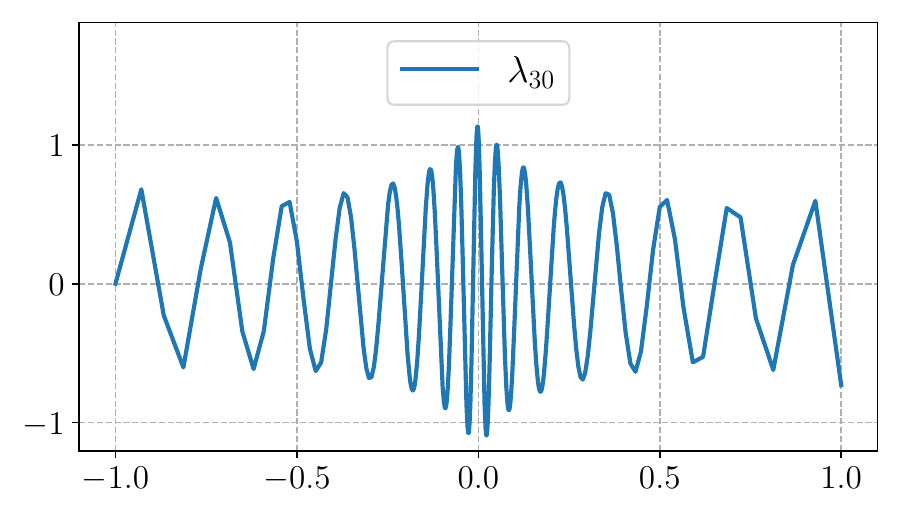}
    \end{subfigure}\hfill
    \begin{subfigure}[c]{0.245\textwidth}
    \centering            \includegraphics[width=0.9985\textwidth]{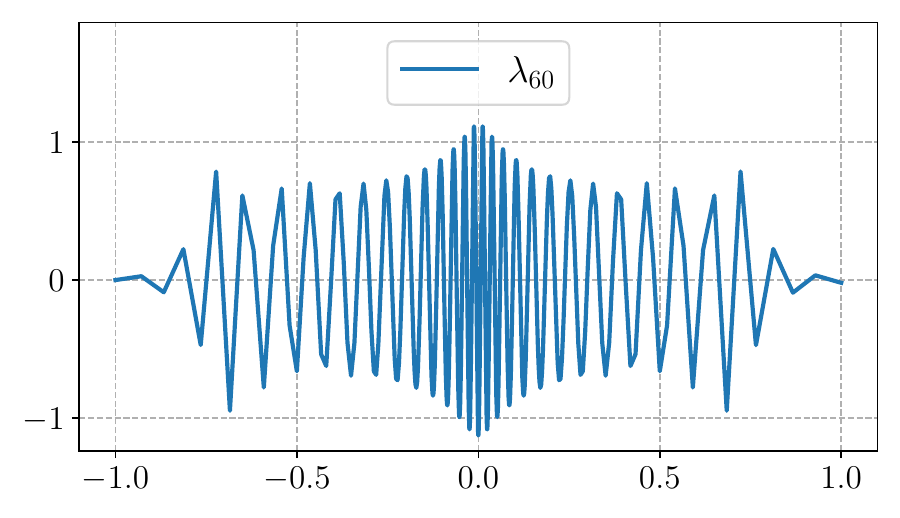}
    \end{subfigure}
    \caption{Eigenmodes of $\lambda_k$ for $k=\{1,2,3\},\{4,5,6\},30, 60$ with $n=1000$. The first and second rows correspond to uniform and adaptive $\bmb$, respectively.}
    	\label{fig:eigenmodes:1D}
\end{figure}
\begin{figure}
    \centering
        \begin{subfigure}[c]{0.2475\textwidth}
    \centering            \includegraphics[width=0.9985\textwidth]{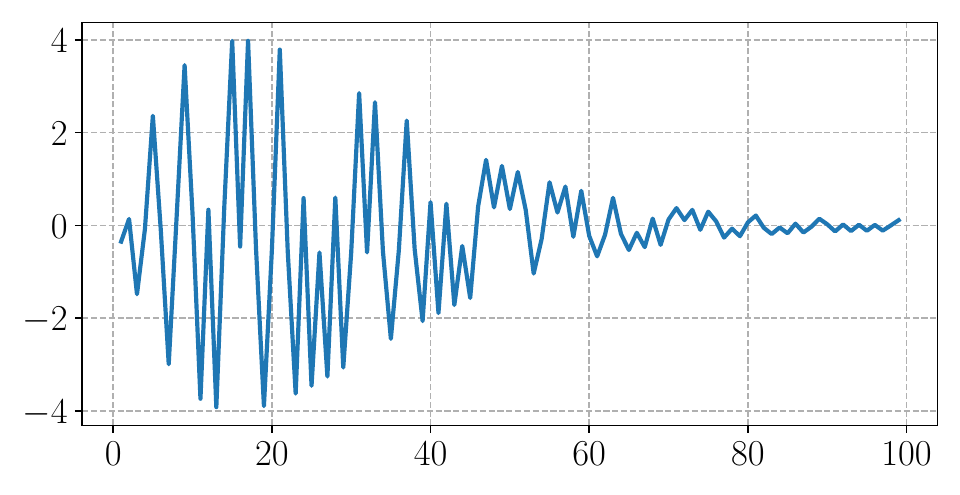}
    \subcaption{Uniform $\bmb$, $n=100$.}
    \end{subfigure}\hfill
    \begin{subfigure}[c]{0.2475\textwidth}
    \centering            \includegraphics[width=0.9985\textwidth]{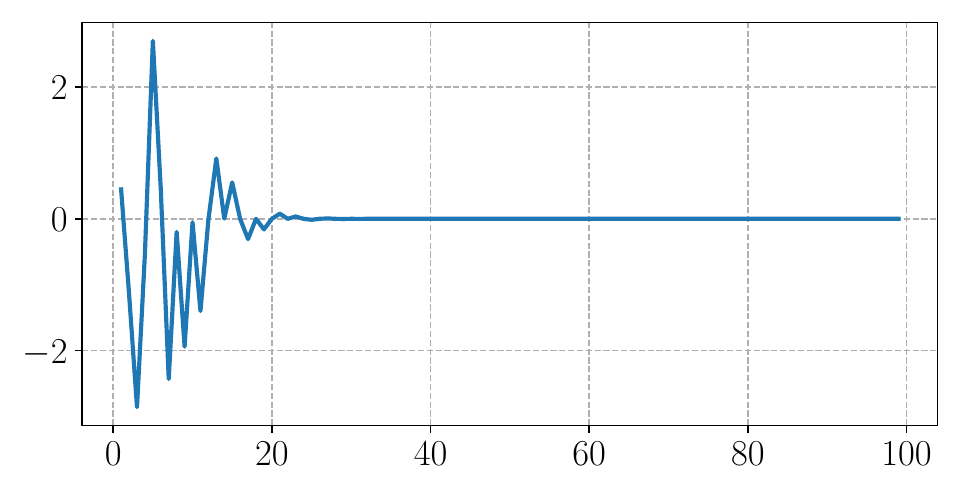}
    \subcaption{Adaptive $\bmb$, $n=100$.}
    \end{subfigure}\hfill
    \begin{subfigure}[c]{0.2475\textwidth}
    \centering            \includegraphics[width=0.9985\textwidth]{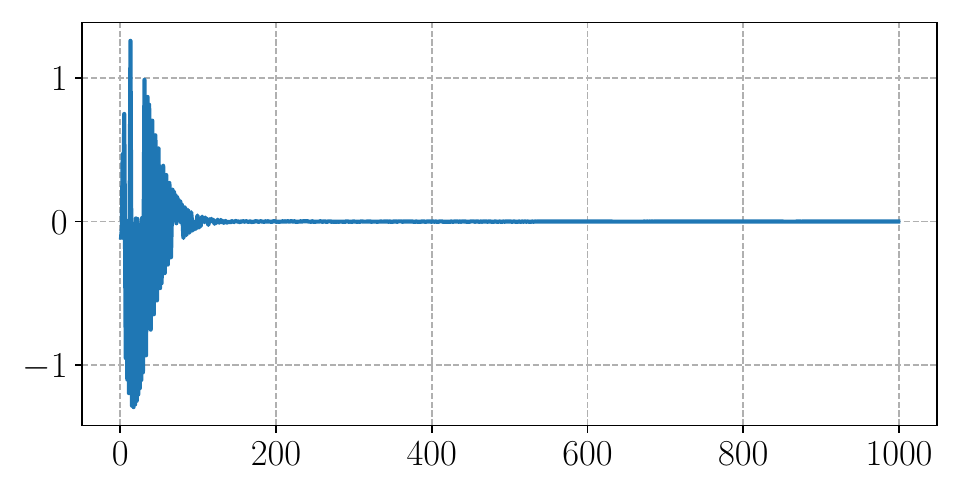}
    \subcaption{Uniform $\bmb$, $n=1000$.}
    \end{subfigure}\hfill
    \begin{subfigure}[c]{0.2475\textwidth}
    \centering            \includegraphics[width=0.9985\textwidth]{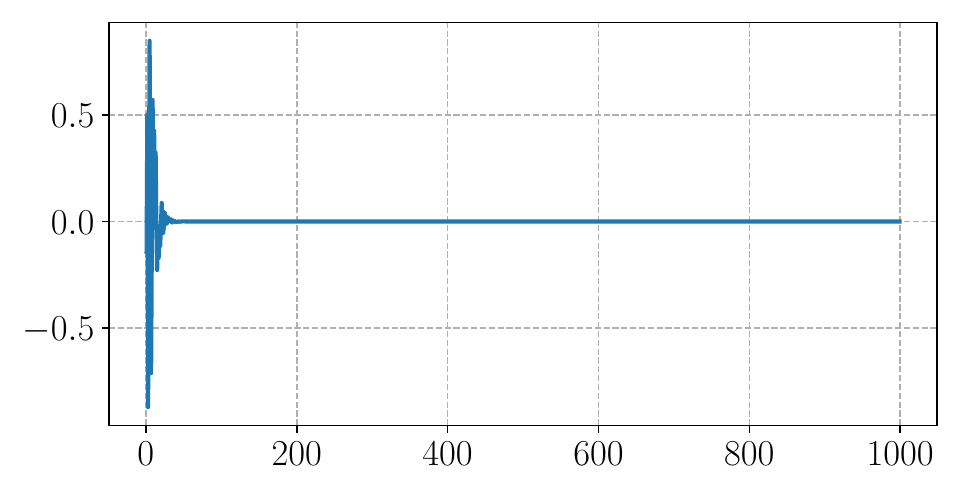}
    \subcaption{Adaptive $\bmb$, $n=1000$.}
    \end{subfigure}
    \caption{Projection of $f$ on the eigenmodes of the Gram matrix: coefficients vs. eigenmode index.}
    	\label{fig:projection}
\end{figure}
\subsubsection{Approximation in one dimension: FEM basis vs. \texttt{ReLU} basis}
Next, we use numerical experiments to show least square approximation using  FEM (finite element methods) basis vs. \texttt{ReLU} basis, two-layer neural networks (NN), in different setups when machine precision, grid resolution, and the conditioning of the Gram matrix all play a role in practice. As discussed before, these two bases are equivalent mathematically. Numerically machine precision and grid resolution are common factors for both representations. The only difference is the Gram matrix. 
Table~\ref{tab:adaptive:b} summarizes the results. In particular, the final numerical mean square error (MSE) is directly related to the least square approximation. 
As we can see from the results, when single precision is used in the computation and $n=100$ evenly distributed grid points (biases) are used, the grid resolution is the bottleneck for numerical accuracy near the rapid change. So the errors for FEM and NN are about the same. Numerical errors are reduced when the grid distribution is adaptive to the rapid change of the target function as described above. However, adaptive grids are much more effective for finite element basis.
When a very fine grid $n=1000$ is used, machine precision and the conditioning of the Gram matrix become the most important factors. Well-conditioned FEM (even for adaptive grid) can reach the machine precision while finite machine precision and ill-conditioned NN limit the total number of leading eigenmodes (low pass filter) and can not approximate a function with rapid change well. Moreover, increasing NN width further does not help. When double precision is used, even the NN has enough leading eigenmodes to approximate the target function well and the grid resolution becomes the limit for both FEM and NN. Hence the numerical errors for FEM and NN are similar.

\begin{table*}
		\caption{Error comparison for approximating $f(x)=\arctan(25x)$ with sufficient samples.} 
		\label{tab:adaptive:b}
		\vskip 0.05in
	\centering  	
	\resizebox{0.999\linewidth}{!}{ 
		\begin{tabular}{cccccccccccccccccccccccccccccc} 
			\toprule
    &
   &\multicolumn{4}{c}{float32}
   &\multicolumn{4}{c}{float64}\\
   			\cmidrule(lr){3-6}
			\cmidrule(lr){7-10}	
   &  & \multicolumn{2}{c}{$n=100$} 
			& \multicolumn{2}{c}{$n=1000$} 
        & \multicolumn{2}{c}{$n=100$} 
			& \multicolumn{2}{c}{$n=1000$} 
			\\

			\cmidrule(lr){3-4}
			\cmidrule(lr){5-6}			
            \cmidrule(lr){7-8}
            \cmidrule(lr){9-10}
			
			 & 
     \hspace{54pt}
    & 
			  MAX & MSE &
                 MAX & MSE &
                 MAX & MSE &
                 MAX & MSE 
			\\
    \midrule
\rowcolor{mygray} NN & Uniform $\bmb$  &  
$6.09\times10^{-2}$  &  $9.58\times10^{-5}$  &  $7.19\times10^{-2}$  &  $1.43\times10^{-4}$  &  $1.37\times10^{-2}$  &  $1.70\times10^{-6}$  &  $1.05\times10^{-4}$  &  $1.33\times10^{-10}$\\
\midrule
FEM & Uniform $\bmb$   & 
$1.37\times10^{-2}$  &  $1.70\times10^{-6}$  &  $1.05\times10^{-4}$  &  $1.33\times10^{-10}$  &  $1.37\times10^{-2}$  &  $1.70\times10^{-6}$  &  $1.05\times10^{-4}$  &  $1.33\times10^{-10}$\\  
\midrule

\rowcolor{mygray} NN & Adaptive $\bmb$   &
$6.83\times10^{-2}$  &  $7.54\times10^{-5}$  &  $1.89\times10^{-2}$  &  $1.06\times10^{-5}$  &  $3.93\times10^{-3}$  &  $1.42\times10^{-6}$  &  $4.74\times10^{-5}$  &  $1.17\times10^{-10}$\\
\midrule
FEM &  Adaptive $\bmb$    &  
$2.92\times10^{-3}$  &  $9.95\times10^{-7}$  &  $3.79\times10^{-5}$  &  $1.02\times10^{-10}$  &  $2.92\times10^{-3}$  &  $9.95\times10^{-7}$  &  $3.77\times10^{-5}$  &  $1.02\times10^{-10}$
 \\ 
			\bottomrule
		\end{tabular} 
	}
\end{table*}

Instead of using a (locally) rapid change function, we perform experiments on more and more oscillatory functions to demonstrate similar conclusions. Our target functions are $f(x)=\cos(6\pi x)-\sin(2\pi x)$ and its rescaled more oscillatory version $f(3x), f(9x)$. In these tests, instead of using the default threshold of the leading singular values of the Gram matrix based on machine precision, we introduce a manual singular value cut-off ratio\footnote{Singular values are treated as zero if they are smaller than $\eta$ times the largest singular value.}, $\eta$, to see the low pass filter effect for NN more clearly. The results are plotted in Figure~\ref{fig:LMH:fre} and the approximation errors are summarized in Table~\ref{tab:LMH:fre}. From both the plots and the MSE error, we can see that different cut-offs do not affect the FEM approximation since the condition number of the Gram matrix corresponding to the FEM basis on the uniform mesh is $\cO(1)$. As a result, all eigenmodes (frequencies) that can be resolved by the grid size can be recovered accurately and stably. The slight decrease in accuracy as the oscillation increases is due to the fact that the least square approximation error is proportional to $h^2\int|f''(x)|dx$, where $h$ is the grid size, for piecewise linear finite element basis by standard approximation theory. The dramatic effect of the cut-offs on NN is due to the fast spectral decay of the Gram matrix corresponding to the \texttt{ReLU} basis. When the cut-off ratio is $\eta=10^{-3}$, the linear space spanned by leading eigenmodes above the threshold can not approximate even the relative smooth $f(x)$ well. When the cut-off ratio is reduced to $\eta=10^{-9}$, there are enough leading modes above the threshold that can approximate $f(x), f(3x)$ well but not $f(9x)$.

\begin{figure}
    \centering
        \begin{subfigure}[c]{0.2401\linewidth}
    \centering            \includegraphics[width=0.9985\textwidth]{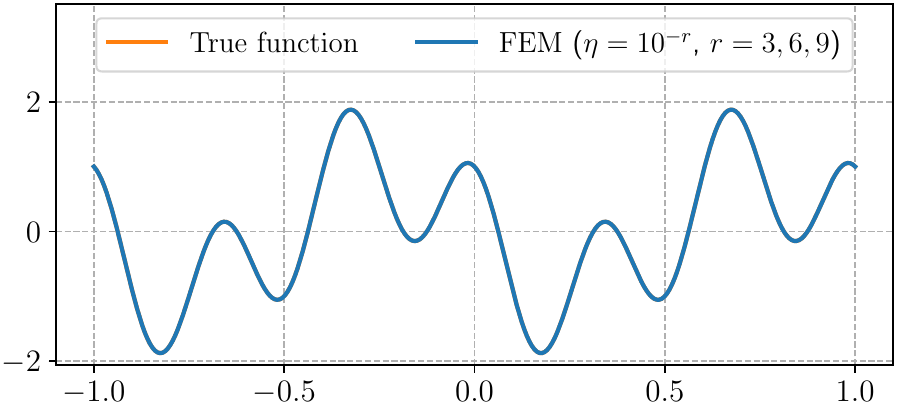}
    \end{subfigure}
    \hfill
    \begin{subfigure}[c]{0.2401\linewidth}
    \centering            \includegraphics[width=0.9985\textwidth]{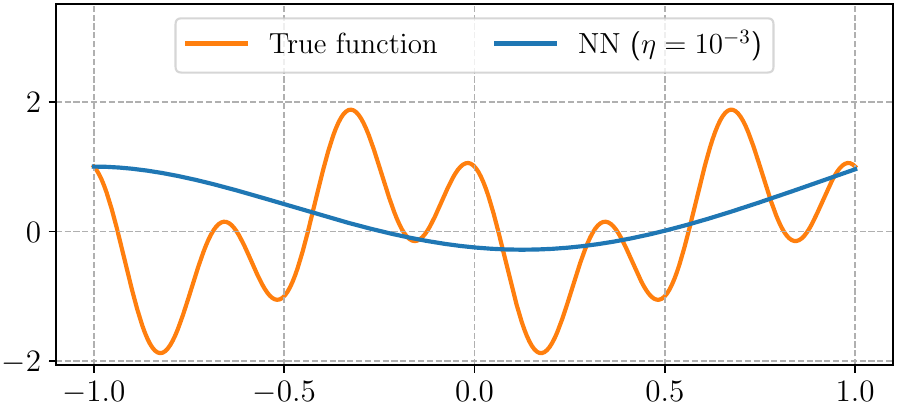}
    \end{subfigure}
    \hfill
    \begin{subfigure}[c]{0.2401\linewidth}
    \centering            \includegraphics[width=0.9985\textwidth]{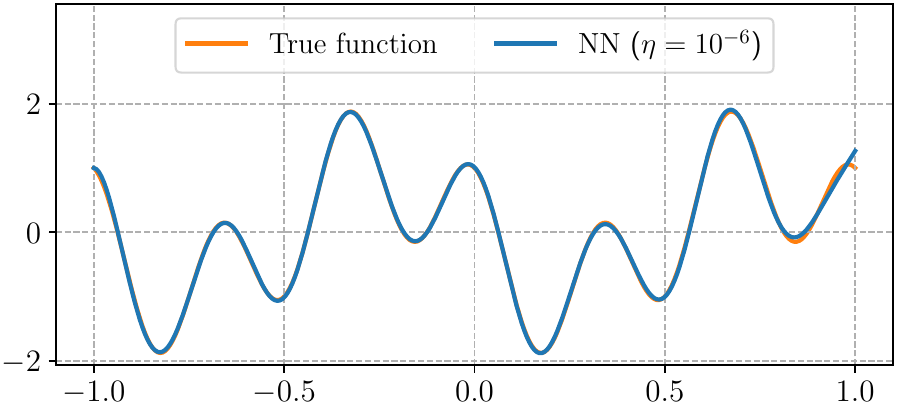}
    \end{subfigure}
    \hfill
    \begin{subfigure}[c]{0.2401\linewidth}
    \centering            \includegraphics[width=0.9985\textwidth]{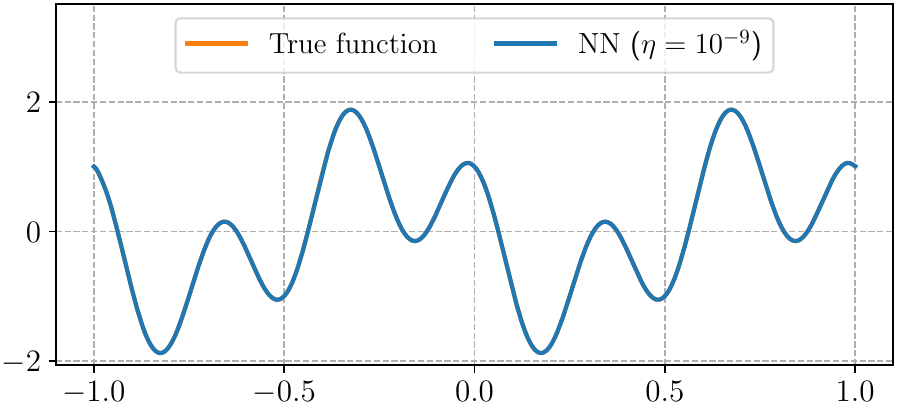}
    \end{subfigure}
    \\
    \vspace{5pt}
            \begin{subfigure}[c]{0.2401\linewidth}
    \centering            \includegraphics[width=0.9985\textwidth]{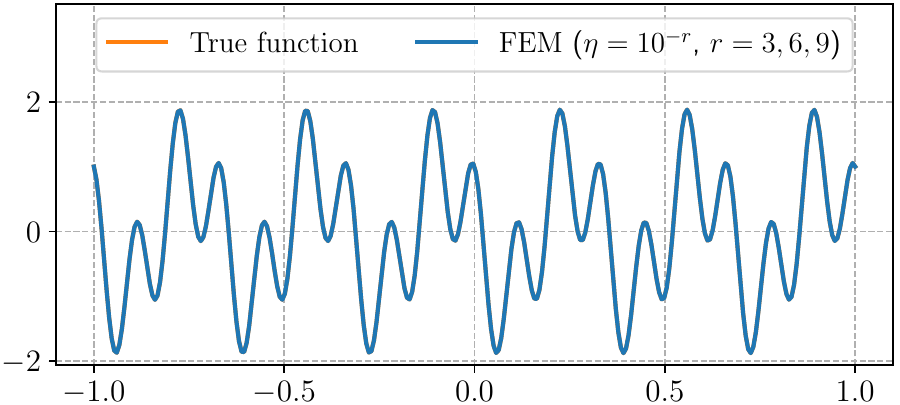}
    \end{subfigure}
    \hfill
    \begin{subfigure}[c]{0.2401\linewidth}
    \centering            \includegraphics[width=0.9985\textwidth]{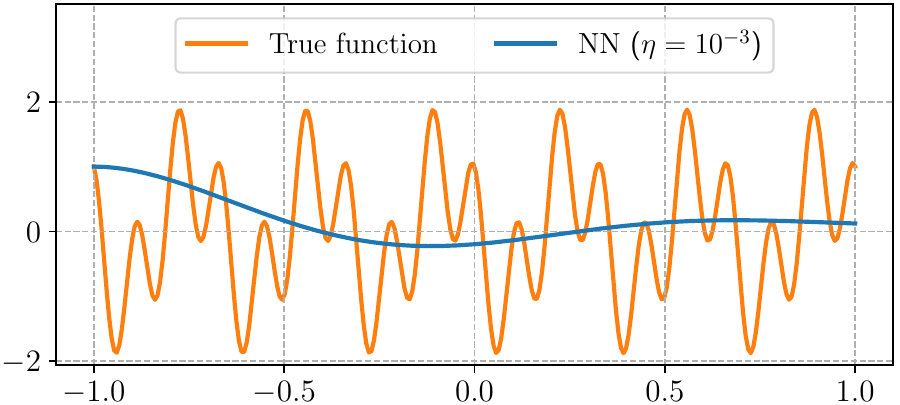}
    \end{subfigure}
    \hfill
    \begin{subfigure}[c]{0.2401\linewidth}
    \centering            \includegraphics[width=0.9985\textwidth]{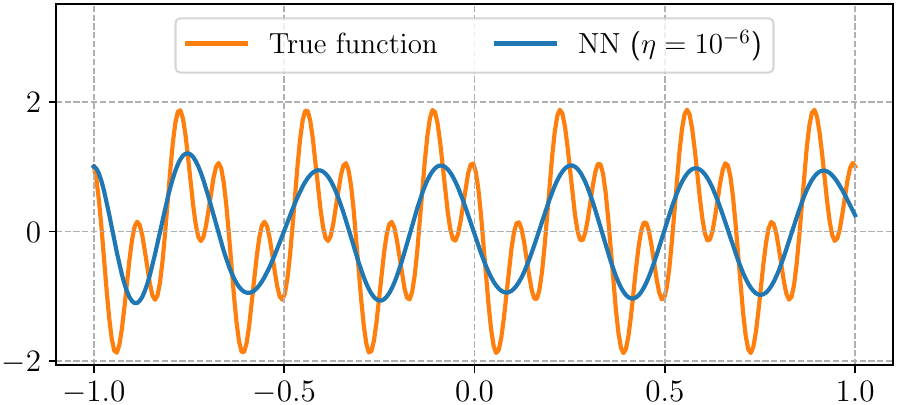}
    \end{subfigure}
    \hfill
            \begin{subfigure}[c]{0.2401\linewidth}
    \centering            \includegraphics[width=0.9985\textwidth]{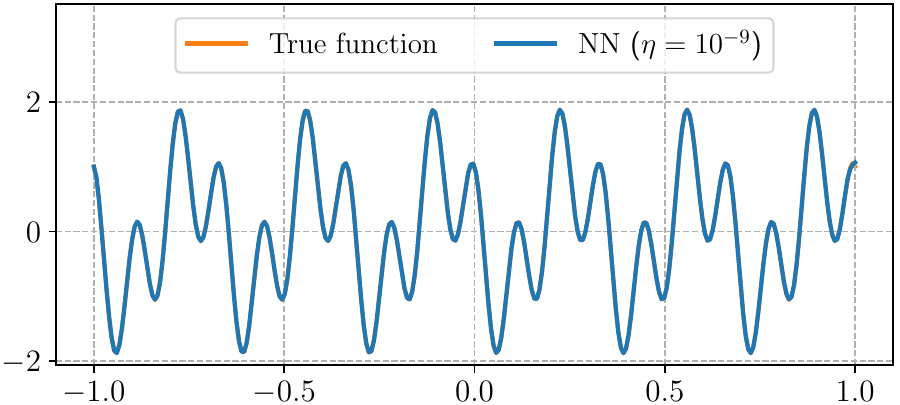}
    \end{subfigure}
    \\
    \vspace{5pt}
            \begin{subfigure}[c]{0.2401\linewidth}
    \centering            \includegraphics[width=0.9985\textwidth]{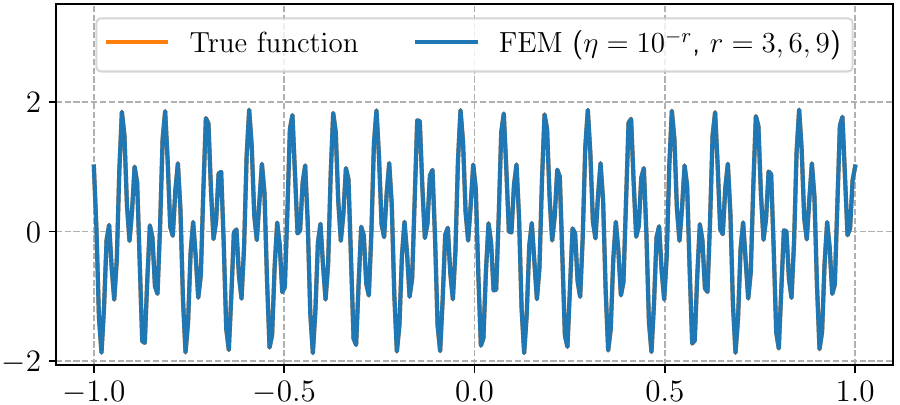}
    \end{subfigure}
    \hfill
    \begin{subfigure}[c]{0.2401\linewidth}
    \centering            \includegraphics[width=0.9985\textwidth]{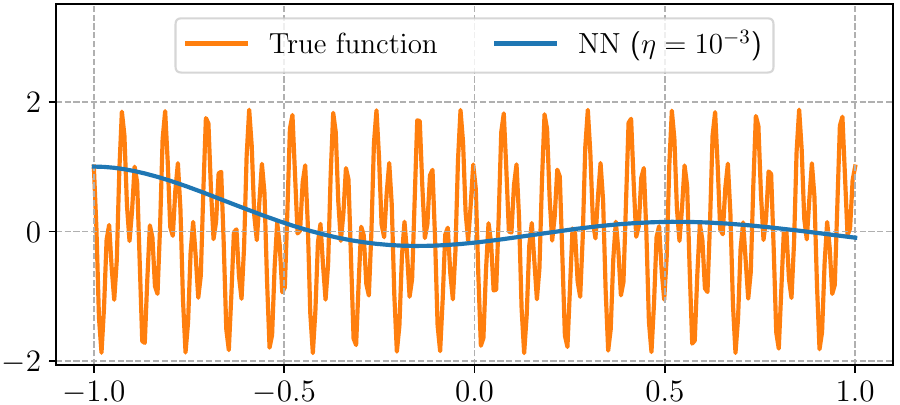}
    \end{subfigure}
    \hfill
    \begin{subfigure}[c]{0.2401\linewidth}
    \centering            \includegraphics[width=0.9985\textwidth]{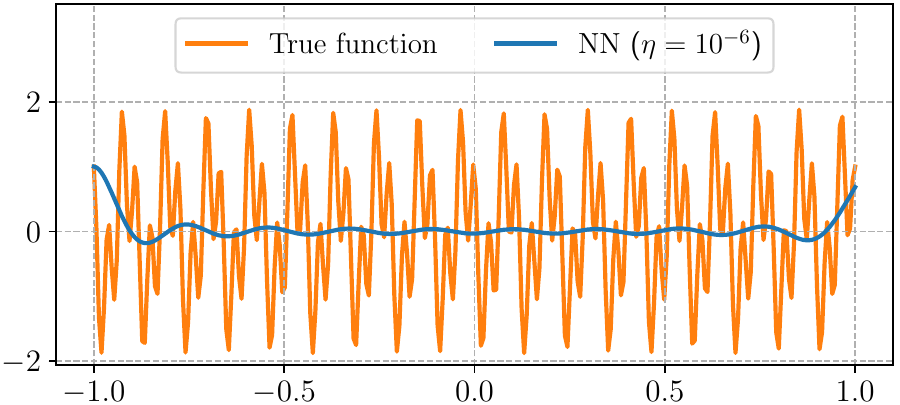}
    \end{subfigure}
    \hfill
            \begin{subfigure}[c]{0.2401\linewidth}
    \centering            \includegraphics[width=0.9985\textwidth]{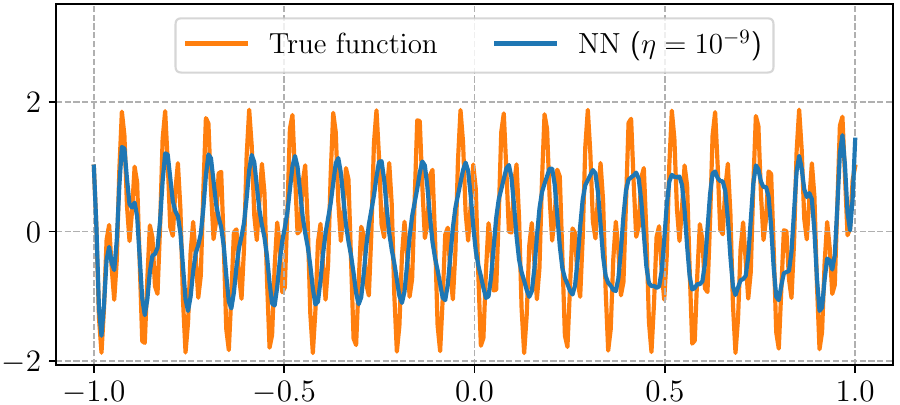}
    \end{subfigure}
    \caption{
    Approximation comparison: FEM vs. NN with $2000$ samples and $n=2000$ equal spaced basis.
    The three rows correspond to $f_1$, $f_2$, and $f_3$, respectively. Here $f_1(x)=f(x)$, $f_2(x)=f(3x)$,  $f_3(x)=f(9x)$, where $f(x)=\cos(6\pi x)-\sin(2\pi x)$.
    }
    	\label{fig:LMH:fre}
\end{figure}

\begin{table*}
		\caption{ Approximation errors in Figure~\ref{fig:LMH:fre}.} 
		\label{tab:LMH:fre}
		\vskip 0.05in
	\centering  	
	\resizebox{0.999\linewidth}{!}{ 
		\begin{tabular}{cccccccccccccccccccccccccccccc} 
			\toprule
  & \multicolumn{2}{c}{$f_1(x)=f(x)$}
   &\multicolumn{2}{c}{$f_2(x)=f(3x)$}
   &\multicolumn{2}{c}{$f_3(x)=f(9x)$}\\
   			\cmidrule(lr){2-3}
      \cmidrule(lr){4-5}
      \cmidrule(lr){6-7}


	
					& MAX	& MSE  & MAX
						& MSE & MAX & MSE
			\\
    \midrule
\rowcolor{mygray} 
  FEM (least square, $\eta=10^{-r}$, $r=3,6,9$)
 & $4.85\times10^{-5}$ & $5.34\times10^{-10}$ & $4.38\times10^{-4}$ & $4.32\times10^{-8}$ &  $3.95\times10^{-3}$ & $3.50\times10^{-6}$ \\
 \midrule
 NN (least square, $\eta=10^{-3}$) & $2.79$ & $1.28$ & $2.86$ & $1.19$ &
 $2.88$ & $1.15$ 
 \\ \midrule
\rowcolor{mygray}   NN (least square, $\eta=10^{-6}$) & $2.66\times10^{-1}$ & $9.44\times10^{-4}$ & $1.74$ & $5.39\times10^{-1}$ &  $2.79$ & $1.04$ 
 \\ \midrule
  NN (least square, $\eta=10^{-9}$) & $1.52\times10^{-2}$ & $1.45\times10^{-6}$ & $6.38\times10^{-2}$ & $1.92\times10^{-5}$ &
  $1.13$ & $4.65\times10^{-1}$ 
  \\
			\bottomrule
		\end{tabular} 
	}
\end{table*}


In the following experiment, we show the low-pass filter nature of two-layer neural networks (NN) vs. finite element methods (FEM) basis with respect to noise and overfitting (or over-parametrization). The target function is $f(x)=\cos(3\pi x)-\sin(\pi x)$ with noise sampled from $\calU(-0.5, 0.5)$ in our test. We manually select the cut-off ratio, $\eta$,  for small singular values. The numerical results are shown in Figure~\ref{fig:noise} and Table~\ref{tab:noise} for evenly distributed $b_i$. Since FEM has a condition number of $\cO(1)$, all modes resolved by the grid are captured independent of the cut-off ratio $\eta$.
On the other hand, NN only captures the leading eigenmodes, the number of which is determined by $\eta$. We can see that NN captures more modes as $\eta$ becomes smaller (less regularized). It is also interesting to see the low pass filter effect when the Adam optimizer is used to minimize the least square, which is related to the learning dynamics analysis in Section~\ref{sec:dynamics}.
In the case of 1000 data points and 1500 degrees of freedom, Figure~\ref{fig:overfitting} shows that a two-layer \texttt{ReLU} network is significantly more stable with respect to over-parametrization due to its low-pass filter nature.

\begin{figure}
    \centering
    \begin{subfigure}[c]{0.32488\textwidth}
    \centering            \includegraphics[width=0.9985\textwidth]{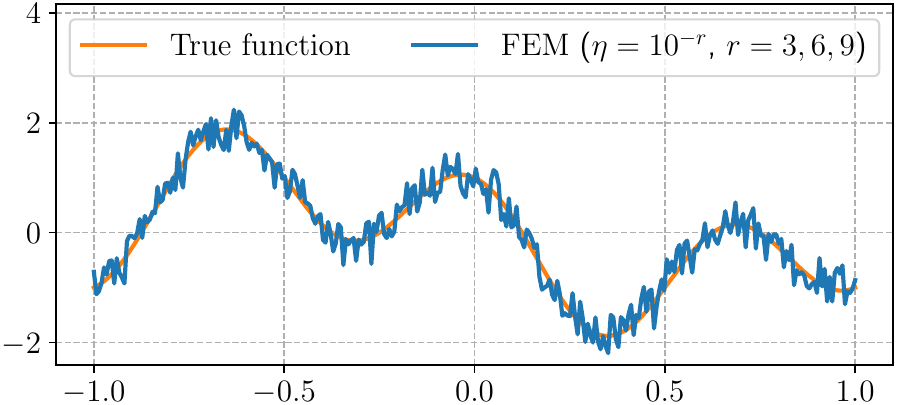}
    \end{subfigure}\hfill
    \begin{subfigure}[c]{0.32488\textwidth}
    \centering            \includegraphics[width=0.9985\textwidth]{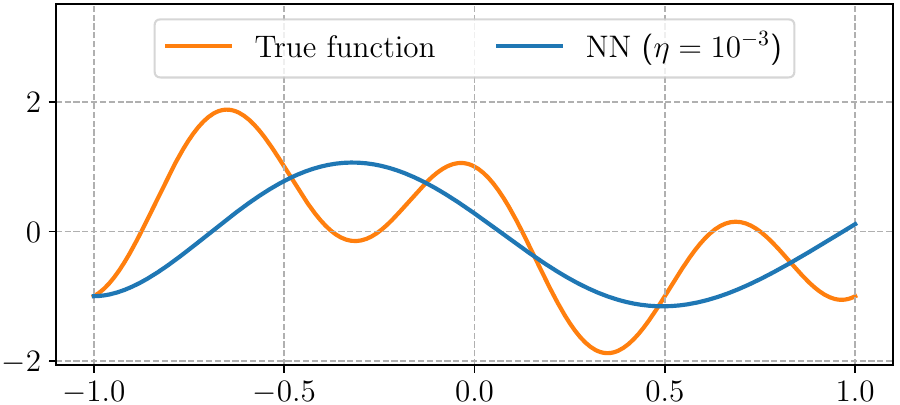}
    \end{subfigure}\hfill
    \begin{subfigure}[c]{0.32488\textwidth}
    \centering            \includegraphics[width=0.9985\textwidth]{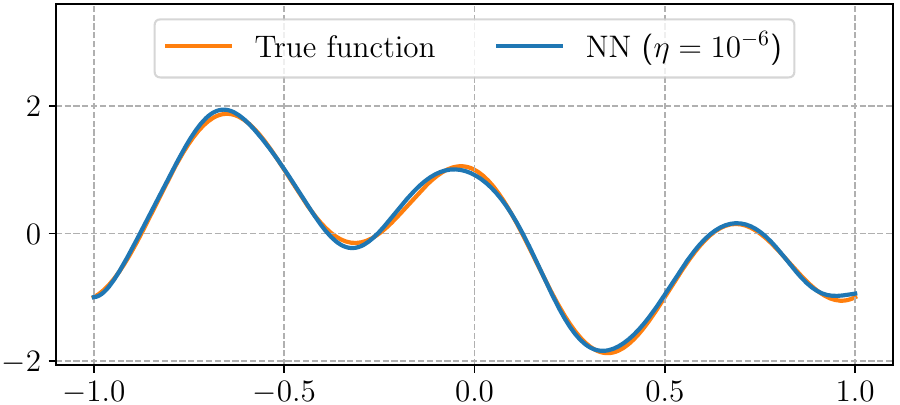}
    \end{subfigure}\\
    \vspace{3pt}
    \begin{subfigure}[c]{0.32488\textwidth}
    \centering            \includegraphics[width=0.9985\textwidth]{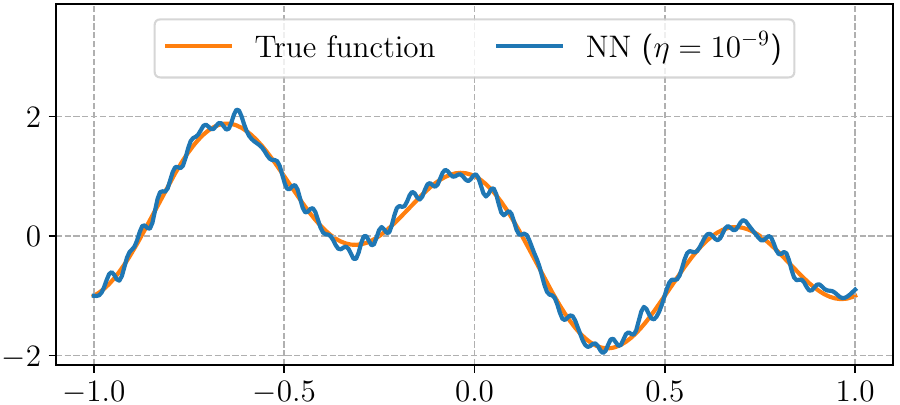}
    \end{subfigure}\hfill
    \begin{subfigure}[c]{0.32488\textwidth}
    \centering            \includegraphics[width=0.9985\textwidth]{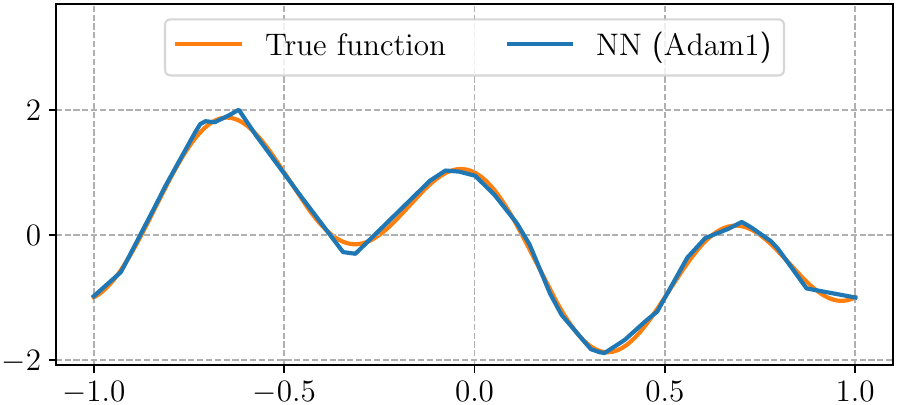}
    \end{subfigure}\hfill
    \begin{subfigure}[c]{0.32488\textwidth}
    \centering            \includegraphics[width=0.9985\textwidth]{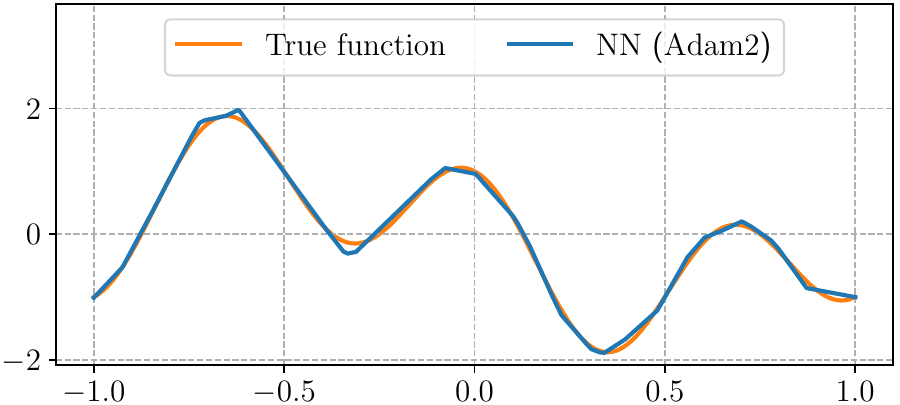}
    \end{subfigure}
    \caption{
    Approximation stability: FEM vs. NN subject to uniform noise $\calU(-0.5,0.5)$ on 1000 samples and $n=1000$ basis.
    The first four plots are results from the least square with different cut-off ratio 
    $\eta$ for singular values. 
    The last two plots are the results trained by the Adam optimizer, where ``Adam1'' and ``Adam2'' use single and double precision, respectively.    
    }
    	\label{fig:noise}
\end{figure}


\begin{table*}
		\caption{Approximation errors in Figure~\ref{fig:noise}.} 
		\label{tab:noise}
		\vskip 0.05in
	\centering  	
	\resizebox{0.999\linewidth}{!}{ 
		\begin{tabular}{cccccccccccccccccccccccccccccc} 
			\toprule
   &\multicolumn{4}{c}{least square}
   &\multicolumn{2}{c}{\hspace*{25pt}Adam optimizer\hspace*{25pt}}\\
   			\cmidrule(lr){2-5}
			\cmidrule(lr){6-7}	


    & FEM ($\eta=10^{-r}$, $r=3,6,9$) 
    &
    NN ($\eta=10^{-3}$) & NN ($\eta=10^{-6}$) &
    NN ($\eta=10^{-9}$) &  NN (float32)  &
       NN (float64)
	
			\\
    \midrule
\rowcolor{mygray} 
MAX
& 
$4.97\times10^{-1}$ 
&
$1.81$ & $1.15\times10^{-1}$ & $3.00\times10^{-1}$  & 
$1.73\times10^{-1}$  & $1.77\times10^{-1}$\\
\midrule
MSE
&
$5.90\times10^{-2}$
&
$7.60\times10^{-1}$ & $2.23\times10^{-3}$ & $9.48\times10^{-3}$ &
$3.68\times10^{-3}$ & $3.51\times10^{-3}$\\
			\bottomrule
		\end{tabular} 
	}
\end{table*}

\begin{figure}
    \centering	
    \begin{subfigure}[c]{0.3253\linewidth}
    \centering            \includegraphics[width=0.9985\textwidth]{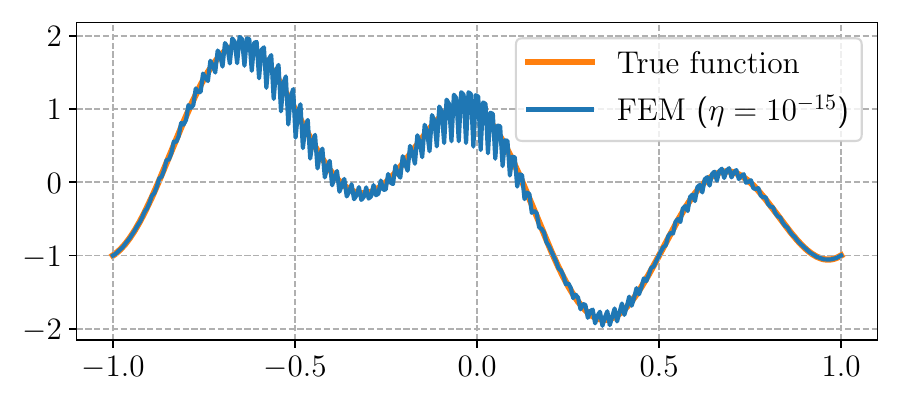}
    \end{subfigure}
    \hfill
    \begin{subfigure}[c]{0.3253\linewidth}
    \centering            \includegraphics[width=0.9985\textwidth]{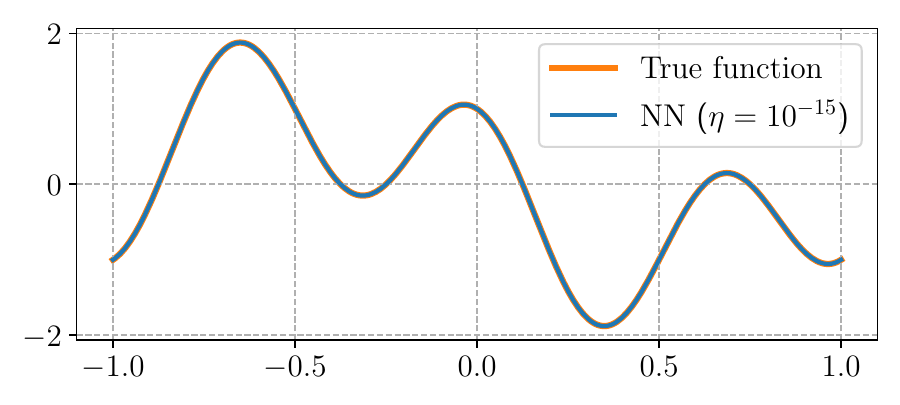}
    \end{subfigure}
    \hfill
    \begin{subfigure}[c]{0.3253\linewidth}
    \centering            \includegraphics[width=0.9985\textwidth]{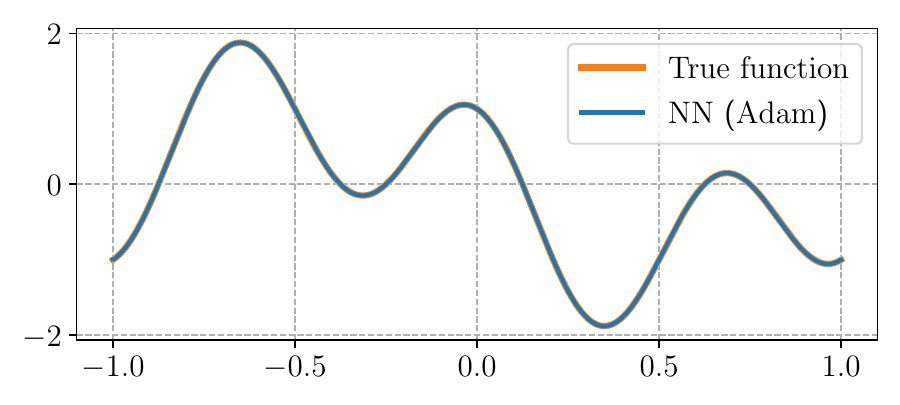}
    \end{subfigure}
    \caption{Test of overfitting: FEM vs. NN with $1000$ samples, $n=1500$ basis, and double precision. The first two plots are results from least square, where $\eta$ is the cut-off ratio for small singular values. The last plot is the result trained by the Adam optimizer.}
    	\label{fig:overfitting}
\end{figure}



\subsubsection{Spectrum and eigenmodes of Gram matrix in two dimensions}

We use a numerical example to show, Figure~\ref{fig:spectrum:2D}, the spectrum of the discrete Gram matrix in two dimensions with $25600$ evenly spaced samples for $( \bmw, b)\in \bbS^{1}\times[-1,1]$. The numerical experiments agree with our analysis in Section~\ref{sec:multiD} very well.
Several eigenmodes are presented in Figure~\ref{fig:eigenmodes:2D}.
In practice, those high-frequency modes whose corresponding eigenvalues are smaller than the machine precision threshold can not be captured.

\begin{figure}
            \centering    
        \includegraphics[width=0.8\textwidth]{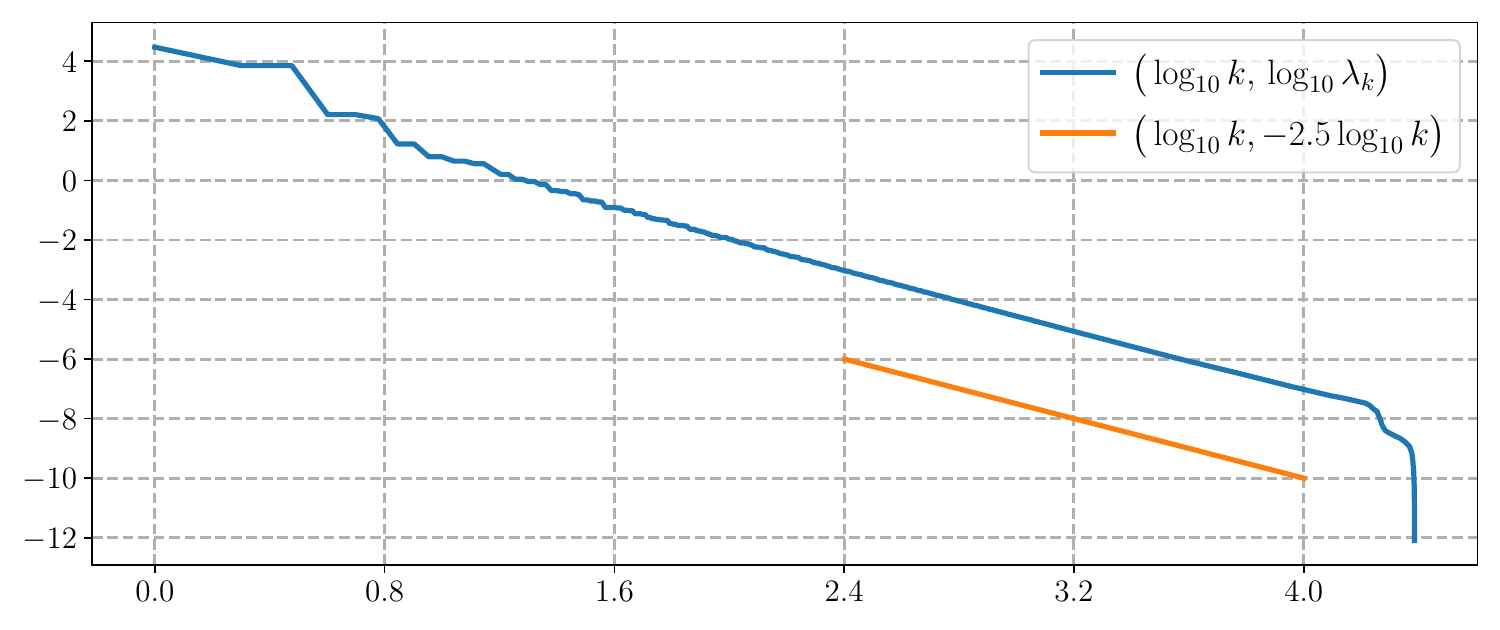}
            \caption{Spectrum of the discrete Gram matrix in two dimensions with $25600$  evenly spaced samples.}
    	\label{fig:spectrum:2D}
\end{figure}

\begin{figure}
    \centering	
    \begin{subfigure}[c]{0.1666\textwidth}
    \centering            \includegraphics[width=0.985\textwidth]{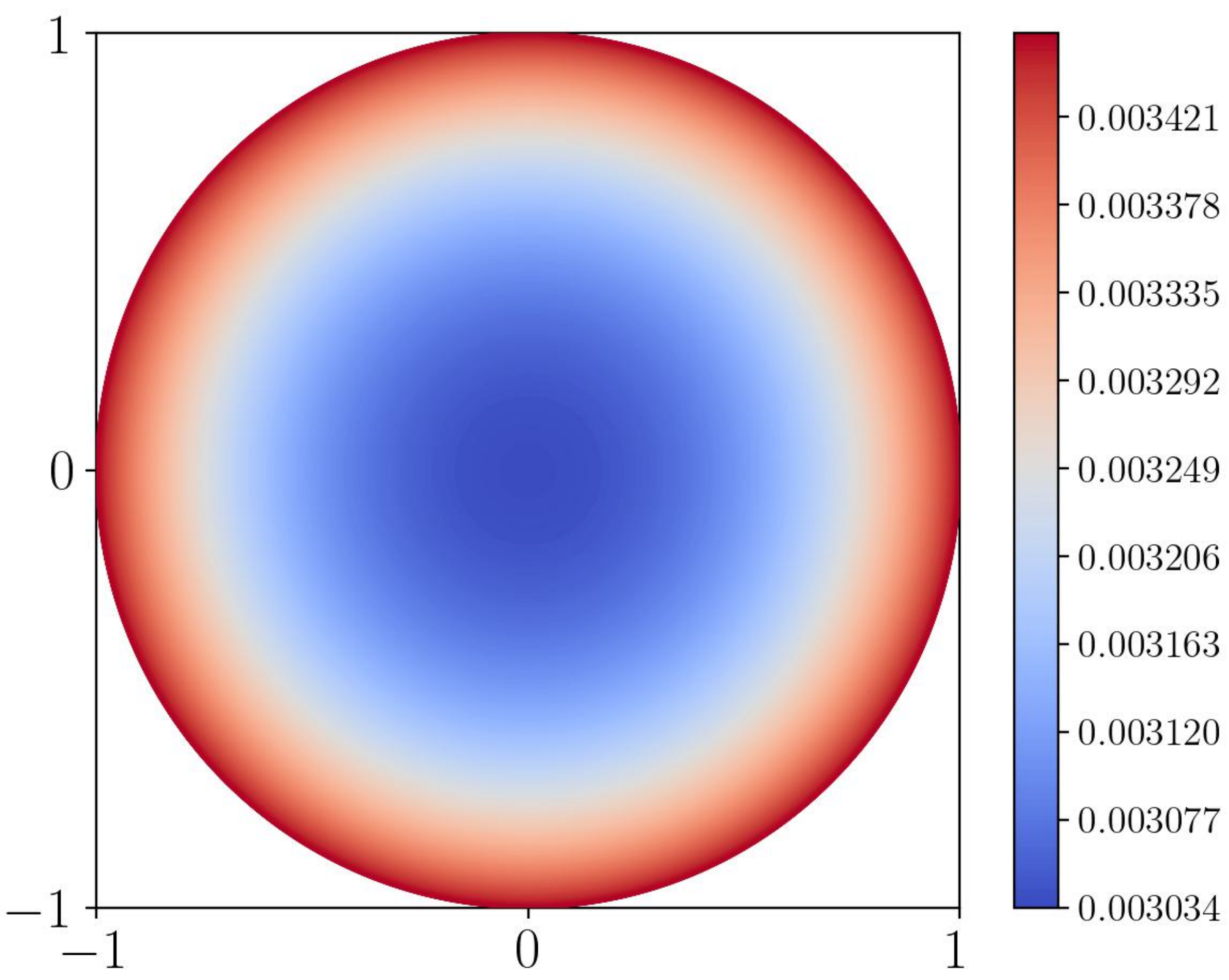}
    \subcaption{$\lambda_1$.}
    \end{subfigure}\hfill
    \begin{subfigure}[c]{0.1666\textwidth}
    \centering            \includegraphics[width=0.985\textwidth]{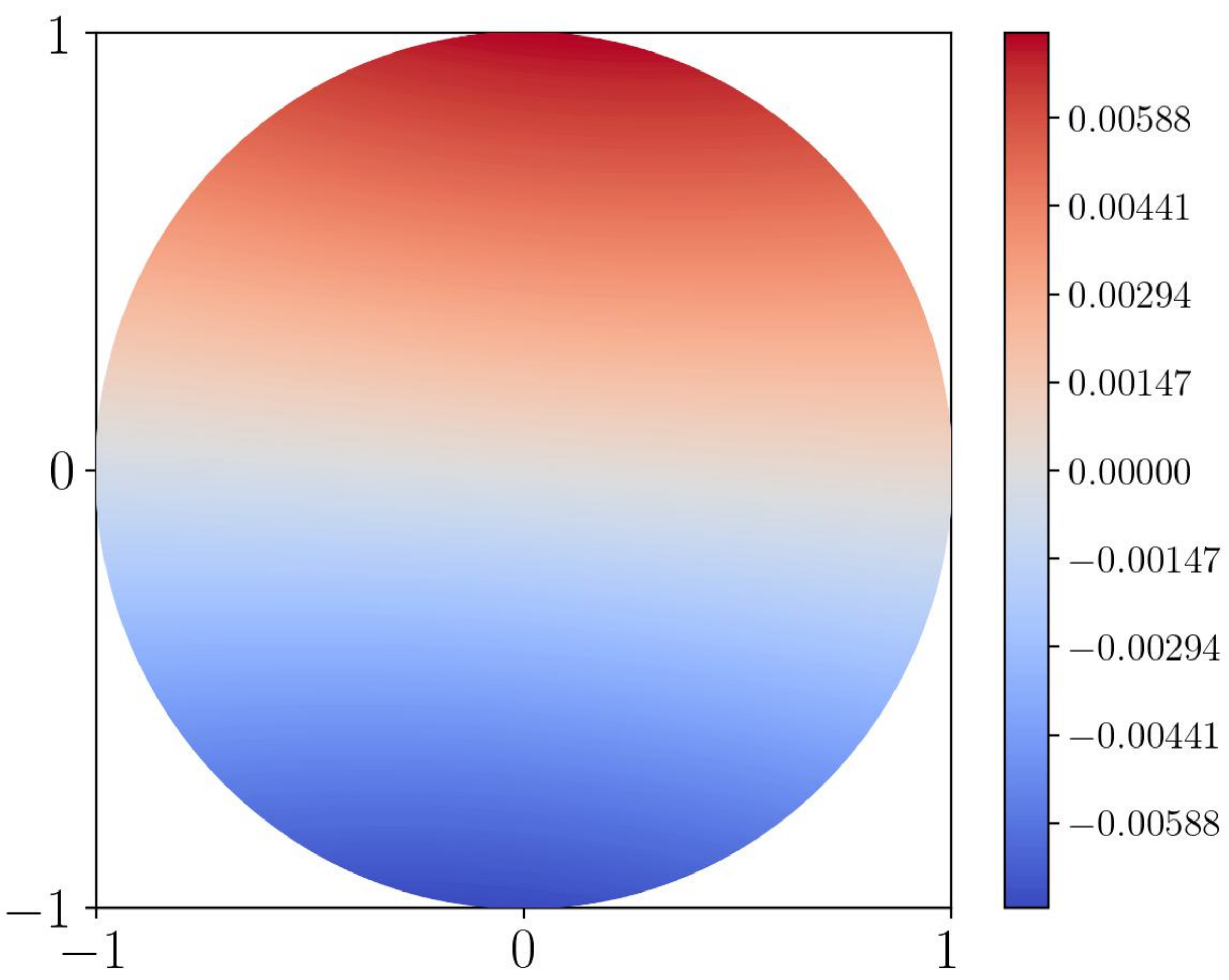}
    \subcaption{$\lambda_2$.}
    \end{subfigure}\hfill
    \begin{subfigure}[c]{0.1666\textwidth}
    \centering            \includegraphics[width=0.985\textwidth]{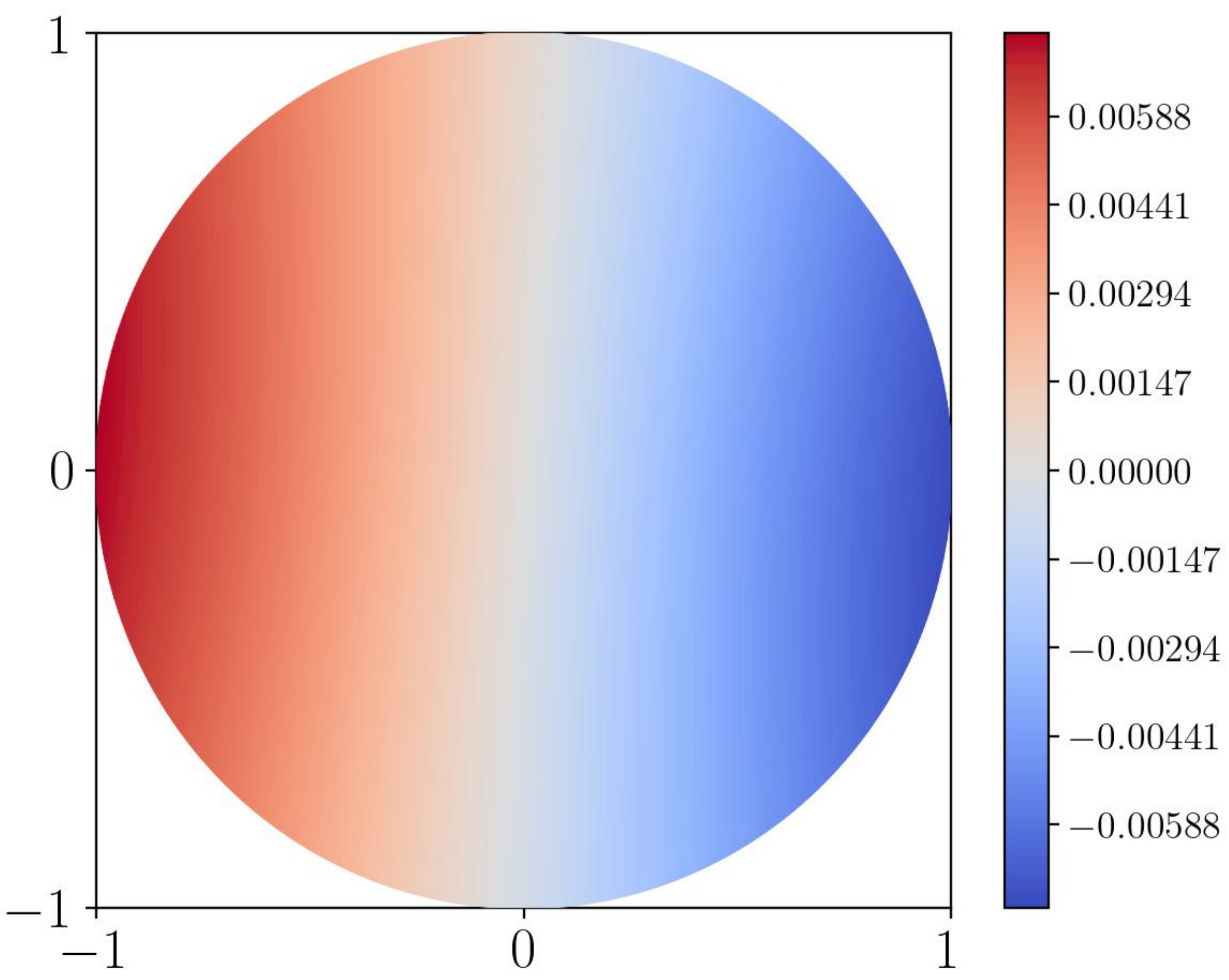}
    \subcaption{$\lambda_{3}$.}
    \end{subfigure}\hfill
    \begin{subfigure}[c]{0.1666\textwidth}
    \centering            \includegraphics[width=0.985\textwidth]{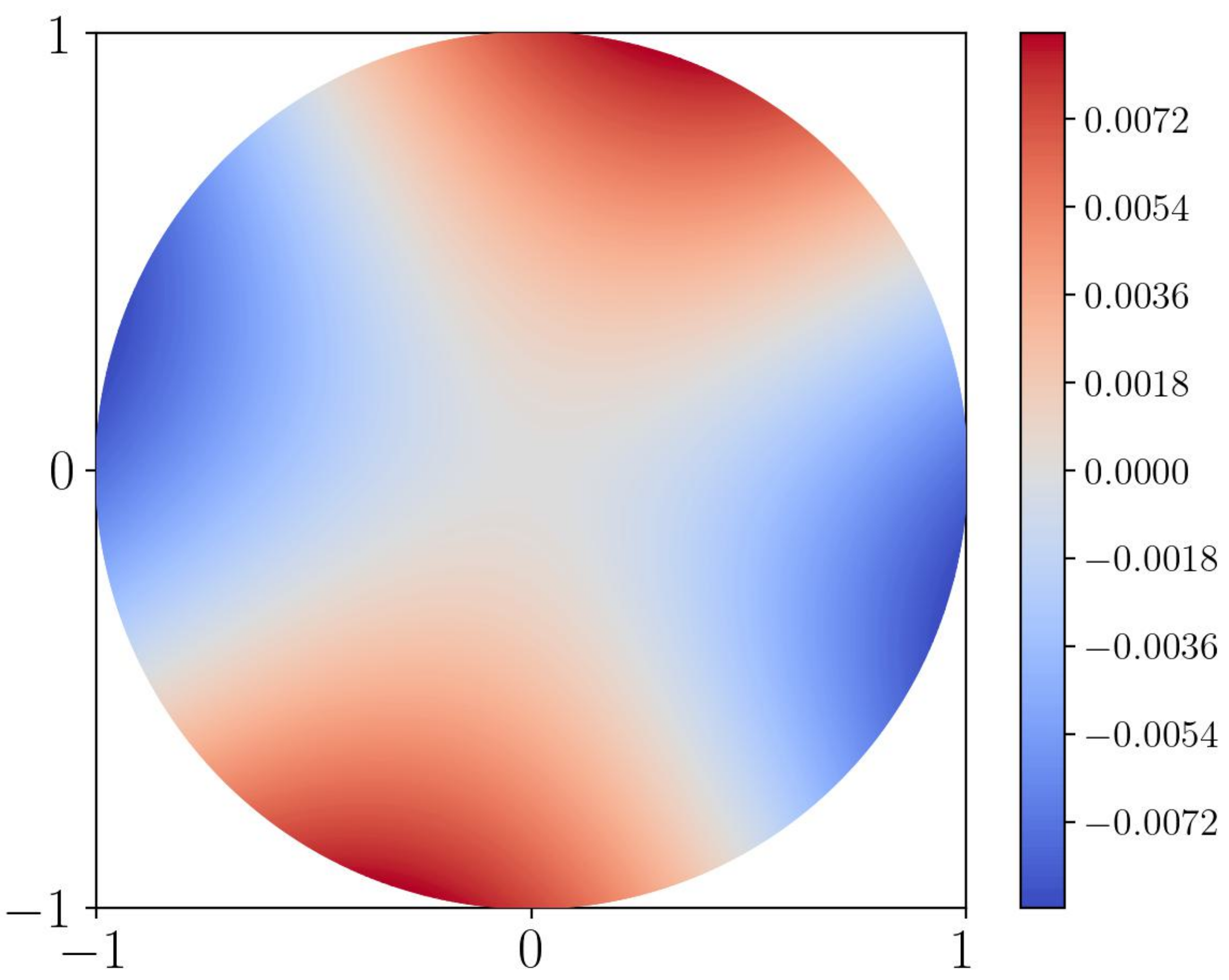}
    \subcaption{$\lambda_{4}$.}
    \end{subfigure}\hfill
    \begin{subfigure}[c]{0.1666\textwidth}
    \centering            \includegraphics[width=0.985\textwidth]{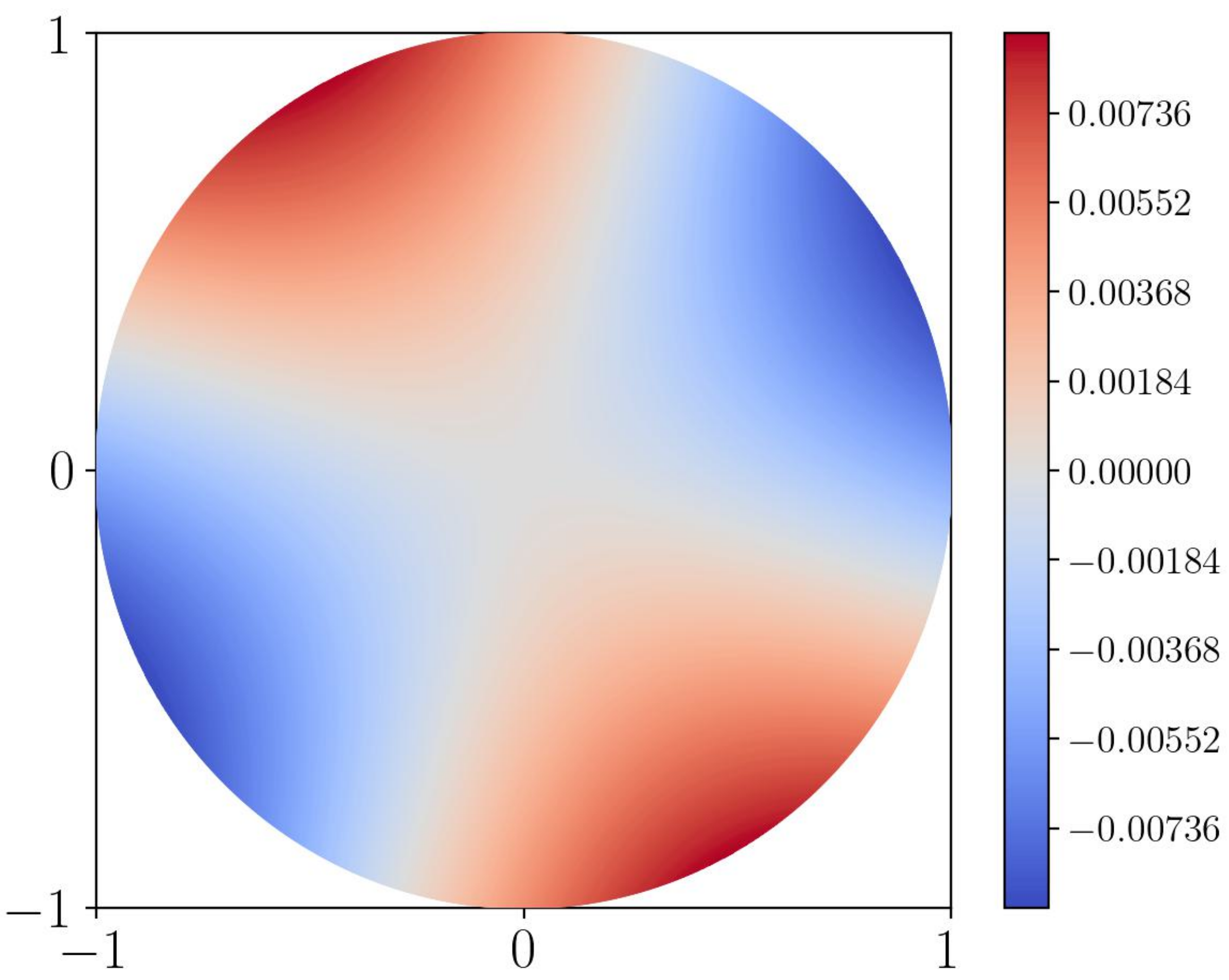}
    \subcaption{$\lambda_{5}$.}
    \end{subfigure}\hfill
    \begin{subfigure}[c]{0.1666\textwidth}
    \centering            \includegraphics[width=0.985\textwidth]{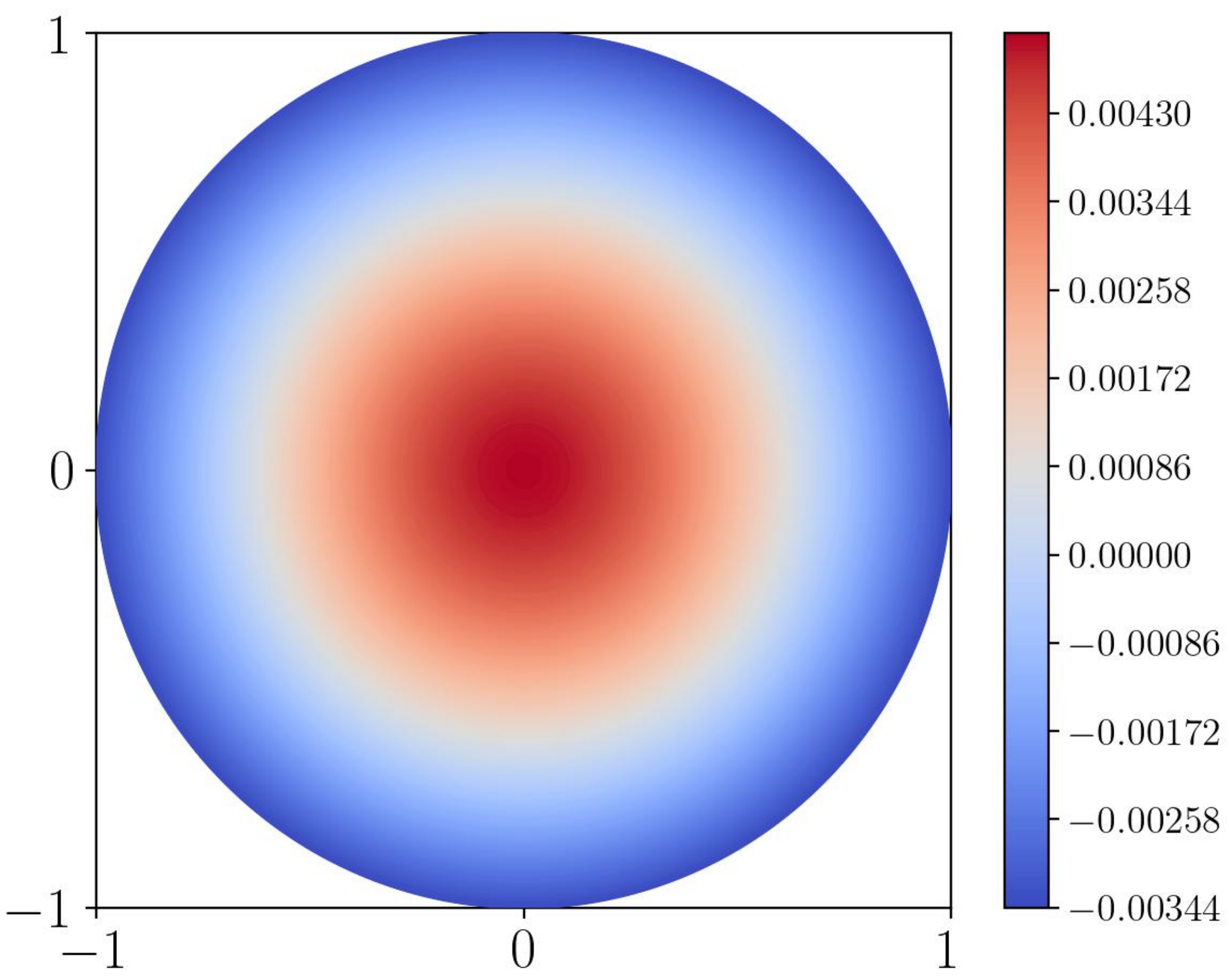}
    \subcaption{$\lambda_{6}$.}
    \end{subfigure}\\
     \begin{subfigure}[c]{0.1666\textwidth}
    \centering            \includegraphics[width=0.985\textwidth]{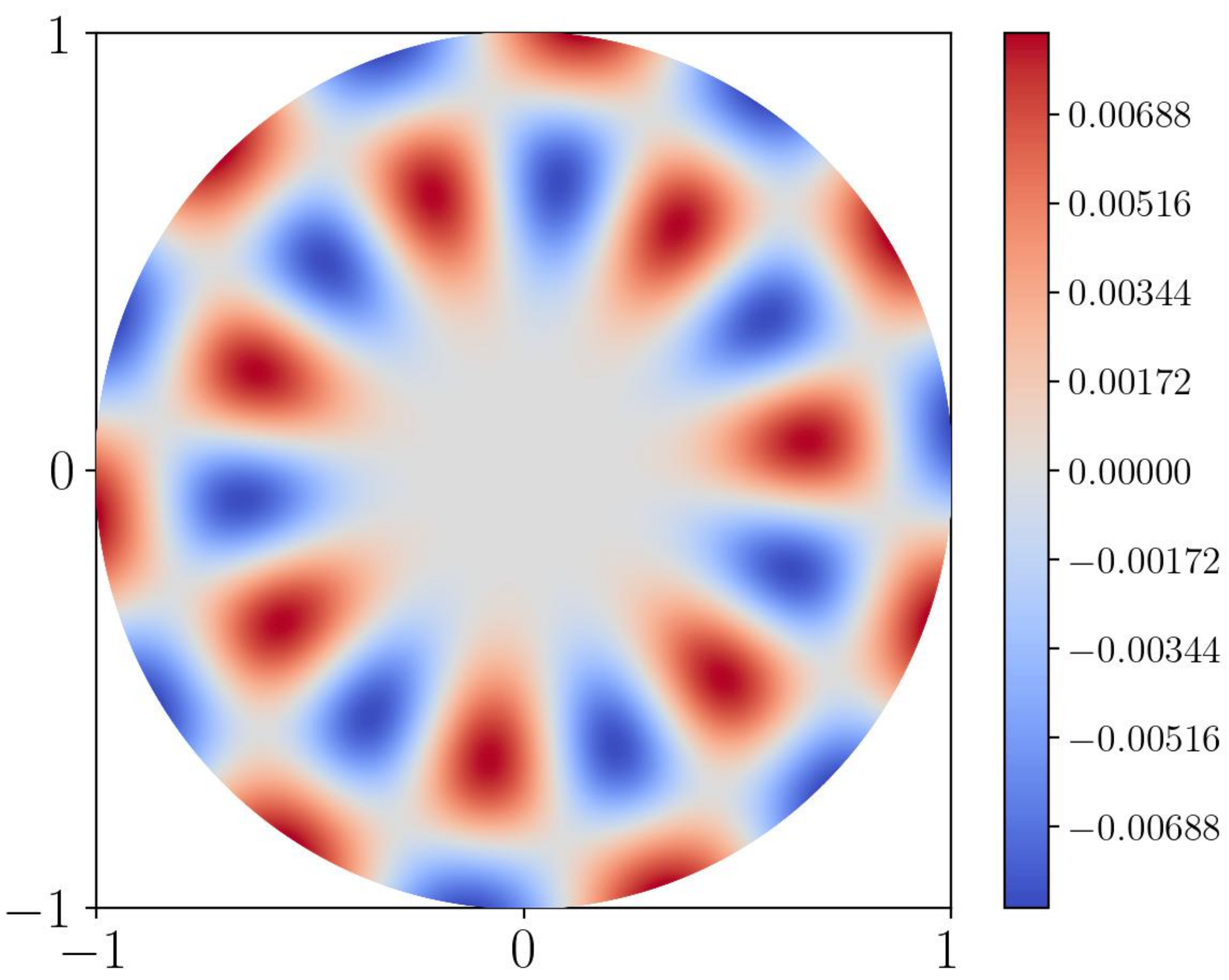}
    \subcaption{$\lambda_{50}$.}
    \end{subfigure}\hfill
    \begin{subfigure}[c]{0.1666\textwidth}
    \centering            \includegraphics[width=0.985\textwidth]{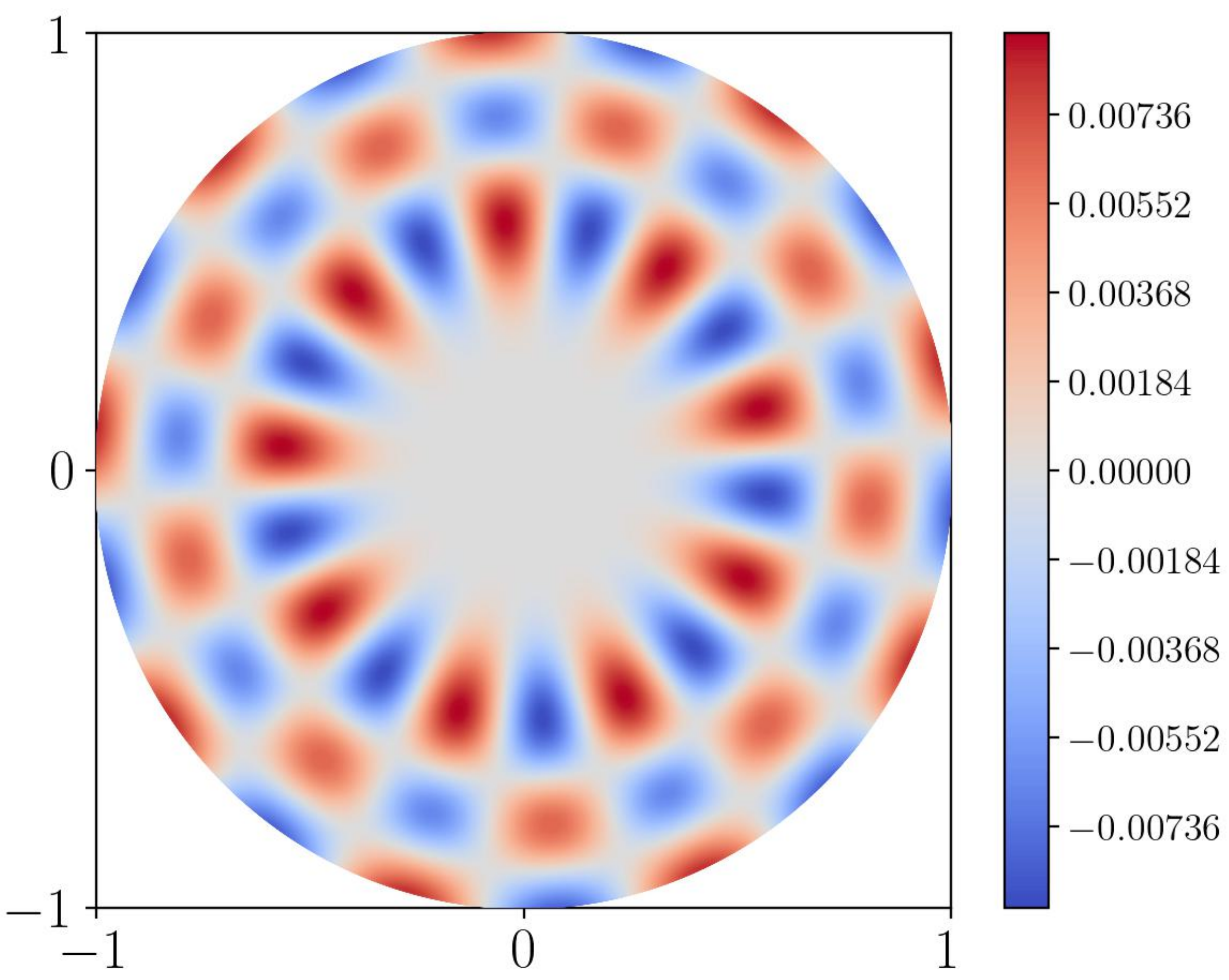}
    \subcaption{$\lambda_{100}$.}
    \end{subfigure}\hfill
    \begin{subfigure}[c]{0.1666\textwidth}
    \centering            \includegraphics[width=0.985\textwidth]{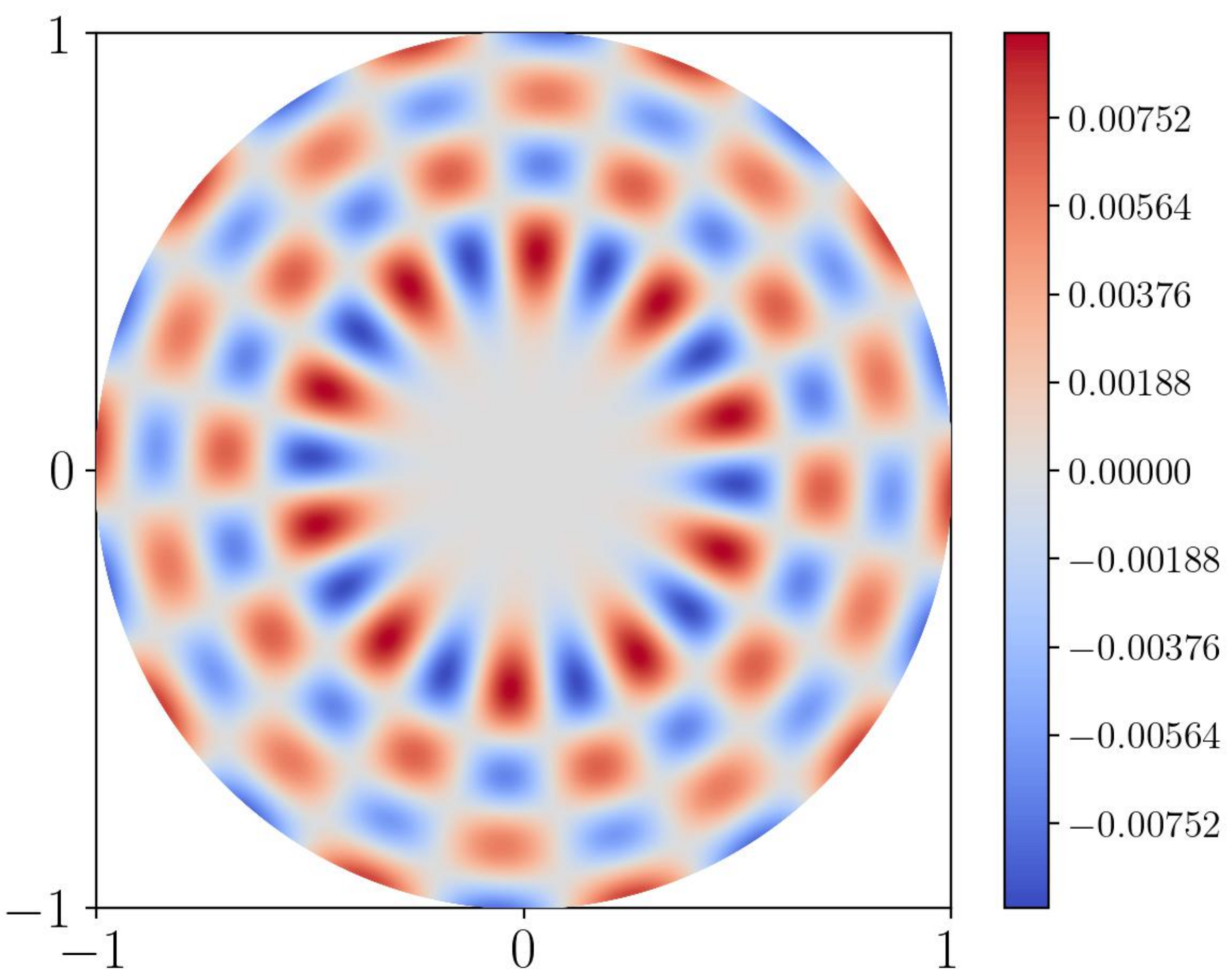}
    \subcaption{$\lambda_{150}$.}
    \end{subfigure}\hfill
    \begin{subfigure}[c]{0.1666\textwidth}
    \centering            \includegraphics[width=0.985\textwidth]{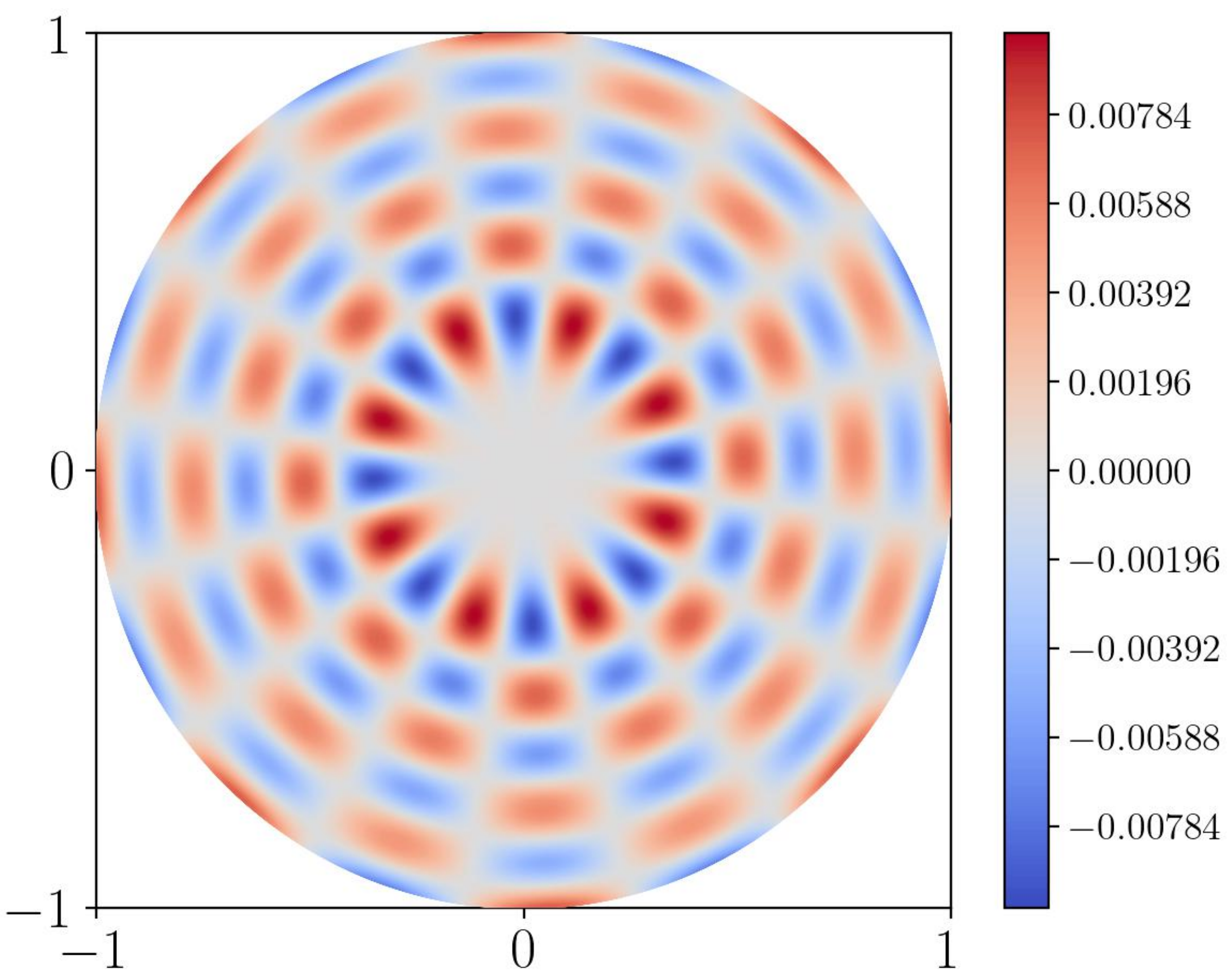}
    \subcaption{$\lambda_{200}$.}
    \end{subfigure}\hfill
    \begin{subfigure}[c]{0.1666\textwidth}
    \centering            \includegraphics[width=0.985\textwidth]{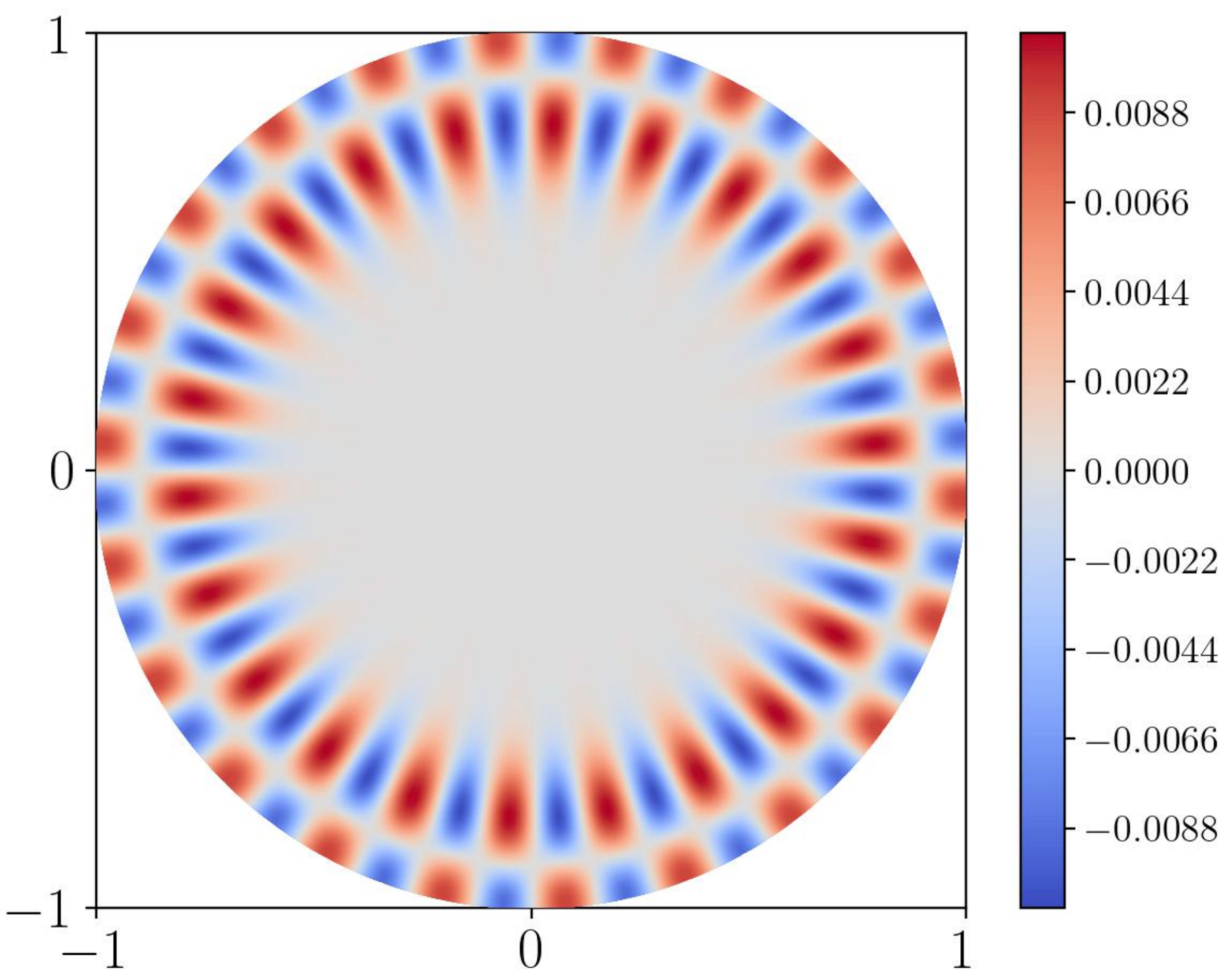}
    \subcaption{$\lambda_{250}$.}
    \end{subfigure}\hfill
    \begin{subfigure}[c]{0.1666\textwidth}
    \centering            \includegraphics[width=0.985\textwidth]{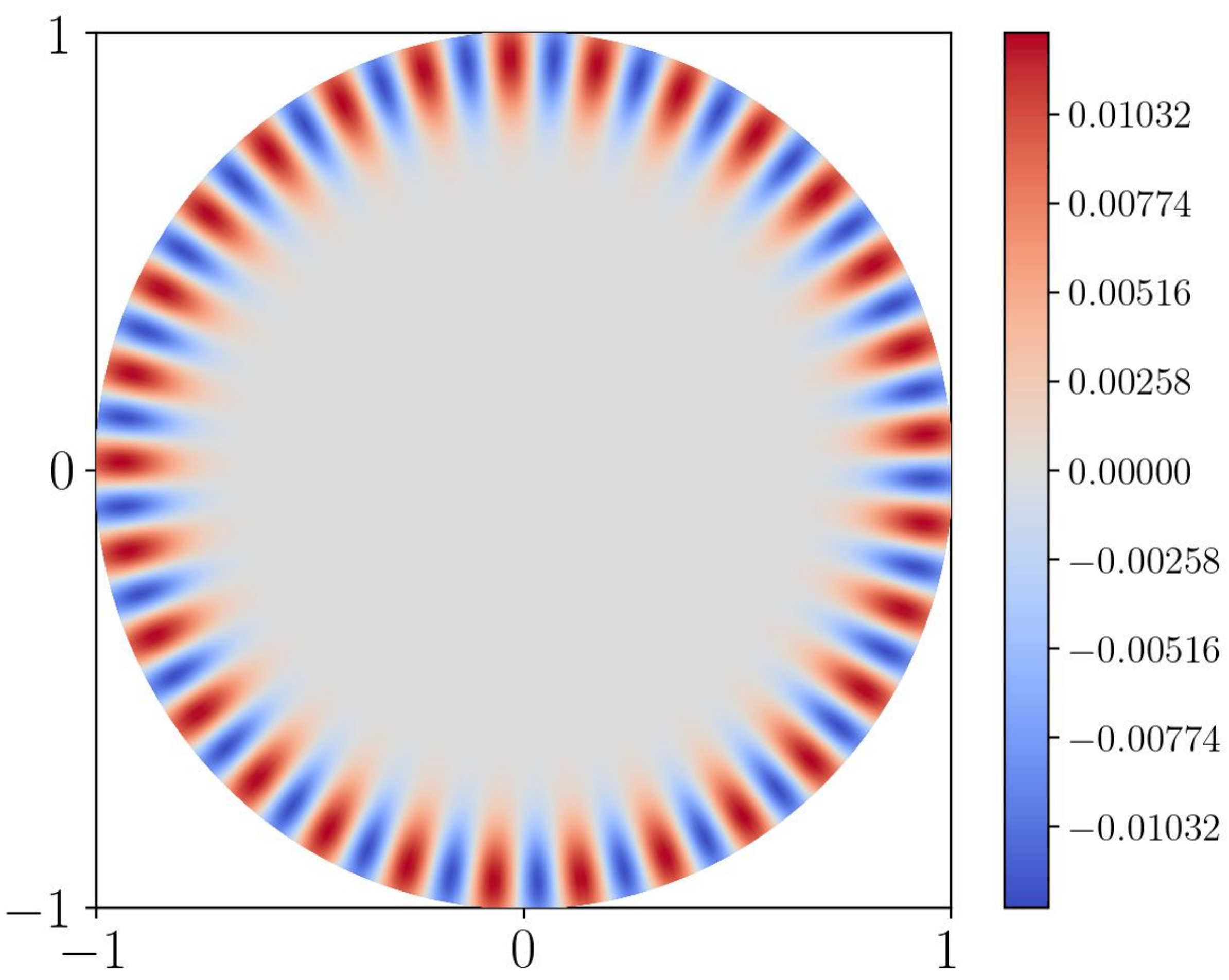}
    \subcaption{$\lambda_{300}$.}
    \end{subfigure}
    \caption{Eigenmodes of $\lambda_k$ 
    for $k\in\{1,2,3,4,5,6,50,100,150,200,250,300\}$ 
    with $n=25600$.}
    	\label{fig:eigenmodes:2D}
\end{figure}

\subsection{Different activation functions and scaled parameter initialization}
\label{sec:act-comp}

Here we demonstrate the behaviors and their comparison for two-layer neural networks using different activation functions, denoted generically by $\sigma$. The activation functions compared here are \texttt{ReLU, sine}, and \texttt{tanh} and its first and second derivatives $\texttt{tanh}'$, $\texttt{tanh}''$. We limit our discussions to two-layer networks in one dimension. Figure~\ref{fig:acts}(a) plots these activation functions. 

A general two-layer network can be regarded as a parametrized function, denoted by $h(x;\bma,\bmw,\bmb)$, represented as the linear combination of a set of activation (or basis) functions parametrized by $\bmw$ and $\bmb$:
\begin{equation}
\label{eq:standard}
h(x;\bma,\bmw,\bmb) = \sum_{i=1}^n a_i \sigma(w_i x + b_i).
\end{equation}
As the default setting, we set $x\in [-1,1]$ and initialize  $w_i, b_i \sim \mathcal{U}(-1,1)$ and $a_i \sim \mathcal{U}(-1/\sqrt{n}  1/\sqrt{n})$ uniformly distributed. 

As discussed above, the conditioning of a two-layer network representation~\eqref{eq:standard} is determined by the spectrum of the Gram Matrix $\bmG$, where
\[
\bmG_{i,j}=\int_{-1}^1\sigma(w_ix+b_i)\sigma(w_jx+b_j) dx, \quad w_i, b_i \sim \mathcal{U}(-1,1).
\]
The logarithmic plot of the spectra of the Gram matrices corresponding to different activation functions are shown in Figure~\ref{fig:acts}(b).
Except \texttt{ReLU}, the Gram matrix of which has a polynomial decay as shown above,  all other activation functions are analytic and hence their corresponding Gram matrices have exponential spectral decay \cite{reade1983eigenvalues,reade1984eigenvalues}.  
\begin{figure}
    \centering	
        \begin{subfigure}[c]{0.4833\linewidth}
    \centering            \includegraphics[width=0.9985\textwidth]{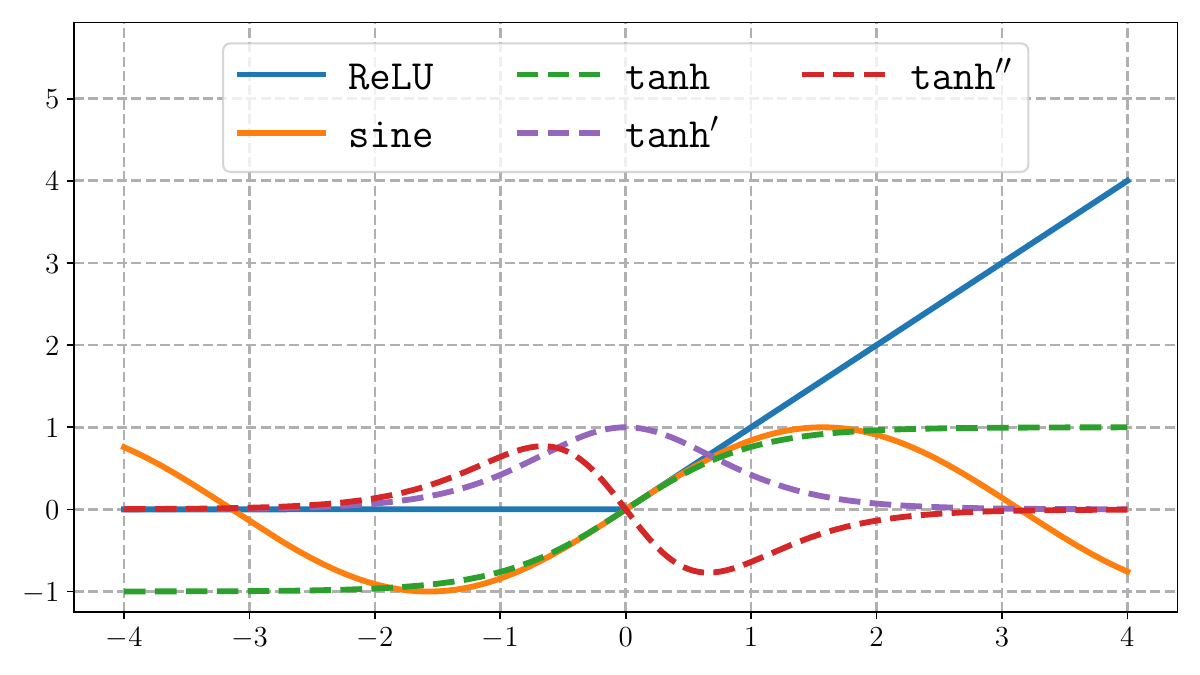}
    \subcaption{Activation functions.}
    \end{subfigure}
    \hfill
    \begin{subfigure}[c]{0.4833\linewidth}
    \centering            \includegraphics[width=0.9985\textwidth]{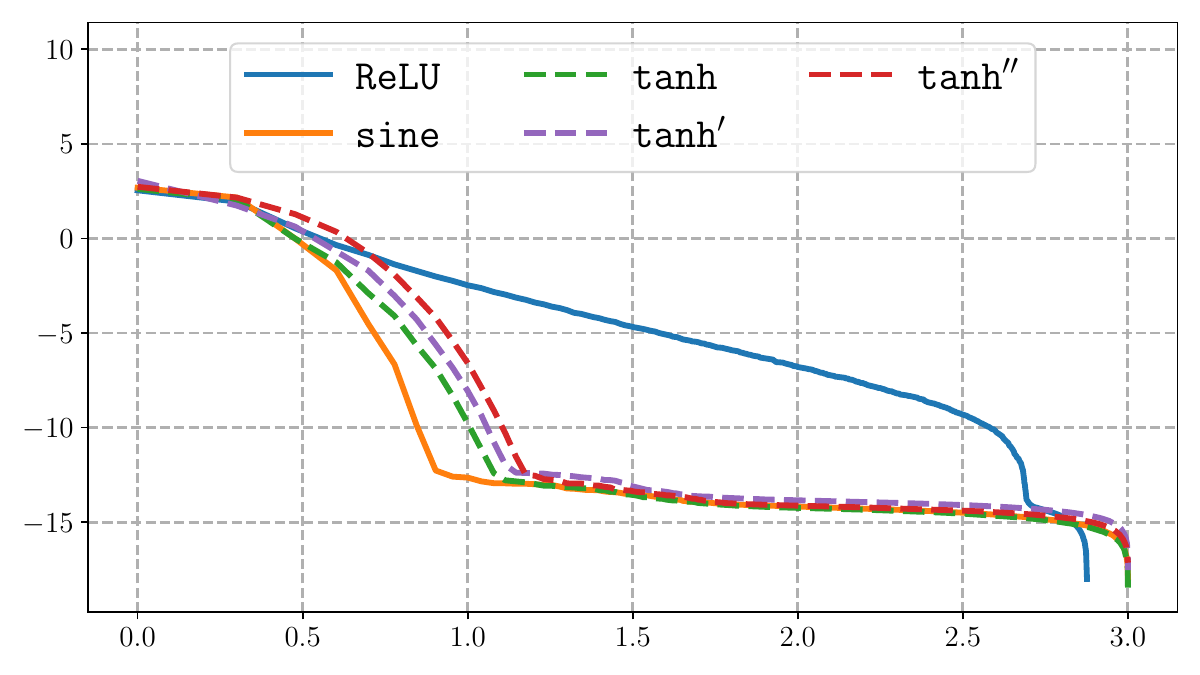}
    \subcaption{Spectrum of
    $\bmG$
    with $\bmw,\bmb\in \calU(-1,1)$.}
    \end{subfigure}\hfill  
    \caption{Illustrations of different activation functions and the corresponding sepctrum.}
    	\label{fig:acts}
\end{figure}

Figure~\ref{fig:act-comp} shows the approximation of $f(x)=\frac{1}{1+3600(x-0.2)^2}$ on $[-1,1]$ using two-layer networks~\eqref{eq:standard} 
with different activation functions. The networks have a width $n=512$ and computations are implemented in double precision. The first row shows the least square approximation with fixed $\bmw,\bmb \in \mathcal{U}(-1,1)$ and solving the linear system for $\bma$ with Gram matrix $\bmG$. 
The second row presents the approximation results obtained by training with Adam, where the parameters \((a_i, w_i, b_i)\) are optimized starting from uniform initialization: \(\bma \sim \mathcal{U}(-1/\sqrt{n}, 1/\sqrt{n})\), and \(\bmw, \bmb \sim \mathcal{U}(-1, 1)\).
Based on the spectral analysis, \texttt{ReLU} provides the best approximation due to its slowest spectral decay, or equivalently, its least bias against high frequencies among this group of activation functions. Among the remaining activation functions, $\texttt{tanh}''$ yields better results owing to its relatively slower spectral decay.


\begin{figure}
    \centering	
        \begin{subfigure}[c]{0.192555\linewidth}
    \centering            \includegraphics[width=0.9985\textwidth]{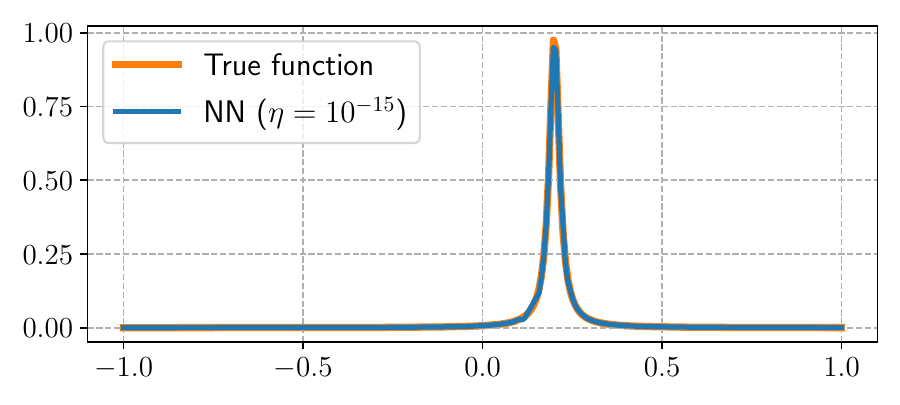}
    \subcaption{\texttt{ReLU}.}
    \end{subfigure}
    \hfill
        \begin{subfigure}[c]{0.192555\linewidth}
    \centering            \includegraphics[width=0.9985\textwidth]{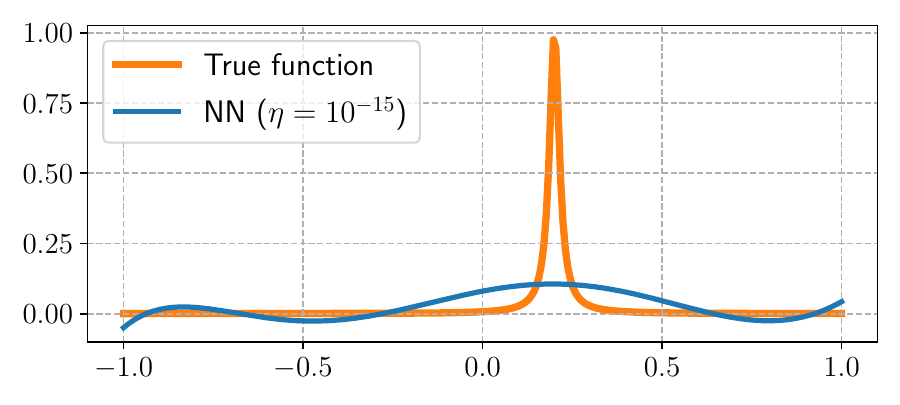}
    \subcaption{\texttt{sin}.}
    \end{subfigure}
    \hfill        \begin{subfigure}[c]{0.192555\linewidth}
    \centering            \includegraphics[width=0.9985\textwidth]{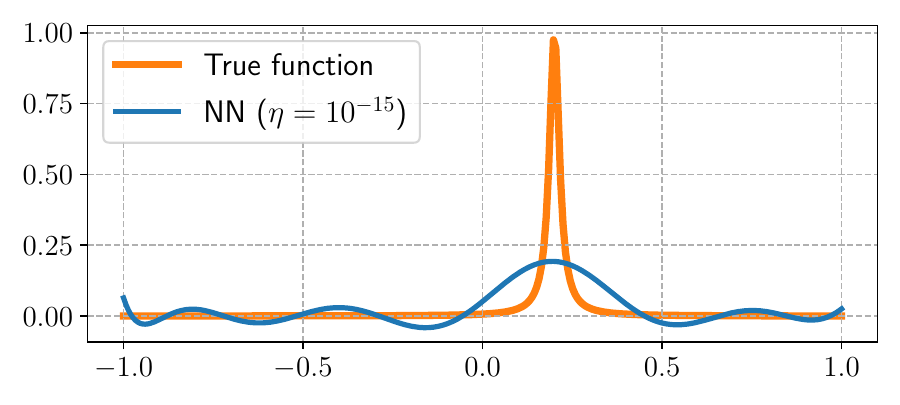}
    \subcaption{\texttt{tanh}.}
    \end{subfigure}
    \hfill        \begin{subfigure}[c]{0.192555\linewidth}
    \centering            \includegraphics[width=0.9985\textwidth]{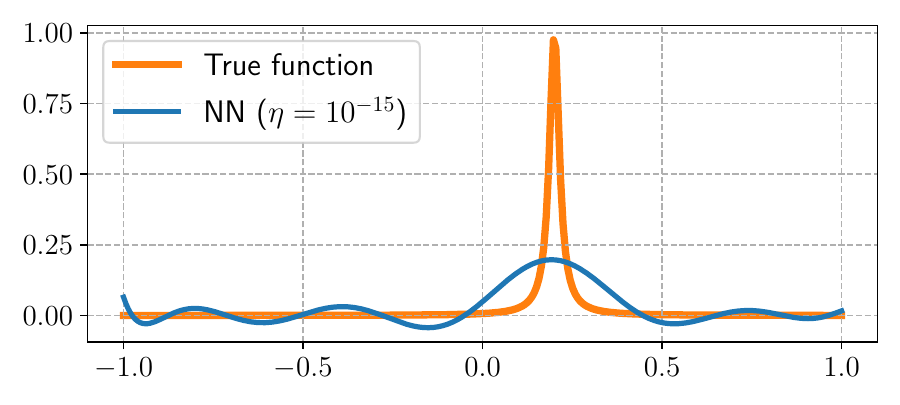}
    \subcaption{$\texttt{tanh}^{\prime}$.}
    \end{subfigure}
    \hfill        \begin{subfigure}[c]{0.192555\linewidth}
    \centering            \includegraphics[width=0.9985\textwidth]{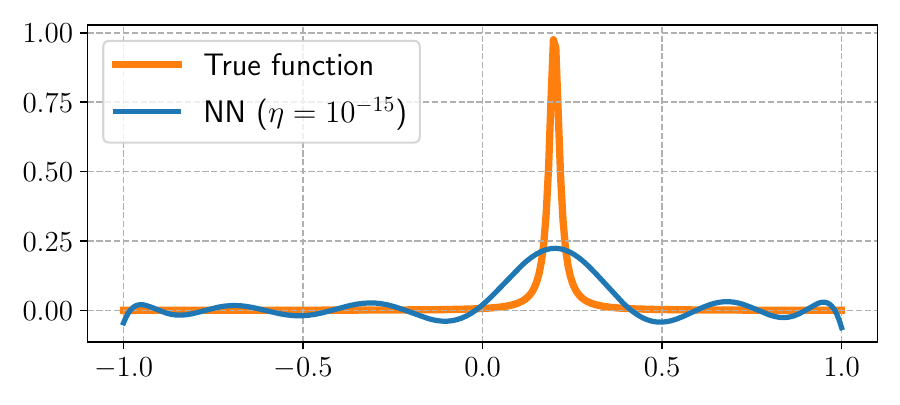}
     \subcaption{$\texttt{tanh}^{\prime\prime}$.}
    \end{subfigure}
\\
        \begin{subfigure}[c]{0.192555\linewidth}
    \centering            \includegraphics[width=0.9985\textwidth]{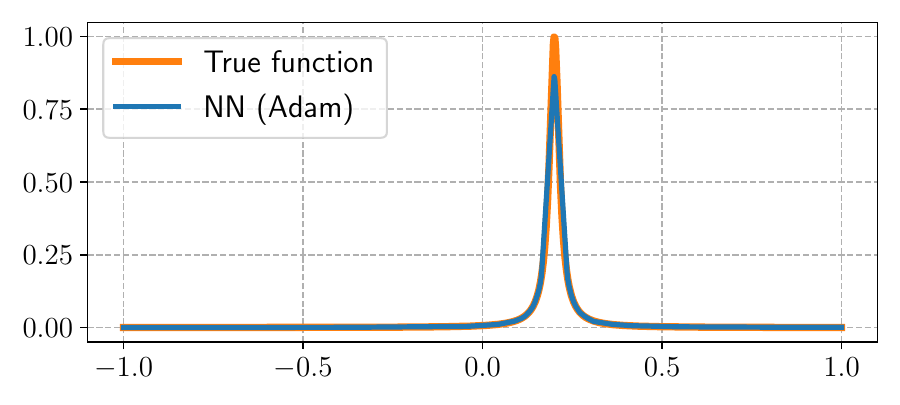}
    \subcaption{\texttt{ReLU}.}
    \end{subfigure}
    \hfill
        \begin{subfigure}[c]{0.192555\linewidth}
    \centering            \includegraphics[width=0.9985\textwidth]{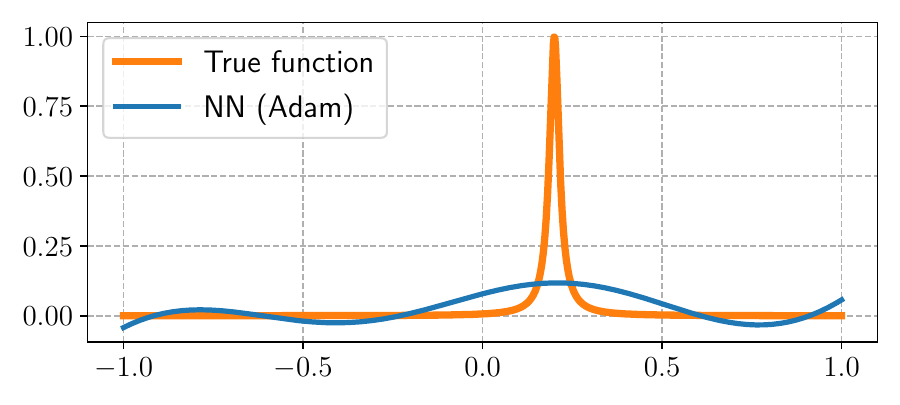}
    \subcaption{\texttt{sin}.}
    \end{subfigure}
    \hfill        \begin{subfigure}[c]{0.192555\linewidth}
    \centering            \includegraphics[width=0.9985\textwidth]{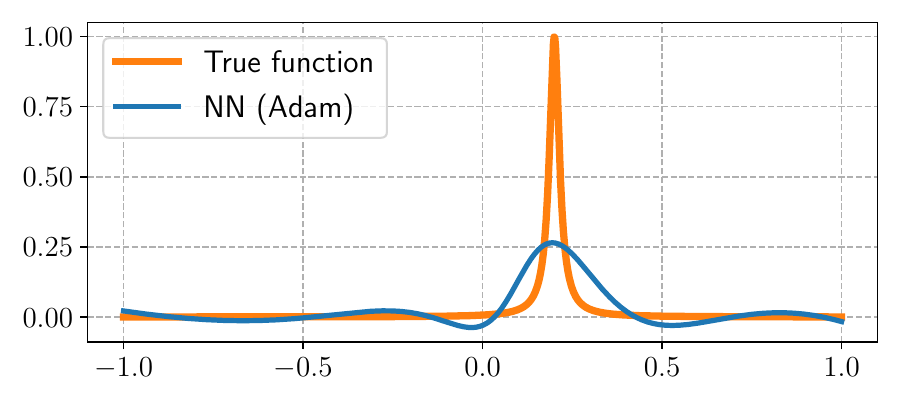}
    \subcaption{\texttt{tanh}.}
    \end{subfigure}
    \hfill        \begin{subfigure}[c]{0.192555\linewidth}
    \centering            \includegraphics[width=0.9985\textwidth]{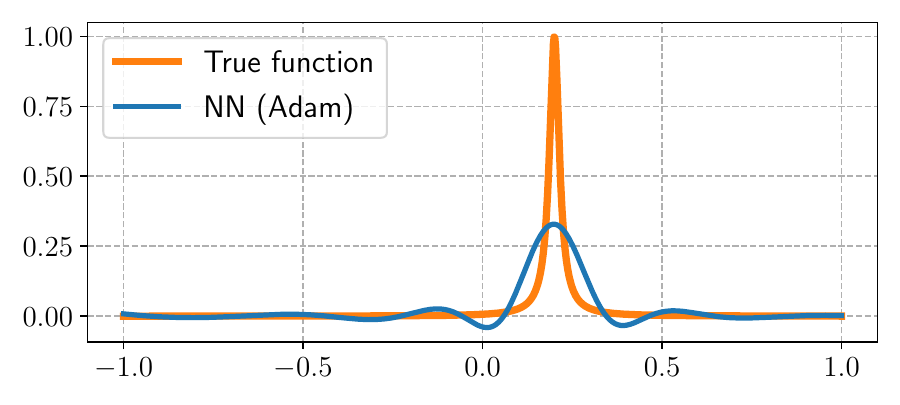}
    \subcaption{$\texttt{tanh}^{\prime}$.}
    \end{subfigure}
    \hfill        \begin{subfigure}[c]{0.192555\linewidth}
    \centering            \includegraphics[width=0.9985\textwidth]{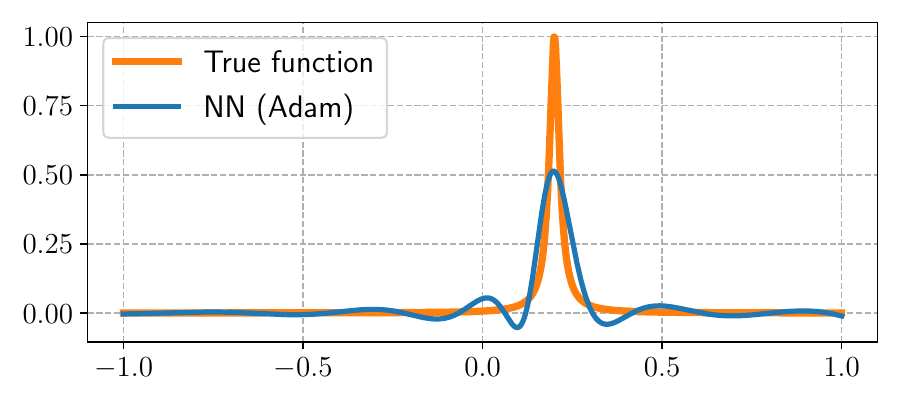}
     \subcaption{$\texttt{tanh}^{\prime\prime}$.}
    \end{subfigure}
    \caption{First row: using least square approximation with $n=512$, fixing $\bmw,\bmb\in \calU(-1,1)$.
    Second row: optimization using Adam with $n=512$, initialization $\bmw,\bmb\in \calU(-1,1)$. 
    All tests were conducted using double precision.
    }
    	\label{fig:act-comp}
\end{figure}

Actually, for two-layer neural networks using activation functions that are not homogeneous of degree one such as \texttt{ReLU}, instead of initializing $w_i, b_i$ uniformly distributed in $(-1,1)$, one can scale the range to be $(-s,s)$. By choosing a larger $s$, one can improve the representation capability of a two-layer neural network~\eqref{eq:standard1}. One way to see this is again through spectral analysis. The introduction of larger $w_i$ and hence basis functions with more rapid changes (larger derivatives) leads to a slower spectral decay of the corresponding Gram matrix, $\bmG_{i,j}=\int_{-1}^1\sigma\big(w_i(x+b_i)\big)\sigma\big(w_j(x+b_j)\big) dx$ with $\bmw \in \calU(-s,s)$ and $ \bmb\in \calU(-1,1)$, as shown in Figure~\ref{fig:specturm:scalew}.

\begin{figure}
    \centering	
        \begin{subfigure}[c]{0.3283\linewidth}
    \centering            \includegraphics[width=0.9985\textwidth]{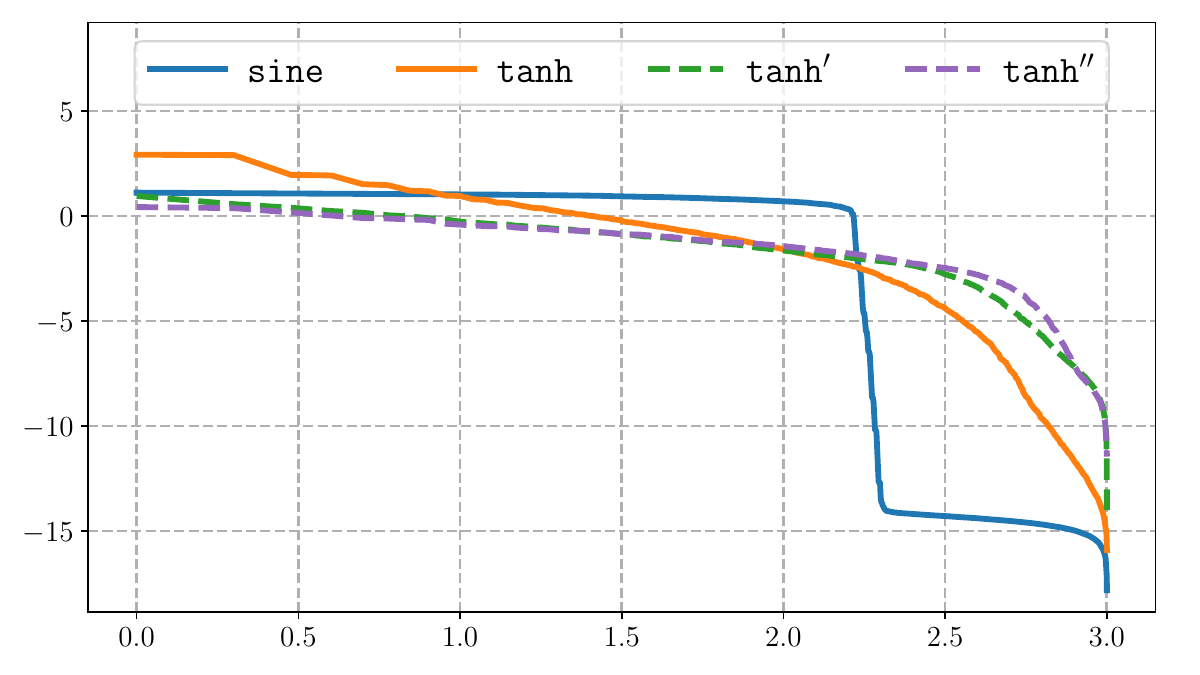}
    \subcaption{$s=256$.}
    \end{subfigure}
    \hfill
             \begin{subfigure}[c]{0.3283\linewidth}
    \centering            \includegraphics[width=0.9985\textwidth]{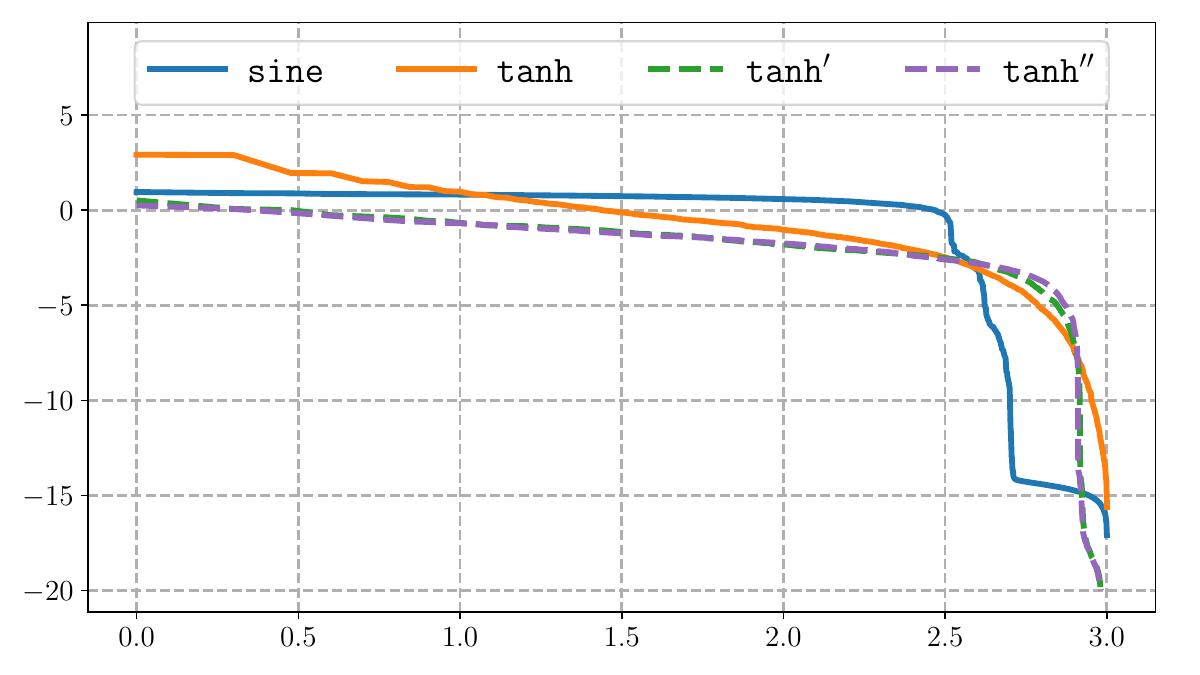}
    \subcaption{$s=512$.}
    \end{subfigure}
    \hfill
    \begin{subfigure}[c]{0.3283\linewidth}
    \centering            \includegraphics[width=0.9985\textwidth]{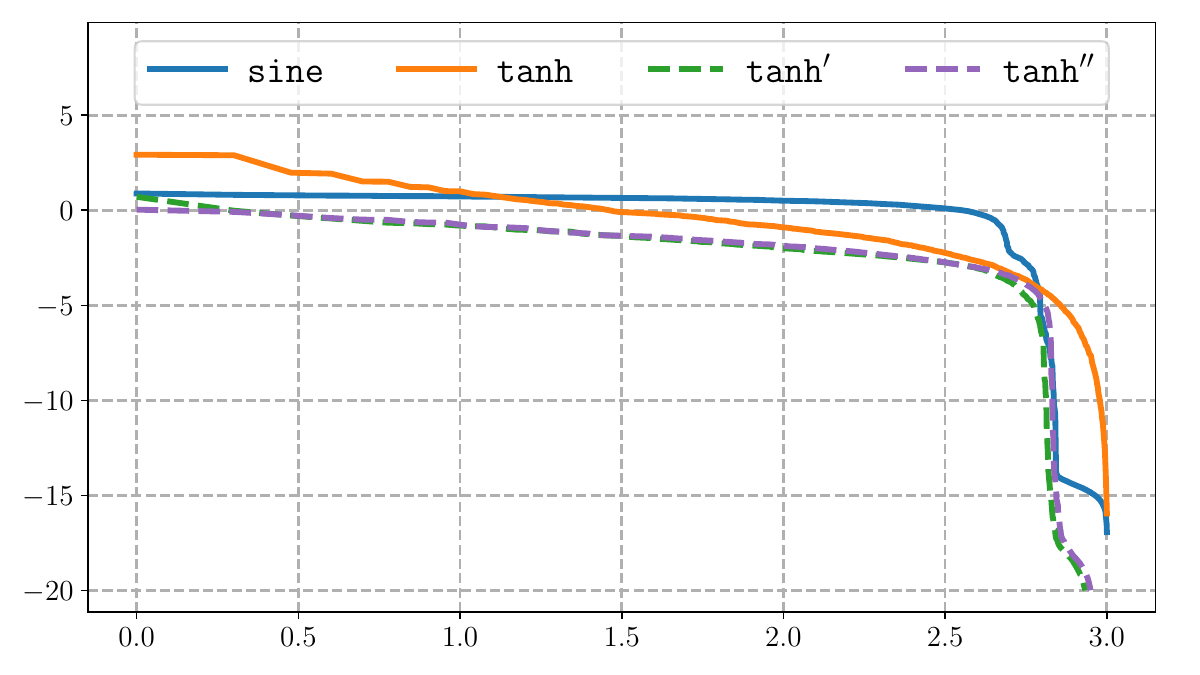}
    \subcaption{$s=768$.}
    \end{subfigure}\hfill  
    \caption{Spectrum of
    $\bmG$
    with $\bmw \in \calU(-s,s)$ and $ \bmb\in \calU(-1,1)$, where $\bmG_{i,j}$ is given by  $\bmG_{i,j}=\int_{-1}^1\sigma\big(w_i(x+b_i)\big)\sigma\big(w_j(x+b_j)\big) dx$.} 	\label{fig:specturm:scalew}
\end{figure}

Another way to look at this is that scaling $\bmw, \bmb$ by $s$ can be regarded as scaling up the target function by $s$, i.e., approximating $f(x)$ on $[-1,1]$ by a two-layer network $h(x)$
\begin{equation}
\label{eq:standard1}
   h(x) = \sum_{i=1}^n a_i \sigma\big(w_i (x + b_i)\big) \quad \text{on} \ [-1,1] \ \text{with} \  \bmw \in \mathcal{U}(-s,s),\  \bmb \in \mathcal{U}(-1,1)
\end{equation}
is equivalent to approximating $\tilde{f}(x)=f(\frac{x}{s})$ on $[-s,s]$ by the two-layer network $\tilde{h}(x)=h(\frac{x}{s})$
\begin{equation*}
   \tilde{h}(x)= h(\frac{x}{s}) = \sum_{i=1}^n a_i \sigma(w_i x + b_i) \quad \text{on} \ [-s,s] \ \text{with} \ \bmw \in \mathcal{U}(-1,1), \ \bmb \in \mathcal{U}(-s,s).
\end{equation*}

However, in order to be able to resolve the rapid change of activation functions with  derivatives proportional to $s$, or to have biases distributed dense enough in the interval $[-s, s]$, $s$ can be at most proportional to the network width $n$. 
For example, if one uses \texttt{sine} as the activation function, then \(s\) can be at most \(\mathcal{O}(n)\) (preferably \(n/2\)) so that the network can provide a set of maximally diverse random Fourier bases without missing intermediate frequencies.
This is because using linear combinations of $\sin(wx+b_i), i=1,2$, for two different $b_i$ can generate both $\sin(wx)$ and $\cos(wx)$. In the case of equally spaced $w$, using this set of parametrized activation functions is equivalent to the Fourier series with basis $\sin(mx), \cos(mx), m=-\frac{n}{2}+1, \ldots, \frac{n}{2}$. 

The experiment results shown in Figure~\ref{fig:LS:Adam:scaling:w} demonstrate how scaling up $\bmw$ in two-layer networks as in \eqref{eq:standard1} can enhance the representation capability and hence the approximation results. We note that different initializations may lead to different results, but the overall outcomes are largely similar. The subtle issue is how large $s$ should be. It is interesting to observe that for activation functions \texttt{sine}, \texttt{tanh}, and $\texttt{tanh}'$, the maximum magnitudes of their first derivatives are all bounded by 1, $s=\frac{n}{2}$ and $s=n$  work well. Once $s=\frac{3n}{2}$, the results degrade. For $\texttt{tanh}''$, the maximum magnitude of its first derivative is bounded by 2. It can be seen that $s=\frac{n}{2}$ works well but $s=n$ and $s=\frac{3n}{2}$ have degraded results. These tests suggest that the network size need to be able to resolve the rapid change of the activation function which is characterized by $s\cdot \sup_x|\sigma'(x)|$.


\begin{figure}
    \centering	
       \begin{subfigure}[c]{0.2432555\linewidth}
    \centering            \includegraphics[width=0.9985\textwidth]{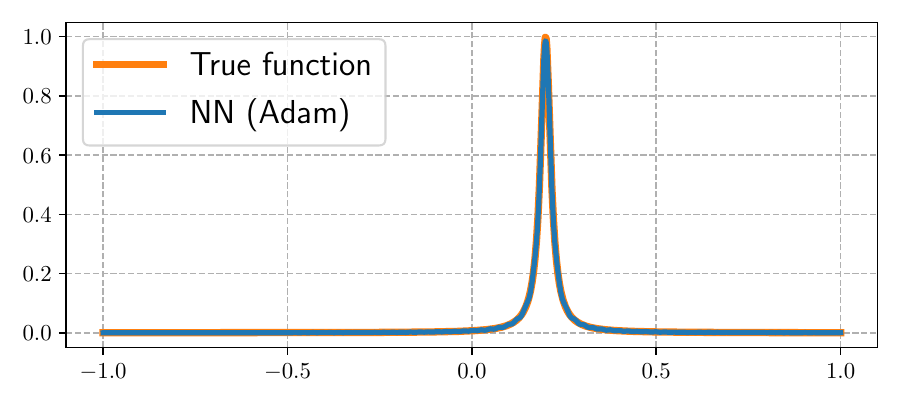}
    \subcaption{\texttt{sin} ($s=n/2$).}
    \end{subfigure}
    \hfill       
    \begin{subfigure}[c]{0.2432555\linewidth}
    \centering            \includegraphics[width=0.9985\textwidth]{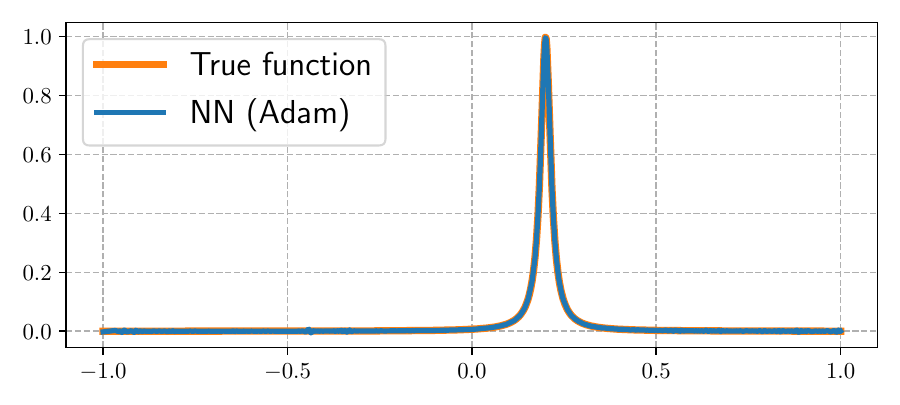}
    \subcaption{$\texttt{tanh} $ ($s=n/2$).}
    \end{subfigure}
           \begin{subfigure}[c]{0.2432555\linewidth}
    \centering            \includegraphics[width=0.9985\textwidth]{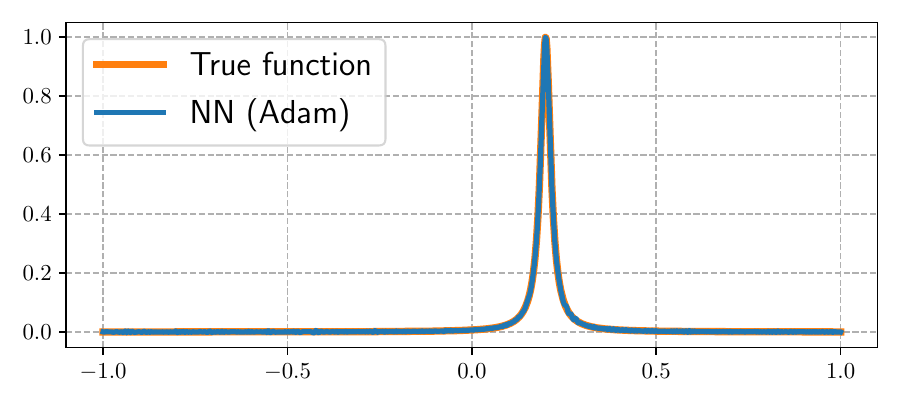}
    \subcaption{$\texttt{tanh}' $  ($s=n/2$).}
    \end{subfigure}
    \hfill        \begin{subfigure}[c]{0.2432555\linewidth}
    \centering            \includegraphics[width=0.9985\textwidth]{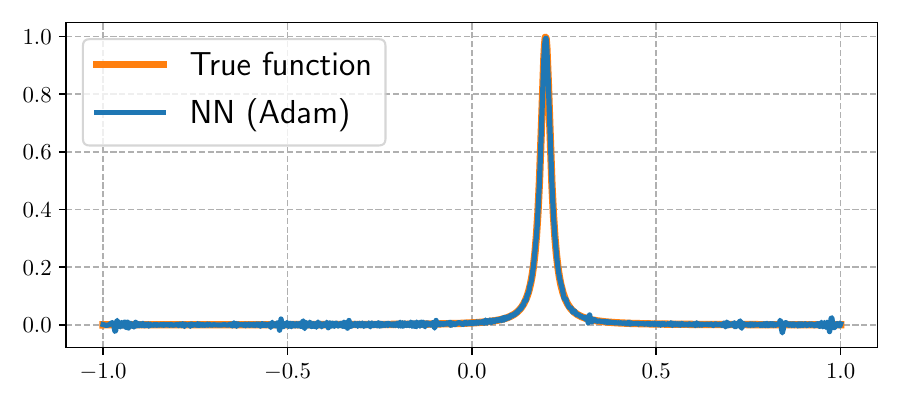}
    \subcaption{$\texttt{tanh}'' $ ($s=n/2$).}
    \end{subfigure}
    \\
           \begin{subfigure}[c]{0.2432555\linewidth}
    \centering            \includegraphics[width=0.9985\textwidth]{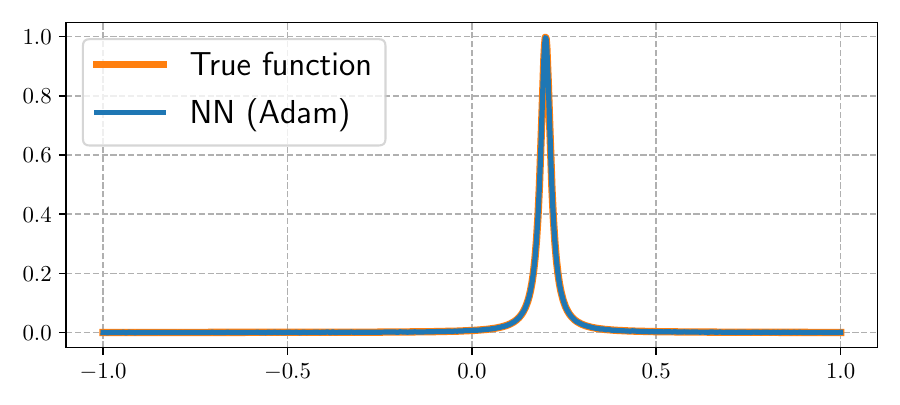}
    \subcaption{\texttt{sin} ($s=n$).}
    \end{subfigure}
    \hfill       
    \begin{subfigure}[c]{0.2432555\linewidth}
    \centering            \includegraphics[width=0.9985\textwidth]{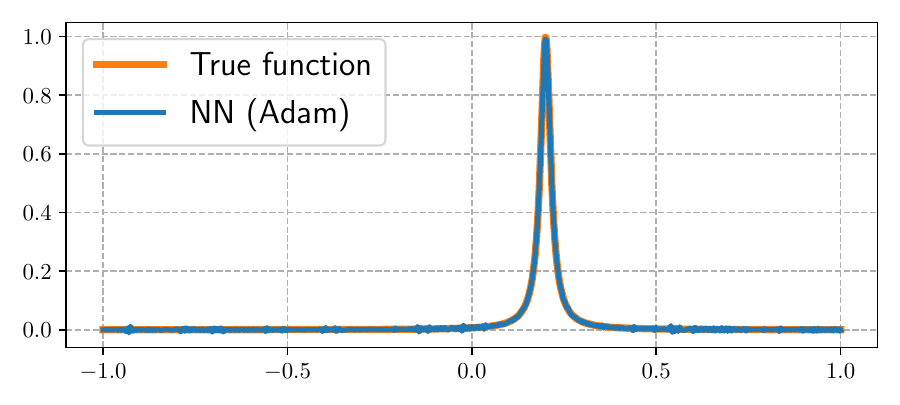}
    \subcaption{$\texttt{tanh} $ ($s=n$).}
    \end{subfigure}
           \begin{subfigure}[c]{0.2432555\linewidth}
    \centering            \includegraphics[width=0.9985\textwidth]{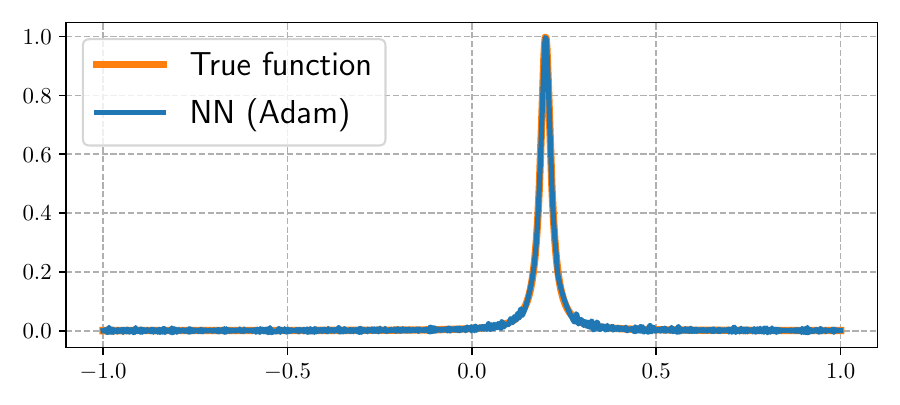}
    \subcaption{$\texttt{tanh}' $  ($s=n$).}
    \end{subfigure}
    \hfill        \begin{subfigure}[c]{0.2432555\linewidth}
    \centering            \includegraphics[width=0.9985\textwidth]{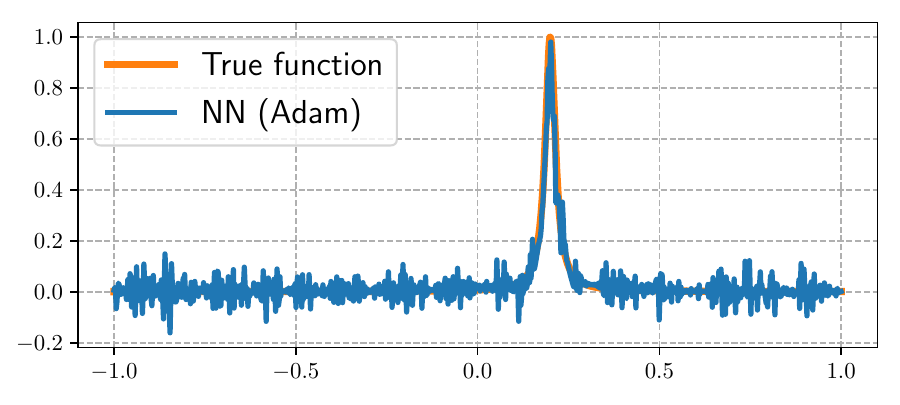}
    \subcaption{$\texttt{tanh}'' $ ($s=n$).}
    \end{subfigure}
    \\
        \hfill  
          \begin{subfigure}[c]{0.2432555\linewidth}
    \centering            \includegraphics[width=0.9985\textwidth]{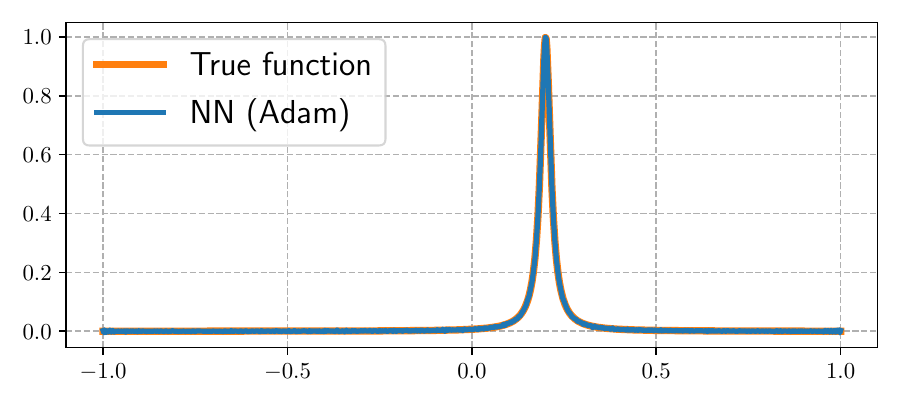}
    \subcaption{\texttt{sin} ($s=3n/2$).}
    \end{subfigure}
    \hfill        
    \begin{subfigure}[c]{0.2432555\linewidth}
    \centering            \includegraphics[width=0.9985\textwidth]{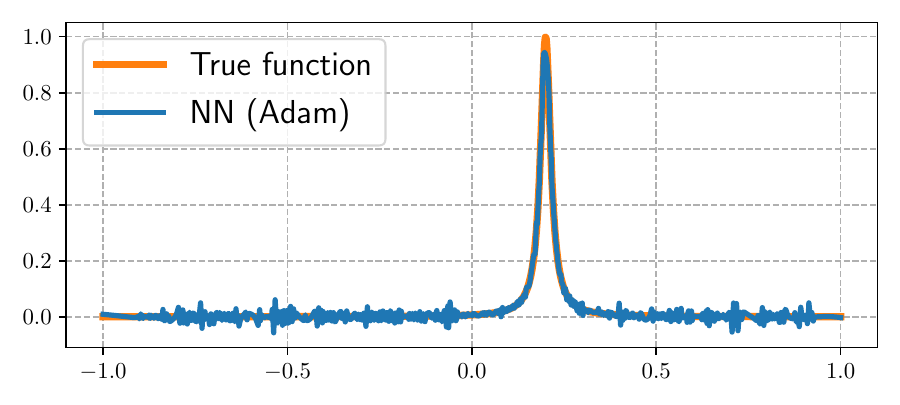}
    \subcaption{$\texttt{tanh} $ ($s=3n/2$).}
    \end{subfigure}
    \hfill  
          \begin{subfigure}[c]{0.2432555\linewidth}
    \centering            \includegraphics[width=0.9985\textwidth]{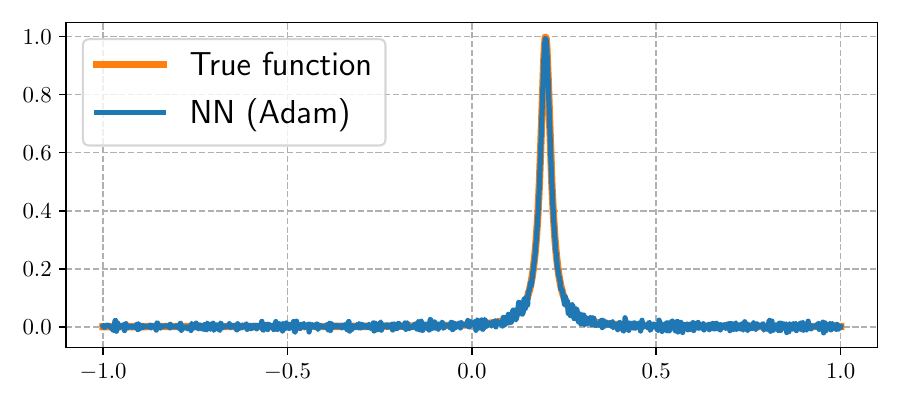}
    \subcaption{$\texttt{tanh}' $  ($s=3n/2$).}
    \end{subfigure}
    \hfill        \begin{subfigure}[c]{0.2432555\linewidth}
    \centering            \includegraphics[width=0.9985\textwidth]{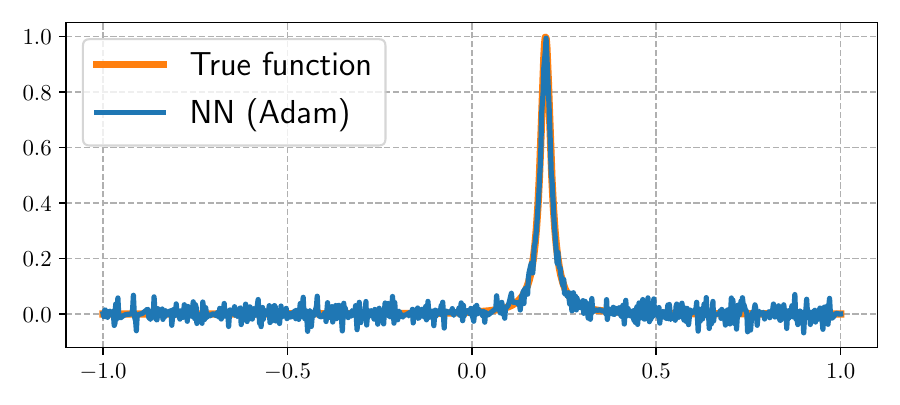}
    \subcaption{$\texttt{tanh}'' $ ($s=3n/2$).}
    \end{subfigure}
    \caption{
   Illustrations of learned networks
   $\sum_{i=1}^n a_i \sigma(w_i (x + b_i))$ optimized using Adam with $n=512$, initialization $\bma \in \calU(-1/\sqrt{n},1/\sqrt{n})$,
    $\bmw \in \calU(-s,s)$, and $ \bmb\in \calU(-1,1)$. 
All tests were conducted in double precision.
    }
    	\label{fig:LS:Adam:scaling:w}
\end{figure}

Although appropriate scaling up $\bmw, \bmb$ can improve approximation significantly over those corresponding results with normalized distribution in Figure~\ref{fig:act-comp}, however, our tests also suggest that the network size needs to be large enough, proportional to $s\cdot \sup_x|\sigma'(x)|$, to resolve the rapid change of the activation function. Moreover, it is difficult for a gradient descent based optimization to effectively learn adaptive $\bmw, \bmb$ to capture those fine features that are biased against in the initial representation. Hence, it implies that the representation capability of a two-layer network is no more than a linear representation which finds the optimal linear combination of a set of a priori given basis functions to approximate a target function, for example, the linear least square approximation. 
In this case, there are already well developed basis functions that provide well-conditioned representations with efficient numerical algorithms (based on solving linear systems), such as finite element bases, Fourier bases, splines, wavelets, which are usually better than using two-layer neural networks when the number of basis functions is equivalent to the network width. On the other hand, the key issue is that the basis functions are given a priori, independent of the target function to be approximated. Hence, to have enough representation capability, the set of basis has to be large and diverse enough and hence suffer from the curse of dimensionality.
We would like to point out that, the scaling of \texttt{ReLU} as the activation function, which is homogeneous degree of degree one, can be absorbed in $\bma$. The distribution of $\bmb$ is equivalent to a grid for piecewise linear approximation in one dimension. Numerically, due to its ill-conditioning, the approximation is worse than using linear representation by finite element basis as discussed in earlier sections.
\begin{remark}
For deep neural networks, scaling up $\bmw, \bmb$ across all layers can cause instability for the training process.
\end{remark}

\subsection{Observations and comments}\label{sec:comments}
Based on the above analysis and numerical experiments, we give a few comments on using shallow neural networks, many of which have been well observed in practice. 
\subsubsection{Low-pass filter nature}
In practice (e.g., MATLAB), solving a linear system is typically approximated/regularized by its {Moore-Penrose} pseudo-inverse by cutting off the small eigenvalues at a threshold of $n\mathcal{\eps}\lambda_1$ to control the computer roundoff error, where $n$ is the matrix size and $\eps$ is the machine precision or noise level in the data. For a two-layer \texttt{ReLU} network with evenly or uniformly distributed biases and width $n\ge k$, given the spectral decay of the Gram matrix, $\lambda_k=\Theta(k^{-\frac{d+3}{d}})$, only leading modes for $k\le m =\cO( \epsilon^{-\frac{d}{2d+3}})$  can be stably recovered in the least-square approximation. If the network is wide enough such that all leading modes up to $m$ can be approximated well, the dominant numerical error is caused by the truncation of those modes higher than $m$. 
From~\eqref{eq:expansion}, we see that truncation of the higher modes in the parameter space leads to a low pass filter in the approximation of the original function since $u_k$ is equivalent to the eigenfunction of the Laplace operator asymptotically.  In other words,  at most all eigenmodes of the Laplace operator up to frequency $\mathcal{O}(\mathcal{\epsilon}^{-\frac{1}{2d+3}})$ can be captured in $d$ dimensions no matter how wide the network is. This is because, 1) before the network width reaches a threshold $n_{\epsilon,d}=\mathcal{O}(\mathcal{\epsilon}^{-\frac{d}{2d+3}})$, the network does not have a grid resolution to resolve the frequency modes of order $\mathcal{O}(\mathcal{\epsilon}^{-\frac{1}{2d+3}})$, 2) even as the network width passes $n_{\epsilon,d}$, the network can only approximate leading modes $k$ such that $\lambda_k\ge n\epsilon \lambda_1 \Rightarrow k^{\frac{1}{d}} \le \mathcal{O}(\mathcal{\epsilon}^{-\frac{1}{2d+3}})$ due to the ill-conditioning of the representation.

The machine precision for single and double precision are: $\mathcal{\epsilon}_1 = 2^{-23}$ and $\mathcal{\epsilon}_2= 2^{-52}$ respectively. Hence, a two-layer neural network can resolve about $2^{23d/(2d+3)}$ and $2^{52d/(2d+3)}$ eigenmodes respectively in $d$-dimensions. The number of modes in each direction that can be resolved is $2^{23/(2d+3)}$ and $2^{52/(2d+3)}$ respectively in $d$-dimensions, which is roughly $24 ~(d=1), 10 ~(d=2), 6 ~(d=3)$, and $2 ~(d=10)$ for single precision, and $1351~(d=1), 172 ~(d=2), 55 ~(d=3)$, and $5 ~(d=10)$ for double precision. 

\begin{remark}
Our analysis applies to the neural tangent kernel (NTK) regime where the network width goes to infinity while the biases are pretty much fixed. 
\end{remark}


\subsubsection{Approximation error}
Given a machine precision $\epsilon$, the number of modes that can be captured by a two-layer $\texttt{ReLU}$ network is at most $\cO(\epsilon^{-\frac{1}{2d+3}})$ (no matter how wide the network is). One can characterize the numerical error for a two-layer $\texttt{ReLU}$ network in two regimes:
\begin{itemize}
    \item \textbf{Small network}\\
    If the network width $n$ is less than $\cO(\epsilon^{-\frac{d}{2d+3}})$, the corresponding grid resolution is less than $h=\cO(n^{-\frac{1}{d}})\le \cO(\epsilon^{\frac{1}{2d+3}})$, which can not resolve the highest mode limited by $\epsilon^{-\frac{1}{2d+3}}$. Hence the numerical error is dominated by the discretization error. Since the resulting approximation is continuous peicewise linear, the $L^2$ approximation error of a function $f$ with bounded Sobolev norm $\|f\|_{H^2}$ is of order $h^2\|f\|_{H^2}\sim n^{-\frac{2}{d}}\|f\|_{H^2}$. In the small network regime, using a smooth activation function may be beneficial when approximating a smooth function due to the reduction of discretization error. 
    \item \textbf{Large network}\\
    When the network is wide enough to resolve the highest mode of order $\cO(\epsilon^{-\frac{1}{2d+3}})$ accurately, then the numerical error is dominated by the truncation of higher modes. For a function $f$ in Sobolev space $H^{p}$, the $L^2$ error due to truncation is $\cO(\epsilon^{\frac{p}{2d+3}}\|f\|_{H^p})$.
\end{itemize}  


\subsubsection{Implications}
The ill-conditioning of two-layer neural network representation and its bias against high frequencies explain why it is widely observed that shallow neural networks can approximate smooth functions well, while for functions with fast transitions or rapid oscillations, one may achieve a numerical accuracy that is far from machine precision even using \emph{wide} shallow networks for which universal approximation is proved in theory. 
On the other hand, the low-pass filter nature of the shallow neural networks also alleviates instability with respect to noise or overfitting/over-parametrization. 

The spectral decay of the Gram matrix for a set of basis depends on the smoothness of the basis. The smoother the activation function is, the faster the spectrum of the corresponding Gram matrix decays (see Remark~\ref{re:relu-k} and Section~\ref{sec:act-comp}), and hence the fewer eigenmodes can be used for approximation in practice and the more bias against high frequencies. 

However, with a fixed network width (or a fixed grid resolution), approximation order induced by the activation function also plays a role. For example, using \texttt{ReLU} results in a piecewise linear approximation, the error of which is proportional to the grid resolution squared (2nd order) if the target function is twice differentiable. If the Heaviside function (the derivative of \texttt{ReLU}) is used as the activation function, although the Gram matrix is better conditioned than using \texttt{ReLU}, the resulting piecewise constant approximation is only 1st order if the target function is differentiable, and hence may limit the numerical accuracy in practice.


In applications, if one can make the target function or map smooth under certain transformation, for instance, a linear transformation with a set of new bases, e.g., Fourier basis with high frequencies, using a neural network in the transformed domain can achieve high accuracy.

\subsubsection{Gradient decent for least squares}
Before we investigate the full nonlinear learning dynamics for two-layer neural networks in the next section, we demonstrate how ill-conditioning in the representation will affect the convergence of a gradient descent based optimization for a simple quadratic convex function corresponding to a linear least square approximation.  Instead of finding the least square solution directly by solving the linear system (normal equation) with the Gram matrix $\bmG$, if one chooses to minimize the least square by using gradient descent, then the dynamics of the coefficients $\bm{a}(t)=[a_1(t), a_2(t), \ldots, a_n(t)]^T$ follow the system of ODEs
\[
\frac{d\bm{a}(t)}{dt}\!=\!-\bm{G}\bm{a}(t)+\bm{f}, \quad \bm{G}_{i,j}\!\!=\!\!\int_D\sigma(\bmw_i\cdot \bmx-b_i)\sigma(\bmw_j\cdot \bmx-b_j)d\bmx, ~\bm{f}_i\!\!=\!\!\int_D f(\bmx)\sigma(\bmw_i\cdot \bmx-b_i)d\bmx.
\]
The system is stiff due to the fast spectral decay of $\bmG$. Let $(\lambda_k, \bm{g}_k)$ be the eigen pairs of $\bmG$ and define $\hat{a}_k(t)=\bm{a}^T(t)\bm{g}_k, \hat{f}_k=\bm{f}^T\bm{g}_k$. We have
\[
\frac{d\hat{a}_k(t)}{dt}=-\lambda_k \hat{a}_k(t)+\hat{f}_k 
\quad \Longrightarrow \quad 
\hat{a}_k(t)=(\hat{a}_k(0)-\frac{\hat{f}_k}{\lambda_k})e^{-\lambda_k t}+\frac{\hat{f}_k}{\lambda_k}.
\]
It takes at least $t>\cO(\lambda_k^{-1})$ for the intial error in $k$-th mode to reduce significantly. Since $\lambda_k\rightarrow 0$ as $k\rightarrow \infty$ and the corresponding eigenmode becomes more and more oscillatory, two-layer neural networks bias against high frequencies in both representation and training. An appropriate stopping time can be used for the sake of computation cost or regularization for noise or machine roundoff error. 




\section{Learning dynamics}
\label{sec:dynamics}
The key feature in machine learning using neural network representation is the training process, which ideally can find the optimal parameters, i.e., an adaptive representation driven by the data. 
The relevant approximation theory has been studied extensively. The best-known approximation error estimates using \texttt{ReLU} as the activation function for two-layer neural network~\cite{barron2018approximation, ongie2019function} are based on the fact that $\sigma''(t) = \delta(t)$, which implies $\Delta f$ can be expressed as a standard Radon transform when viewing the parameters $\{a_i\}_{i=1}^n$ as an empirical measure of $(\bmw, b)\in \bbS^{d-1}\times \bbR$. In practice, gradient-based methods (and the variants) are adopted to seek the optimal parameters. 

Our previous analysis shows that a two-layer neural network with fixed (or randomly sampled) biases will have a low-pass filter nature, which makes it challenging to capture high-frequency components, e.g. rapid changes or fast oscillations. Here we show that the ill-conditioning of the Gram matrix causes difficulties in the learning process as well. Intuitively, without high frequency information, there is no correct guidance in effectively and efficiently optimizing the distribution of the parameters, $(\bmw,b)$, from initial uniform distribution to be non-uniform adaptive to the target function. Often in pracitce, undesirable clustering of parameters, which reduces the effecitve network width, may occur during the training.  



Here we study the following natural question: assuming that a gradient-based optimization can find the optimal solution, which itself is a challenging question in general, what is the training dynamics and its computation cost to attain such a solution from a random initial guess?
In particular, we demonstrate that 
 initial high-frequency component error can take a long time to correct. All these lead to numerical difficulties in achieving the optimal solution even if one assumes the learning process can find the optimal solution in theory. It implies that even if the full learning process is applied, the numerical error can still be far from the machine precision for functions with high-frequency components in practice.

Training a neural network with the gradient flow can be regarded as the gradient descent method with a very small step size for finite time dynamics. We illustrate this using a one-dimensional example. Let $D = [-1,1]$, the gradient flow of $\{a_i\}_{i=1}^n$ and $\{b_i\}_{i=1}^n$ follow 
 \begin{equation*}
 \begin{aligned}
    \frac{d a_i}{d t} &= -\int_{D}  (h(x, t) - f(x)) \sigma(x - b_i) dx\quad\tn{and}\quad
    \frac{d b_i}{d t} &= a_i \int_{D} (h(x, t) - f(x))\sigma'(x - b_i)  dx. 
 \end{aligned}
\end{equation*}
One popular way to analyze the dynamics of the neuron network is the mean-field representation in which the network is written as
\begin{equation*}
    \overline{h}(x, t) = \int_{\bbR^2} a\sigma(x - b) \mu_n(a, b, t) da db 
\end{equation*}
with empirical measure $\mu_n(a, b, t) = \frac{1}{n}\sum_{i=1}^n \delta(a - a_i(t), b- b_i(t))$. 
The analysis of the limiting behavior of mean-field neural networks can be found in~\cite{rotskoff2018trainability,mei2018mean,sirignano2020mean} and the references therein. However, most of the mean-field studies assume the measure $\mu_n(\cdot,\cdot, t)$ converges as $t\to\infty$ and $n\to\infty$. 


Our study works for fully discrete two-layer neural networks with no convergence assumption for the measure $\mu_n$ or requiring $n\to\infty$. The main difficulty for our analysis is the possibility of biases $b_i$ moving out of the bounded domain of interest. Define generalized Fourier modes, $\{\theta_m\}_{m\ge 1}$, which are the eigenfunctions of the Gram kernel $\cG$ in~\eqref{EQ: GRAM KERNEL}, and $\widehat{g}$ as the generalized discrete Fourier transform of $g$, 
\begin{equation*}
    \widehat{g}(m) = \int_D g(x) \theta_m(x) dx.
\end{equation*}
In one dimension, the generalized Fourier transform is asymptotically close to the standard Fourier transform as the mode increases.
 The key to our study is the construction of an auxiliary function $w(x, t)\in C^2(D)$ (defined by \eqref{eq:aux}) that satisfies $\partial^2_x w(x,t) = h(x, t) - f(x)=e(x,t)$, which is the approximating error at time $t$, with boundary conditions $w(1, t) = \partial_x w(1, t) = 0$. By studying the evolution of $\widehat{w}(m,t)$ we show at least how slow the learning dynamics can be in terms of the lower bound for time needed to reduce the initial error in mode $m$ in half. 
We prove the following statement for learning dynamics in one dimension under the following mild assumptions: 1) there exists a constant $M > 0$ that $\sup_{i=1}^n|a_i(t)|^2\le M$; 2) the initialization of biases $\{b_i(0)\}_{i=1}^n$ are equispaced on $D$.

\begin{theorem}\label{THM: PAPER SLOW DECAY}
If $|\widehat{w}(m, 0)| \neq 0$, then it will take at least $\cO(\frac{m^3 |\widehat{w}(m, 0)|}{n})$ time to reduce the {generalized} Fourier coefficient in half, i.e., $|\widehat{w}(m, t)| \le \frac{1}{2} |\widehat{w}(m, 0)|$.
\end{theorem}

In numerical computation, to follow the gradient flow closely, the discrete time-step is $\cO(\frac{1}{n})$. From the relation $|\widehat{w}(m,t)|\simeq m^{-2}|\widehat{e}(m,t)|$, Theorem~\ref{THM: PAPER SLOW DECAY} says that if the initial error in mode $m$, $|\widehat{e}(m,0)|\ne 0$, (which means $|\widehat{w}(m,0)|> cm^{-2}$ for some $c>0$,) it takes at least $\cO(m)$ steps to reduce the error in mode $m$ by half. Note that the result does not depend on the convergence of the optimization algorithm and the above estimate is a lower bound. In practice, the gradient-based training process could have an even slower learning rate. 
In the next theorem, the lower bound is improved in a more specific scaling regime (in terms of network width vs. frequency mode) and using better estimates (see Appendix~\ref{sec:improved}). It implies that it takes at least $\cO(m^2)$ steps to reduce the error in mode $m$ by half in numerical computation. In the special case of fixed biases (i.e., least square problem), the number of steps needed is at least $\cO(m^4)$ if $n = \Omega( m^3 )$ (see Remark~\ref{rem: fix bias}).

\begin{theorem}\label{THM: PAPER SLOW DECAY 2}
If the total variation of the sequence $\{a_i^2(t)\}_{i=1}^n$ is bounded by $M'$. Let $n \ge m^4$ be sufficiently large, then it will take at least $\cO(\frac{m^4 |\widehat{w}(m, 0)|}{n})$ time to reduce the {generalized} Fourier coefficient by half, i.e., $|\widehat{w}(m, t)| \le \frac{1}{2} |\widehat{w}(m, 0)|$. 
\end{theorem}

The above results show that reducing the initial error in high-frequency modes by gradient-based learning dynamics is slow. Moreover, due to the low-pass filter nature shown earlier, it is difficult to capture high-frequency modes for two-layer neural networks with evenly or uniformly distributed initial biases. These two effects show why a two-layer neural network struggles with high-frequency modes in approximation even if a full training process is employed. 

Before we prove the above two theorems, we need to introduce the mathematical setup and prove a few lemmas and intermediate results. 

\subsection{Mathematical setup of learning dynamics}
Let $D = (-1,1)$, we consider approximating the objective function $f(x)\in C(D)$, with the shallow neural network 
\begin{equation*}
    h(x) =  \sum_{i=1}^n a_i \sigma(x - b_i).
\end{equation*}
Generally speaking, the biases $\{b_i\}_{i=1}^n$ are supported on $\bbR$ during the training process. For analysis purposes, we restrict $\{b_i\}_{i=1}^n\subset \overline{D^{\eps}}$, where $D^{\eps} =\supp\chi = (-1-\eps, 1)$, and modify the activation function in the two-layer neural network representation as follows
\begin{equation*}
    h(x) = \sum_{i=1}^n a_i \psi(x, b_i),\quad \psi(x, b_i) = \chi(b_i) \sigma(x - b_i), \quad \partial_x^2 \psi(x, b) =\partial_b^2 \psi(x, b)=\delta(x-b), ~x, b\in D.
\end{equation*}
Here $\chi$ is an approximation to the characteristic function of $D$ by adding a quadratic transition region:
\begin{equation*}
    \chi(x) =  \begin{cases}
        1,& x\in D, \\
        1-\left(\frac{1}{\eps} (x + 1)\right)^2, & x\in (-1-\eps, -1], \\
        0, & \text{otherwise},
    \end{cases}
\end{equation*}
and $\eps$ is a positive parameter. The loss function is 
\begin{equation*}
    \cL(h, f)\coloneqq \frac{1}{2}\int_D |h(x) - f(x)|^2 dx.
\end{equation*}
Minimizing $\cL(h, f)$ over the parameters $\{a_i\}_{i\in[n]}$ and $\{b_i\}_{i\in [n]}$ by gradient descent follows the gradient flow
\begin{equation}\label{EQ: FLOW A}
    \frac{d}{dt} a_i(t) = -\int_{D} ( h(x, t) - f(x) ) \psi(x, b_i(t)) d x 
\end{equation}
and 
\begin{equation}\label{EQ: FLOW B}
    \frac{d }{d t} b_i(t) = -a_i (t) \int_D (h(x, t) - f(x)) \partial_b\psi(x , b_i(t)) dx 
\end{equation}
with certain initial conditions $\{a_i(0)\}_{i\in[n]}$ and $\{b_i(0)\}_{i\in[n]}$. Due to the modification of the network representation, one can show that no bias can move outside $\overline{D^{\eps}}$ at a later time.
\begin{theorem}
    If the initial biases and weights satisfy $\{b_i(0)\}_{i\in[n]} \subset D^{\eps}$ and $\sup_{i\in[n]} |a_i(0)|<\infty$, we have $\{b_i(t)\}_{i\in[n]}\subset \overline{D^{\eps}}$, $\forall t\ge 0$.
\end{theorem}
\begin{proof}
    At any time $t$, using Cauchy-Schwartz inequality, 
    \begin{equation*}
    \begin{aligned}
        \left|\frac{d}{dt} a_i(t)\right| &\le \sqrt{\int_D |h(x, t) - f(x)|^2 dx}\sqrt{ \int_D \psi(x, b_i(t))^2 dx } \\&\le   \sqrt{\int_D |h(x, 0) - f(x)|^2 dx}\sqrt{ \int_D \psi(x, b_i(t))^2 dx } 
        & <\infty.
    \end{aligned}
    \end{equation*}
    Therefore if initial $a_i(0)$ is finite, $a_i(t)$ is always finite, which in turn implies every derivative $\frac{d}{dt} b_i(t)$ is also finite.  
    Suppose at time $t > 0$ that $b_k(t) \in \overline{D^\eps}^{ \complement}$, then there exists an open path that $\forall t'\in (t_0, t_1)\subset (0, t)$ that $\frac{d}{dt} b_k(t') \neq 0$ and $b_k(t') \in \overline{D^\eps}^{\complement}$.
    However, $\partial_b \psi(x, b)  = 0$ for $b\in \overline{D^{\eps}}^{\complement}$, which is a contradiction. Therefore every $\{b_i(t)\}_{i\in[n]}\subset \overline{D^{\eps}}$.
\end{proof}

Define the Gram kernel $\cG(b, b')$ on $D^{\eps}\times D^{\eps}$ by
\begin{equation}
\label{eq:modified}
    \cG(b, b') = \int_D \psi(x, b) \psi(x, b') dx,
\end{equation}
which is a compact operator in $L^2(D^{\eps})$. Let $\lambda_k\ge 0,\  k=1,2, \ldots,$ be the eigenvalues in descending order and $\phi_k$ be the corresponding eigenfunctions, which form an orthonormal basis in $L^2(D^{\eps})$.  We have $  \cG(b, b') = \sum_{k \ge 1} \lambda_k \phi_k(b) \phi_k(b') $. Some properties of $\phi_k$ are studied in Appendix~\ref{sec:eigen}. Expand the coefficient $a_i(t)$ along the eigenfunction $\phi_k$ (an orthonormal basis in $L^2(D^{\epsilon}$) in discrete version (at the nodal points $b_i(t)$), its dynamics follows
\begin{equation*}
\begin{aligned}
\frac{d}{dt} \sum_{i=1}^n a_i(t)\phi_k(b_i(t)) &=\sum_{i=1}^n \phi_k(b_i(t)) \frac{d a_i}{d t} + \sum_{i=1}^n a_i(t) \phi_k'(b_i(t))\frac{d b_i}{dt} \\
   &= -\sum_{i=1}^n \phi_k(b_i(t))\int_{D} \left(h(x,t) - f(x)\right)\psi(x , b_i(t)) dx \\
   &\quad - \sum_{i=1}^n \phi_k'(b_i(t)) |a_i(t)|^2 \int_D (h(x, t) - f(x) ) \partial_b \psi(x, b_i(t)) dx \\
   &= -\sum_{i=1}^n \phi_k(b_i(t)) \left(\sum_{j=1}^n \bmG_{i,j}(t) a_j(t) - P(b_i(t)) \right) \\& \quad - \sum_{i=1}^n |a_i(t)|^2\phi_k'(b_i(t)) \left(\sum_{j=1}^n K_{i,j}(t) a_j(t) - Q(b_i(t)) \right).
\end{aligned}
\end{equation*}
where $\bmG = (\bmG_{i,j})_{i,j\in[n]}$ and $\bmK = (\bmK_{i,j})_{i,j\in[n]}$ are 
\begin{equation*}
    \bmG_{i,j}(t) = \cG(b_i(t), b_j(t)) ,\quad \bmK_{i,j}(t) = \partial_b \cG(b_i(t), b_j(t)).
\end{equation*}
The functions $P$ and $Q$ are 
\begin{equation*}
    P(b) = \int_D f(x) \psi(x, b) dx ,\quad Q(b) = \int_D f(x) \partial_b \psi(x, b) dx. 
\end{equation*}
We can represent $P(b)$ and $Q(b)$ as sum of eigenfunctions on $D^{\eps}$ as well:
\begin{equation*}
    P(b) = \sum_{k\ge 1} p_k \phi_k(b), \quad Q(b) = \sum_{k\ge 1} p_k \phi_k'(b).
\end{equation*}
Therefore we can derive that 
\begin{equation*}
\begin{aligned}
   &\phantom{\mspace{2mu}=\;} -\sum_{i=1}^n \phi_k(b_i(t)) \left(\sum_{j=1}^n \bmG_{i,j}(t) a_j(t) - P(b_i(t)) \right) \\&= - \sum_{l=1}^{\infty}  \sum_{i=1}^n \phi_k(b_i(t)) \phi_l(b_i(t))  \left[\lambda_l \sum_{j=1}^n \phi_l(b_j(t)) a_j(t) - p_l \right] 
\end{aligned}
\end{equation*}
and 
\begin{equation*}
\begin{aligned}
     &\phantom{\mspace{2mu}=\;} -\sum_{i=1}^n |a_i(t)|^2\phi_k'(b_i(t)) \left(\sum_{j=1}^n K_{i,j}(t) a_j(t) - Q(b_i(t)) \right)\\
     &=-\sum_{l=1}^{\infty} \sum_{i=1}^n |a_i(t)|^2  \phi_k'(b_i(t)) \phi_l'(b_i(t)) \left[\lambda_l \sum_{j=1}^n \phi_l(b_j(t)) a_j(t) - p_l\right].
\end{aligned}
\end{equation*}
Denote $\Theta_k(t) = \sum_{j=1}^n \phi_k(b_j(t)) a_j(t) - \frac{p_k}{\lambda_k}$, we find that 
\begin{equation}\label{EQ: THETA K}
    \frac{d \Theta_k(t)}{dt} = - \sum_{l=1}^{\infty} \lambda_l \left[  M_{l,k} (t) + S_{l,k}(t) \right]\Theta_l(t), 
\end{equation}
where the infinite matrices 
\begin{equation*}
    M_{l,k}(t) =\sum_{i=1}^n \phi_l(b_i(t)) \phi_k(b_i(t)),\quad S_{l,k}(t) = \sum_{i=1}^n |a_i(t)|^2  \phi_k'(b_i(t)) \phi_l'(b_i(t)) 
\end{equation*}
are both semi-positive.

Consider the energy $E(t) = \frac{1}{2} \sum_{k\ge 1}\lambda_k |\Theta_k^2(t)|$, then 
\begin{equation*}
    \frac{d E(t)}{d t} = - \sum_{k, l = 1}^{\infty} \lambda_k\lambda_l \Theta_k(t)  \left[M_{l,k}(t) + S_{l,k}(t)\right]\Theta_l(t).
\end{equation*}
Since the matrices are both non-negative, the energy is non-increasing, thus there is a limit for $E(t)$ as $t\to\infty$. Define the auxiliary function 
\begin{equation}\label{eq:aux}
    w(b, t)\coloneqq \sum_{k=1}^{\infty} \lambda_k \Theta_k(t) \phi_k(b),
\end{equation}
which is directly related to the error from the following relation for $b\in D$,
\begin{equation}\label{eq:order2}
\begin{aligned}
 \partial_b^2 w(b, t) & 
    = \partial_{b}^2\left[\sum_{k=1}^{\infty}\left(\lambda_k \sum_{j=1}^n \phi_k(b_j(t)) a_j(t) - p_k\right) \phi_k(b)\right]\\
    &= \partial_b^2 \left[\sum_{j=1}^n \cG(b, b_{j}(t)) a_j(t)- P(b)\right]\\
    &= \partial_b^2 \left[\sum_{j=1}^n  a_j(t)\int_D\psi(x,b)\psi(x,b_j(t)) dx- \int_D \!f(x)\psi(b)  dx \right]\\
    &=\sum_{j=1}^n  a_j(t)\int_D\delta(x- b)\psi(x,b_j(t)) dx - \!\int_D f(x)\delta(x - b)  dx
    =h(b, t) -  f(b).
\end{aligned}
\end{equation}
We first provide its regularity property which can be related to the spectral decay of the Gram matrix determined by the regularity of the activation function.

\begin{theorem}
    The auxiliary function $w\in H^1(D^{\eps})$ and $w\in H^2(D)$ uniformly, respectively.
\end{theorem}
\begin{proof}
      From~\eqref{eq:order2}, we have $\|\partial_b^2 w\|_{L^2(D)} \le \|h(\cdot, 0) - f(\cdot)\|_{L^2(D)}$ uniformly. We also notice that each eigenfunction $\phi_k$, $k\in\bbN$ satisfies that 
$\phi_k(b, t)|_{b = 1} = \partial_b \phi_k(b, t) |_{b=1} = 0$, hence $w(b, t)|_{b=1} = \partial_b w(1, t) = 0$. Using Cauchy-Schwartz inequality, we derive the following Poincar\'e inequalities, 
\begin{equation*}
    \left\| w(b, t) \right\|_{L^2(D)}^2 = \left\|\int_{1}^b \partial_b w(b', t) db'\right\|^2_{L^2(D)} \le 2 \left\| \partial_b w(b, t) \right\|_{L^2(D)}^2  
\end{equation*}
and 
\begin{equation*}
    \left\| \partial_b w(b, t) \right\|_{L^2(D)}^2 = \left\|\int_{1}^b \partial_b^2 w(b', t) db'\right\|^2_{L^2(D)} \le 2 \left\| \partial_b^2 w(b, t) \right\|_{L^2(D)}^2. 
\end{equation*}
Therefore $w\in H^2(D)$. Now we prove the other part of the theorem: $w\in H^1(D^{\eps})$ uniformly. Here we use the fact that $E(t) = \frac{1}{2}\sum_{k\in\bbN}\lambda_k |\Theta_k(t)|^2$ is uniformly bounded by $E(0)$, then 
\begin{equation*}
    \|w(\cdot, t)\|^2_{L^2(D)} \le \sum_{k\in\bbN} \lambda_k^2 |\Theta_k(t)|^2 < \lambda_1 \sum_{k\in\bbN} \lambda_k |\Theta_k(t)|^2 \le \lambda_1 E(0).
\end{equation*}
Follow the same derivation of~\eqref{eq:order2}, 
\begin{equation*}
\begin{aligned}
 \partial_b w(b, t) & 
    = \partial_{b} \left[\sum_{k=1}^{\infty}\left(\lambda_k \sum_{j=1}^n \phi_k(b_j(t)) a_j(t) - p_k\right) \phi_k(b)\right]\\
    & = \partial_b \left(\sum_{j=1}^n \cG(b, b_{j}(t)) a_j(t) - P(b)\right)\\
    &=  \partial_b \left(\sum_{j=1}^n  a_j(t)\int_D\psi(x,b)\psi(x,b_j(t)) dx - \int_D f(x)\psi(x,  b)  dx\right)\\
    &=\sum_{j=1}^n  a_j(t)\int_D \partial_b \psi(x, b) \psi(x,b_j(t)) dx- \int_D f(x) \partial_b \psi(x, b)  dx \\
    &=\int_D \left( h(x, t) - f(x)\right) \partial_b \psi(x, b) d x .
\end{aligned}
\end{equation*}
Thus using $\|h(\cdot, t) - f\|_{L^2(D)} \le \|h(\cdot, 0) - f\|_{L^2(D)}$ and $\partial_b \psi(x, b)\in C^0(D\times D^{\eps})$, the following inequality holds uniformly:
\begin{equation*}
    \|\partial_b w(b, t) \|_{L^2(D^{\eps})}^2 \le \|h(\cdot, 0) - f\|_{L^2(D)}^2 \int_{D\times D^{\eps}} |\partial_b \psi(x, b)|^2 dx d b < \infty.
\end{equation*}
\end{proof}
Let $\{E(t_m)\}_{m\ge 1}$ be a minimizing sequence for $E(t)$, then 
\begin{equation*}
\begin{aligned}
&\phantom{\mspace{2mu}=\;} \sum_{k,l=1}^{\infty} \lambda_k\lambda_l\Theta_k(t_m)[M_{l,k}(t_m) + S_{l,k}(t_m)]\Theta_l(t_m) \\&=\sum_{i=1}^n \left( |w(b_{i}(t_m), t_m)|^2+a_i^2(t_m)|\partial_b w(b_{i}(t_m), t_m)|^2\right) \to 0,
\end{aligned}
\end{equation*}
which implies that $w(b_{i}(t_m), t_m)\to 0$ and $|a_i(t_m)\partial_b w(b_i(t_m),t_m)| \to 0$. Note that $H^1(D^{\eps})$ embeds into $L^2(D^{\eps})$ compactly, there exists a subsequence $\{w(\cdot, t_{m_s})\}_{s\ge 1}$ converges to $\bar{w}\in H^1(D^{\eps})$ strongly (which is in $C^{0,\alpha}(D^{\eps})$). Here are two cases: 
\begin{itemize}
\item If the gradient dynamics $b_i(t)$ converges to $b_i^{\ast}$ as $t\to\infty$, then we must have  $\bar{w}(b_i^{\ast}) = 0$. If $\lim_{t\to\infty }a_i(t) \neq 0$ or does not exists, then $\partial_b \bar{w}(b_i^{\ast}) = 0$. 
\item If during the dynamics $b_i(t)$ is not converging, then there exists an open set 
$T\subset D^{\eps}$ that $T\subset \limsup_{\tau\to\infty} \{b_i(t)\}_{t\ge \tau}$ or equivalently, $T$ appears infinitely often in the dynamics. Then we must have $    \bar{w}(b) = 0,\,\forall b\in T$. 
\end{itemize}
\begin{remark}
    In mean-field representation that $n\to\infty$ and assume the limiting measure $\int_{\bbR} \mu(da, b, t)$ has full support on $D^{\eps}$, we immediately conclude that $\bar{w} \equiv 0$. Hence $\Theta_k(t)\to 0$ and $E(t)\to 0$ as $t\to\infty$. This implies the network will converge to the objective function. However, in a discrete setting, it becomes more complicated due to the existence of local minimums.
\end{remark}

\subsection{Generalized Fourier analysis of learning dynamics}
In this section, we study the convergence of the learning dynamics in terms of frequency modes, especially the asymptotic behavior for high-frequency components. The key relation is \eqref{eq:order2}, which implies one can study the evolution of the Fourier modes of $\partial^2_bw(b,t)$ to understand the evolution of error $h(b,t)-f(b)$. The other key fact is $\partial_b^4\phi(b)=\lambda_k^{-1}\phi(b)$ and the spectral decay rate $\lambda_k=\Theta(k^{-4})$.  However, due to the bounded domain of interest $b\in D$ and non-periodicity at the boundary, generalized Fourier modes have to be designed for our study. We introduce an orthonormal basis  $\theta_k(x)\in C(D)$ that solves the eigenvalue problem 
\begin{equation}\label{EQ: EIG SYS}
\begin{aligned}
     \theta_k^{(4)}(x) &= p_k \theta_k(x), \\
    \theta_k(1) &= \theta_k'(1) = 0,  \\
    \theta_k''(-1) &= \theta'''_k(-1) = 0.
\end{aligned}
\end{equation}

In Appendix~\ref{sec:GF} we provide an explicit characterization of $\theta_k$. In a nutshell, $p_k=\Theta(k^4)$ and $\theta_k$ form a complete orthonormal basis for $L^2(D)$ and it is close to a shifted Fourier mode when $k$ is relatively large. 

Let $\widehat{w}$ denote the \emph{generalized} discrete Fourier transform of $w$ on $D$:
\begin{equation}
    \widehat{w}(k, t) = \int_D \theta_k(b) w(b, t) db.
\end{equation}
where $\theta_k$ is the eigenfunction defined by~\eqref{EQ: EIG SYS}. We emphasize the following key relation (due to the property of $\texttt{ReLU}$ activation function): $\partial_b^2 w(b, t) = h(b, t)  - f(b)$ for $b\in D$ by~\eqref{eq:order2}.
 Therefore the generalized discrete Fourier transform of $\partial^2_b w$ will be the generalized discrete Fourier transform of the error $h(\cdot, t) - f(\cdot)$ on $D$. 
Using Lemma~\ref{LEM: 4TH DERIV}, $\partial^4_b w(b, t) = \sum_{k=1}^{\infty} \Theta_k(t)\phi_k(b) $ for $b\in D$. 
\begin{lemma} The following equality holds.
    $$\int_D w(b, t) \theta_k^{(4)}(b) db = \int_D \partial^2_b w(b, t) \theta_k''(b) db.$$
\end{lemma}
\begin{proof}
    Note $\theta_k''(-1) = \theta_k'''(-1) = 0, w(1,t)=\partial_bw(1,t)=0 ~(\mbox{due to } \phi_k(1)=\phi'_k(1)=0)$, 
    \begin{equation*}
    \begin{aligned}
            \int_D w(b, t) \theta_k^{(4)}(b) db &= w(b, t) \theta_k^{(3)}(b)\big|_{\partial D} -\int_D \partial_b w(b, t) \theta_k^{(3)}(b) db \\&= - \partial_b w(b, t) \theta_k^{(2)}(b)\big|_{\partial D} + \int_D \partial^2_b w(b, t) \theta_k''(b) db \\
            &= \int_D \partial^2_b w(b, t) \theta_k''(b) db.
    \end{aligned}
    \end{equation*}
\end{proof}
\begin{lemma}
There exists a constant $c > 0$ that 
    $|\widehat{w}(k, t)|\le c k^{-2}$.
\end{lemma}
\begin{proof}
    Using $\theta_k^{(4)}(b)  = p_k \theta_k(b)$, we have 
    \begin{equation*}
        |p_k \widehat{w}(k, t)| \le  \int_D |\partial^2_b w(b, t) \theta_k''(b) |db\le \|h(\cdot, 0) - f(\cdot)\|_{L^2(D)}  \| \theta_k'' \|_{L^2(D)}. 
    \end{equation*}
   Since $ \| \theta_k'' \|_{L^2(D)} = \cO(k^2)$ by Lemma~\ref{lem:B4} and $p_k=\Theta(k^4)$ by Lemma~\ref{lem:B1}, there exists a constant $c > 0$ that $|\widehat{w}(k, t)|\le \frac{c}{k^2}$.
\end{proof}
From now on, we assume that $b_i(t)$ is arranged in ascending order. Let $s = s(t)$ be the smallest index such that $\{ b_{j}(t) \}_{j\ge s}\subset {D}$. Then use integration by parts taking into account the boundary condition of $\theta_k, w$, 
\begin{equation}
\label{eq:integrationbypart}
\begin{aligned}
 \int_D \theta_k(b) \partial^4_b w(b, t) db &= \theta_k(b) \partial^3_b w(b, t)\Big|_{\partial D}  -\theta_k'(b)\partial^2_b w|_{\partial D} + \int_D \theta''_k(b) \partial^2_b w(b, t) db  \\
 &=\theta_k(b) \partial^3_b w(b, t)\Big|_{\partial D} -\theta_k'(b)\partial^2_b w|_{\partial D} +\int_D \theta_k^{(4)} w(b, t) db \\
 &= \theta_k(b) \partial^3_b w(b, t)\Big|_{\partial D}  -\theta_k'(b)\partial^2_b w|_{\partial D} +  p_k \widehat{w}(k, t) .
\end{aligned}  
\end{equation}
From the fact,  $\partial^2_b w(b,t)=h(b,t)-f(b)=\sum_{i=1}^n a_i(t)\chi(b_i(t))\sigma(b-b_i(t))-f(b)$ and $\theta_k(1)=\theta'_k(1)=0$, the two boundary terms can be explicitly written out as
\begin{equation}\label{EQ: H}
\begin{aligned}
\cH_k(t)\coloneqq \theta_k(b) \partial^3_b w(b, t)\Big|_{\partial D} =  -\theta_k(-1) \left[\sum_{i=1}^{s-1} a_i(t) \chi(b_i(t)) -f'(-1)\right],
\end{aligned}
\end{equation}
\begin{equation}\label{EQ: J}
  \cJ_k(t)\coloneqq -\theta'_k(b) \partial^2_b w(b, t)\Big|_{\partial D}=- \theta'_k(-1) \left[\sum_{i=1}^{s-1} \chi(b_i(t) )a_i(t) (1 + b_i(t))+f(-1)\right].
\end{equation}
Differentiating~\eqref{eq:integrationbypart} in time and using~\eqref{EQ: THETA K}, \eqref{eq:aux}, and~\eqref{eq:order4}, we have
\begin{equation}\label{EQ: PDE}
\begin{aligned}
        &\phantom{\mspace{2mu}=\;} p_{m} \partial_t \widehat{w}(m, t) + \cH_m'( t) + \cJ_m'( t) =  \sum_{k=1}^{\infty} \Theta_k'(t) \widehat{\phi_k}(m)\\
        &=  -  n \sum_{k=1}^{\infty}\widehat{\phi_k}(m) \left(  \int_{D^{\eps}} w(b, t) \phi_k(b) \mu_0(b, t) db +  \int_{D^{\eps}} \partial_b w(b, t) \phi'_k(b) \mu_2(b, t) db\right) , 
\end{aligned}
\end{equation}
where $\mu_0$ and $\mu_2$ are defined by the following positive distributions ($n$ can be finite)
\begin{equation}
\begin{aligned}
        \mu_0(b, t) &= \frac{1}{n}\sum_{i=1}^n \delta(b - b_i(t)), \\ \mu_2(b, t) &=  \frac{1}{n}\sum_{i=1}^n |a_i(t)|^2\delta(b - b_i(t)) .
\end{aligned}
\end{equation}
Apply the equality $    \sum_{k=1}^{\infty} \phi_k(x) \phi_k(y) = \delta(x - y)$ to~\eqref{EQ: PDE}, we find 
\begin{equation*}
\begin{aligned}
        &\phantom{\mspace{2mu}=\;}  \sum_{k=1}^{\infty}\widehat{\phi_k}(m) \int_{D^{\eps}} w(b, t) \phi_k(b) \mu_0(b, t) db  \\
        &= \int_{D^{\eps}} \int_D w(b, t) \sum_{k=1}^{\infty}  \phi_k(b) \phi_k(b') \mu_0(b, t) \theta_m(b') db' db \\
        &=  \int_{D^{\eps}} \int_D w(b, t) \delta(b - b') \mu_0(b, t) \theta_m(b') db' db \\
        &= \int_D w(b', t) \mu_0(b', t) \theta_m(b') db' = \widehat{w\mu_0}(m, t).
\end{aligned}
\end{equation*}
Similarly, 
\begin{equation*}
\begin{aligned}
 &\phantom{\mspace{2mu}=\;} \sum_{k=1}^{\infty} \widehat{\phi_k}(m)  \int_{D^{\eps}}  \partial_b w(b, t) \phi'_k(b) \mu_2(b, t) db \\&=
  \int_{D^{\eps}} \int_D \partial_b w(b, t) \sum_{k=1}^{\infty}  \phi_k'(b) \phi_k(b') \mu_2(b, t) \theta_m(b') db' db \\
  &=  \int_{D^{\eps}} \int_D \partial_b w(b, t) \delta'(b - b')\mu_2(b, t) \theta_m(b') db' db\\
 &=    \int_D \theta_m(b') \delta'(b - b') db' \int_{D^{\eps}}  \partial_b w(b, t) \mu_2(b, t)   db \\
 &=\int_{D} \partial_b w(b', t) \mu_2(b', t) \theta_m'(b') db' .
\end{aligned}
\end{equation*}
Therefore~\eqref{EQ: PDE} can be further reduced to 
\begin{equation}\label{EQ: NEW PDE}
    \partial_t \widehat{w}(m, t) = -\frac{1}{p_m} (\cH_m'(t) + \cJ_m'(t)) -\frac{n}{p_m}  \widehat{w\mu_0}(m, t)-\frac{n}{p_m} \int_{D} \partial_b w(b', t) \mu_2(b', t) \theta_m'(b') db' .
\end{equation}
Now we provide an estimate for the above equation and show a slow reduction of high-frequency modes in the initial error during the gradient flow. 
\begin{theorem}\label{THM: SLOW DECAY}
Assume that $\sup_{1\le i \le n} |a_i(t)|^2$ is uniformly bounded by $M > 0$ and the biases $\{b_i(0)\}$ are initially equispaced on $D$. If $\widehat{w}(m, 0) \neq 0$, then there exists a constant $\widetilde{C} > 0$ depending on the initialized loss that 
 \begin{equation}
    | \widehat{w}(m, t) | > \frac{1}{2}|\widehat{w}(m, 0)|,\quad 0\le t\le    \frac{ p_m |\widehat{w}(m, 0)|}{2n\widetilde{C}(m+1)}.
 \end{equation}
 Especially, denote the initial error in the generalized Fourier mode $\theta_m$ by $|\widehat{w}(m, 0)| > c' m^{-2}$ for certain $c' > 0$, then the error in the generalized Fourier mode $\theta_m$ takes at least $\cO(\frac{c'm}{n} ) $ time to get reduced by half following gradient decent dynamics. 
\end{theorem}
\begin{proof}
We estimate the contribution from each term on the right-hand side of~\eqref{EQ: NEW PDE}. In particular we have $\|\theta_k\|_{L^{\infty}(D)} = \cO(1)$,  $\|\theta_k'\|_{L^{\infty}(D)} = \cO(k)$ from Lemma~\ref{lem:B4}.
\begin{enumerate}
    \item First, there exists a constant $C > 0$ that 
    \begin{equation}\label{EQ: W MU0}
    \begin{aligned}
        \left|-\frac{n}{p_m}  \widehat{w\mu_0}(m, t) \right| &= \left|\frac{n}{p_m}  \int_D w(b', t) \mu_0(b', t) \theta_m(b') db'\right|\\&\le \frac{n}{p_m}  \|w(\cdot, t)\|_{C(D)} \|\theta_m\|_{C(D)} \le \frac{C n}{p_m}. 
    \end{aligned} 
    \end{equation}
    \item    Since $w\in H^2(D)$, it can be embedded into $C^{1,\alpha}(D)$ compactly, which means $\partial_b w$ is uniformly bounded on $D$, hence there exists a constant $C'$ that
    \begin{equation*}
        \begin{aligned}
            \left|-  \frac{n}{p_m} \int_{D} \partial_b w(b', t) \mu_2(b', t) \theta_m'(b') db'   \right|&\le \frac{n}{p_m}  
           \|\theta_m'\|_{C(D)} \|\partial_b w\|_{C(D)} \int_D \mu_2(b, t) db  \\&\le \frac{C' M m n}{p_m} .
        \end{aligned}
    \end{equation*}
    \item  Using the definition of $\cH_m$ in~\eqref{EQ: H}  and $\cJ_m$ in~\eqref{EQ: J}, now we give the upper bounds of $|\cH_m(t) - \cH_m(0)|$ and $|\cJ_m(t) - \cJ_m(0)|$.
    Since initially $b_i(0)\in D$ and $\chi(b)\le 1$, we have  
    \begin{equation}\label{EQ: HJ1}
        \begin{aligned}
              |\cH_m(t) - \cH_m(0)| = |\theta_m(-1)| \left|\sum_{i=1}^{s(t)-1} a_i(t) \chi(b_l(t)) \right| \le C''\sqrt{M} |s(t) - 1|
        \end{aligned}
    \end{equation}
    and 
    \begin{equation}\label{EQ: HJ2}
        \begin{aligned}
               |\cJ_m(t) - \cJ_m(0)| = |\theta'_m(-1) |\left|  \sum_{i=1}^{s(t) -1} a_i(t)\chi(b_i(t)) (1 + b_i(t))\right|\le C''' m \sqrt{M} |s(t) - 1|.
        \end{aligned}
    \end{equation}
    Now we estimate the number $s(t)$. Because the biases have a finite propagation speed
    \begin{equation*}
        \left|\frac{d}{dt} b_i(t)\right| \le |a_i|\int_{D} |h(x, t) - f(x)||\partial_b \psi(x,b_i(t))| dx \le K\coloneqq C''''\sqrt{ M}  \|h(\cdot, 0) - f(\cdot)\|_{L^2(D)} .
    \end{equation*}
    Therefore there are at most $\frac{1}{2}nKt$ initially evenly spaced biases moving into the transition interval. That gives $|s(t) - 1|\le \frac{1}{2} nKt$.  
\end{enumerate}
The above estimates imply that there exists a constant $\widetilde{C} > 0$ that 
    \begin{equation}\label{EQ: INEQ}
    \begin{aligned}
       |\widehat{w}(m, t)| - |\widehat{w}(m, 0)|   &\ge - \frac{n}{p_m} (C'M m  + C) t -   \frac{1}{p_m} \left(C''\sqrt{M} + C''' m \sqrt{M}\right) \frac{1}{2} K n t\\
       &\ge -\frac{\widetilde{C} n}{p_m}(m+1) t.
    \end{aligned}
    \end{equation}
When $\widehat{w}(m, 0)\neq 0$, we solve the lower bound of the \emph{half-reduction} time $\tau > 0$ that 
    \begin{equation*}
    \begin{aligned}
    \frac{1}{2}|\widehat{w}(m, 0)|  &= \frac{\widetilde{C} n}{p_m} \left( m+1\right) \tau \quad\Longrightarrow \quad  \tau = \frac{p_m |\widehat{w}(m, 0)|}{2n\widetilde{C} \left(m+1\right) }.
    \end{aligned}
    \end{equation*}
In particular, if $|\widehat{w}(m, 0)| >c'm^{-2}$ for certain $c' > 0$, the \emph{half-reduction} time is at least $\cO(\frac{c'm}{n})$.
\end{proof}

To keep a discretized gradient descent method close to the continuous gradient flow, the learning rate or step size should be small. In practice, one typically takes $\Delta t = \cO(\frac{1}{n})$, it will take at least $\cO(m)$ time steps to reduce the initial error $h(x, 0) - f(x)$ in the generalized Fourier mode $\theta_m$ by half. 
It is also worth noticing that this phenomenon does not depend on the convergence of the trainable parameters. 

A lower bound estimate only provides an optimistic scenario which may not be sharp.  When the network width and the frequency mode satisfy different scaling laws, improved estimates can be obtained. See Appendix~\ref{sec:improved} for improved estimates, numerical evidence, and the proof of Theorem~\ref{THM: PAPER SLOW DECAY 2}.

\section{Rashomon set for two-layer $\texttt{ReLU}$ neural networks}
\label{sec:Rashomon}
In~\cite{semenova2019study,semenova2022existence}, the authors claimed that the measure of the so-called Rashomon set can be used as a criterion to see if a simple model exists for the approximation problem. Let $D = B_d(1)$ be the unit ball in $\bbR^d$. The Rashomon set $\calR_{\eps}$ for the two-layer neural network class $\calH_n$ is defined as the following
\begin{equation*}
    \calR_\eps \coloneqq \big\{h\in \calH_n\mid \|h - f\|_{L^2(D)}\le \eps \|f\|_{L^2(D)}\big\},
\end{equation*}
 where $m$ is the number of parameters for $\calH_n$. If we normalize the measure for $\calH_n$, the measure of the Rashomon set quantifies the probability that the loss is under a certain threshold for a random pick of parameters.
 The main purpose of this section is to characterize the probability measure of the Rashomon set with $\calH_n$ representing the two-layer \texttt{ReLU} neural networks 
\begin{equation*}
      h(\bmx) = \frac{1}{n}\sum_{j=1}^n a_j \sigma(\bmw_j \cdot \bmx - b_j )   + \bmv\cdot \bmx + c,\quad \bmx\in D\subset \bbR^d,
\end{equation*}
where the parameters $\{a_j\}_{j=1}^n$ are bounded by $[-A, A]$ and $\{b_i\}_{i=1}^n \subset [-1,1]$.
It is already known that this network can approximate a function $f$ with an error no more than $\mathcal{O}(\sqrt{d + \log n }\; n^{-1/2 - 1/d})$ if the Fourier transform $\widehat{f}$ satisfies $\int_{\mathbb{R}^d} |\widehat{f}(\zeta)| \|\zeta\|_{1}^2 d\zeta < \infty$~\cite{klusowski2018approximation}. The general bounds in H\"older space have been studied in~\cite{mao2023rates}.
Taking Laplacian on $h$, 
\begin{equation*}
    \Delta h(\bmx) = \frac{1}{n}\sum_{j=1}^n a_j \Delta \sigma(\bmw_j \cdot \bmx - b_j )  = \frac{1}{n}\sum_{j=1}^n a_j  \delta(\bmw_j \cdot \bmx - b_j ).
\end{equation*}
\begin{theorem}\label{THM: RASH}
  Suppose $f\in C(D)$ such that there exists $g\in C_0^2(D)$ that $\Delta g = f$, then the Rashomon set $\mathcal{R}_{\eps}\subset \calH_n$ satisfies 
    \begin{equation*}
        \bbP(\mathcal{R}_{\eps}) \le \exp\left(-\frac{n(1-\eps)^2 \|f\|^4_{L^2(D)}}{2A^2 \kappa^2}\right), \quad \kappa\coloneqq\sup_{(\bmw, b)}\int_{\{\bmx\in D, \bmw\cdot \bmx = b\}}g(\bmx) dH_{d-1}(\bmx).
    \end{equation*}
\end{theorem}
\begin{proof}
    Let $g\in C_0^{2}(D)$ that solves $\Delta g(\bmx) = f(\bmx)$ with Dirichlet boundary condition.  
\begin{equation*}
\begin{aligned}
    \aver{\Delta h(\bmx), g(\bmx)} = \frac{1}{n}\sum_{j=1}^n X_j,\quad X_j\coloneqq a_j \int_{\bmw_j\cdot \bmx = b_j} g(\bmx) d H_{d-1}(\bmx).
\end{aligned}
\end{equation*}
The random variables $X_j$ are i.i.d and bounded by $[-A \kappa, A \kappa]$ where $\kappa$ is an absolute constant defined by
\begin{equation*}
    \kappa\coloneqq \sup_{(\bmw_j, b_j)\in\bbS^{d-1}\times [-1,1]}\int_{\bmw_j\cdot \bmx = b_j} g(\bmx) d H_{d-1}(\bmx).
\end{equation*}
Then use Hoeffding's inequality, 
\begin{equation*}
    \bbP\left[  \frac{1}{n} \sum_{j=1}^n X_j - \bbE[X_j]\ge t \right] \le \exp\left( - \frac{n t^2}{ 2  A^2 \kappa^2 }\right).
\end{equation*}
Especially if $\bbE(a_j)  = 0$ hence $\bbE(X_j) = 0$, then 
\begin{equation*}
        \bbP\left[   \aver{\Delta h(\bmx), g(\bmx)} \ge t \right] \le \exp\left( - \frac{n t^2}{ 2  A^2 \kappa^2 }\right).
\end{equation*}
Using Green's formula, 
\begin{equation*}
     \aver{\Delta h(\bmx), g(\bmx)} - 
     \aver{ h(\bmx), \Delta g(\bmx)} = \int_{\partial D} \left( \partial_n h(\bmx) f(\bmx) - \partial_n f(\bmx) h(\bmx) \right) ds = 0, 
\end{equation*}
then using the network to approximate $\Delta g = f$ is difficult if $\kappa$ is small in the sense that
\begin{equation*}
\begin{aligned}
    \bbP\left[\|h(\bmx) - \Delta g(\bmx)\|_{L^2(D)}\le \eps \|\Delta g\|_{L^2(D)}\right] &\le 
    \bbP\left[ \aver{ h(\bmx), \Delta g (\bmx)} \ \ge (1-\eps) \|\Delta g\|_{L^2(D)}^2 \right] \\& \le \exp\left( - \frac{n (1-\eps)^2\|\Delta g\|_{L^2(D)}^4}{ 2 A^2 \kappa^2 }\right).
\end{aligned}
\end{equation*}
Here we have used the fact that $\|h(\bmx) - \Delta g(\bmx)\|_{L^2(D)}\le \eps \|\Delta g\|_{L^2(D)}$ implies both $\|h\|_{L^2(D)}\ge (1-\eps)\|\Delta g\|_{L^2(D)}$ and
\begin{equation*}
    \aver{h, \Delta g} \ge \frac{1 - \eps^2}{2} \|\Delta g\|^2_{L^2(D)}  + \frac{1}{2}\|h\|^2_{L^2(D)} \ge (1 - \eps) \|\Delta g\|_{L^2(D)}^2.
\end{equation*}
\end{proof}
If $f$ oscillates with frequency $\nu$ in every direction, then $\kappa\approx \nu^{-2}$ and the probability measure of Rashomon set scales as $\exp(-\mathcal{O}(\nu^{-4}))$ which gives another perspective why oscillatory functions are difficult to approximate by neural networks in general. The proof and extension to more general activation functions are provided in Appendix~\ref{SEC: RASH BOUND}. 


\section{Further discussions}\label{sec:discussions}
In this study, it is shown that the use of highly correlated activation functions in a two-layer neural network makes it filter out fine features (high-frequency components) when finite machine precision is imposed, which is an implicit regularization in practice. Moreover, increasing the network width does not improve the numerical accuracy after a certain threshold is reached, although the universal approximation property is proved in theory. 
The smoother the activation function is, the faster the Gram matrix spectrum decays (see Remark~\ref{re:relu-k} and Appendix~\ref{sec:app-discussion}), and hence the stronger the regularization is. We plan to investigate how a multi-layer network could overcome these issues through effective decomposition and composition in our future work.





\section*{Acknowledgments}
The authors gratefully acknowledge the editor and the anonymous referees for their constructive and insightful comments, which have substantially improved the quality and presentation of the paper.

\section*{Funding}
S. Zhang was partially supported by
start-up fund P0053092 from The Hong Kong Polytechnic University.
H. Zhao was partially supported by NSF grants DMS-2309551, and DMS-2012860. Y. Zhong was partially supported by NSF grant DMS-2309530 and H. Zhou was partially supported by NSF grant DMS-2307465.


\bibliographystyle{plain}   
\bibliography{references}


\appendix


\section{Properties of eigenfunctions}\label{sec:eigen}
We show properties of the eigenfunctions $\phi_k$ for the Gram kernel defined by~\eqref{EQ: EIGEN} or equivalently by~\eqref{eq:modified}.
\begin{lemma}\label{LEM: 4TH DERIV}
$\phi_k$ is a cubic polynomial over $(-1-\eps, -1)$ and
\begin{equation}\label{eq:order4}
    \partial_b^4 \phi_k(b) = \begin{cases}
        0, & b\in (-1 - \eps, -1),\\
        \lambda_k^{-1} \phi_k(b), & b\in (-1,1),
    \end{cases}
\end{equation}
where $\partial_b \phi_k$ is continuous at $b=-1$ but $\partial^2_b\phi_k$ has a jump at $b=-1$.
\end{lemma}
\begin{proof}
    We use the definition of eigenfunction
    \begin{equation*}
        \int_{D^{\eps}} \cG(b, b') \phi_k(b') db' = \lambda_k \phi_k(b).
    \end{equation*}
    For $b\in (-1, 1)$, $\chi(b) = 1$, then differentiating both sides 
    \begin{equation*}
        \lambda_k \partial^4_b \phi_k(b) = \partial^4_b \int_{D^{\eps}} \int_{D} \sigma(x - b) \sigma(x - b')\chi(b') \phi_k(b') dx d b' =\int_{D^{\eps}} \delta(b-b')\chi(b') \phi_k(b') d b' = \phi_k(b).   
    \end{equation*}
    For $b\in (-1-\eps, -1)$, $\chi(b)$ is quadratic, $\sigma(x-b) = x- b$ for $x\in D$, we differentiate both sides 
    \begin{equation*}
        \lambda_k \partial^4_b \phi_k(b) = \partial^4_b\left(  \chi(b)\int_{D^{\eps}} \int_{D} (x - b) \sigma(x - b')\chi(b') \phi_k(b') dx d b' \right) = 0.
    \end{equation*}
    Now, we compute $\partial_b \phi_k$ and $\partial^2_b \phi_k$ across $b = -1$. Note $\chi'(-1) = 0$, then 
    \begin{equation*}
    \begin{aligned}
        \lim_{b\to -1^{-}} \partial_b \phi_k(b) &=  \frac{1}{\lambda_k} \lim_{b\to -1^{-}}\partial_b \left(  \chi(b)\int_{D^{\eps}} \int_{D} (x - b) \sigma(x - b')\chi(b') \phi_k(b') dx d b' \right) \\
        &=\frac{1}{\lambda_k} \chi'(-1) \lim_{b\to -1^{-}}\left(\int_{D^{\eps}} \int_{D} (x - b) \sigma(x - b')\chi(b') \phi_k(b') dx d b' \right)\\
        &\quad + \frac{1}{\lambda_k} \chi(-1) \lim_{b\to -1^{-}}\left(\int_{D^{\eps}} \int_{D} (-1) \sigma(x - b')\chi(b') \phi_k(b') dx d b' \right)\\
        &= \frac{1}{\lambda_k} \left(\int_{D^{\eps}} \int_{D} (-1) \sigma(x - b')\chi(b') \phi_k(b') dx d b' \right)
    \end{aligned}
    \end{equation*}
    and because $\sigma'(x - b)  = 1$ if $b=-1$ and $x\in D$,
    \begin{equation*}
    \begin{aligned}
        \lim_{b\to -1^{+}} \partial_b \phi_k(b) &=  \frac{1}{\lambda_k} \lim_{b\to -1^{+}}\partial_b \left(  \int_{D^{\eps}} \int_{D} \sigma(x - b) \sigma(x - b')\chi(b') \phi_k(b') dx d b' \right) \\
        &=\frac{1}{\lambda_k}\lim_{b\to -1^{+}}\left(\int_{D^{\eps}} \int_{D} -\sigma'(x - b) \sigma(x - b')\chi(b') \phi_k(b') dx d b' \right) \\
        &= \frac{1}{\lambda_k}\left(\int_{D^{\eps}} \int_{D} (-1) \sigma(x - b')\chi(b') \phi_k(b') dx d b' \right) .
    \end{aligned}
    \end{equation*}
    For $\partial^2_b \phi_k$ across $b=-1$, following a similar process, 
    \begin{equation*}
    \begin{aligned}
        \lim_{b\to -1^{-}} \partial^2_b \phi_k(b) &=  \frac{1}{\lambda_k} \lim_{b\to -1^{-}}\partial^2_b \left(  \chi(b)\int_{D^{\eps}} \int_{D} (x - b) \sigma(x - b')\chi(b') \phi_k(b') dx d b' \right) \\
        &=\frac{1}{\lambda_k} \chi''(-1) \lim_{b\to -1^{-}}\left(\int_{D^{\eps}} \int_{D} (x - b) \sigma(x - b')\chi(b') \phi_k(b') dx d b' \right)\\
        &= -\frac{1}{\lambda_k} \frac{2}{\eps^2}\left(\int_{D^{\eps}} \int_{D} (x + 1) \sigma(x - b')\chi(b') \phi_k(b') dx d b' \right)
    \end{aligned}
    \end{equation*}
    and 
    \begin{equation*}
    \begin{aligned}
        \lim_{b\to -1^{+}} \partial^2_b \phi_k(b) &=  \frac{1}{\lambda_k} \lim_{b\to -1^{+}}\partial^2_b \left(  \int_{D^{\eps}} \int_{D} \sigma(x - b) \sigma(x - b')\chi(b') \phi_k(b') dx d b' \right) \\
        &=\frac{1}{\lambda_k}\lim_{b\to -1^{+}}\left(\int_{D^{\eps}} \int_{D} \delta(x - b) \sigma(x - b')\chi(b') \phi_k(b') dx d b' \right) \\
        &=\frac{1}{\lambda_k} \int_{D^{\eps}} \sigma(-1- b')\chi(b') \phi_k(b') d b'  = \frac{1}{\lambda_k} \int_{-1-\eps}^{-1} (-1- b') \chi(b') \phi_k(b') d b'. 
    \end{aligned}
    \end{equation*}
    As $\eps \to 0$, $\lim_{b\to -1^{+}} \partial^2_b \phi_k(b)=O(\eps)$ while $ \lim_{b\to -1^{-}} \partial^2_b \phi_k(b) =O(\eps^{-2})$.
\end{proof}

\section{Generalized Fourier modes}\label{sec:GF}
In this section, we study the eigenfunctions $\phi_k$ for the one-dimensional continuous Gram kernel defined in \eqref{EQ: EIGEN}, which are exactly the generalized Fourier modes $\theta_k$ defined in~\eqref{EQ: EIG SYS}.

\begin{lemma}
\label{lem:B1}
    The eigenfunction $\theta_k$ satisfies 
    \begin{equation*}
        \theta_k(x) = A_k \cosh(c_k x) + B_k \sinh(c_k x) + C_k \cos(c_k x) + D_k\sin(c_k x),
    \end{equation*}
    where $\tanh(c_k) \tan(c_k) = \pm 1$, $p_k = c_k^4$, and $c_{2k+1}\in ((k+\frac{1}{4})\pi, (k+\frac{1}{2})\pi)$,  $c_{2k+2}\in ((k + \frac{1}{2})\pi, (k+\frac{3}{4})\pi)$, $k\ge 0$.
\end{lemma}
\begin{proof}
    The form of the eigenfunction is standard. Using the boundary conditions, $\theta_k(1) = \theta_k''(-1) = 0$, 
    \begin{equation*}
    \begin{aligned}
            A_k\cosh(c_k) + B_k\sinh(c_k) + C_k \cos(c_k)+ D_k\sin(c_k) &= 0,\\
            A_k\cosh(c_k) - B_k\sinh(c_k) - C_k\cos(c_k) + D_k\sin(c_k) &= 0,
    \end{aligned}
    \end{equation*}
    which means 
    \begin{equation*}
         A_k\cosh(c_k) + D_k \sin(c_k) =0,\quad B_k\sinh(c_k) + C_k\cos(c_k) = 0.
    \end{equation*}
    The other two boundary conditions are
    \begin{equation*}
        \begin{aligned}
            A_k\sinh(c_k) + B_k\cosh(c_k) - C_k\sin(c_k) + D_k\cos(c_k) &= 0,\\
            -A_k\sinh(c_k) + B_k\cosh(c_k) - C_k\sin(c_k) - D_k\cos(c_k) &= 0,  
        \end{aligned}
    \end{equation*}
    which means
    \begin{equation*}
        A_k\sinh(c_k) + D_k\cos(c_k) = 0,\quad B_k\cosh(c_k) - C_k\sin(c_k) = 0.
    \end{equation*}
    If $C_k\neq 0$, then $B_k = -C_k\frac{\cos(c_k)}{\sinh(c_k)} = C_k\frac{\sin(c_k)}{\cosh(c_k)}$, which is $\tan(c_k)\tanh(c_k) = -1$. If $D_k\neq 0$, then we can derive $\tan(c_k)\tanh(c_k) = 1$. Let $r(x) = \tan(x)\tanh(x)$, for $x\in (0 , \frac{\pi}{2})$ 
   or $x\in (n\pi+\frac{1}{2}\pi, n\pi + \frac{3\pi}{2})$, $n\in\bbZ$, the function $r$ is monotone, therefore $c_{2k+1}\in ((k+\frac{1}{4})\pi, (k+\frac{1}{2})\pi)$,  $c_{2k+2}\in ((k + \frac{1}{2})\pi, (k+\frac{3}{4})\pi)$, $k\ge 0$.
\end{proof}
From the above analysis, the eigenfunctions are 
\begin{equation*}
    \begin{aligned}
        \theta_{2k+1}(x) &= C_{2k+1}\left(-\frac{\cos(c_{2k+1})}{\sinh(c_{2k+1})} \sinh(c_{2k+1} x)  +  \cos(c_{2k+1} x)\right),\\
        \theta_{2k+2}(x) &= D_{2k+2}\left( - \frac{\sin(c_{2k+2})}{\cosh(c_{2k+2})}\cosh(c_{2k+2} x) + \sin(c_{2k+2} x)  \right),
    \end{aligned}
\end{equation*}
which shows $\{\theta_k\}_{k\ge 1}$ are actually the eigenfunctions of the Gram kernel $\cG$ in~\eqref{EQ: GRAM KERNEL}. 

\begin{lemma}
    $\{\theta_k\}_{k\ge 1}$ forms an orthonormal basis of $L^2(D)$.
\end{lemma}
\begin{proof}
Since each eigenvalue is simple from Theorem~\ref{lem:B4} and we verify the following equality using integration by parts, 
    \begin{equation*}
       p_i \int_D \theta_i(x) \theta_j(x) dx =  \int_D \theta_i^{(4)} (x) \theta_j(x) dx  = \int_D \theta_i(x) \theta_j^{(4)} (x) dx = p_j  \int_D \theta_i(x) \theta_j(x) dx.
    \end{equation*}
    Thus $\{\theta_i\}_{i\ge 1}$ forms an orthonormal basis.
\end{proof} 
\begin{lemma}
    The eigenfunctions $\{\theta_k\}_{k\ge 1}$ can form a complete basis for $L^2(D)$. 
\end{lemma}
\begin{proof}
Suppose $\{\theta_m\}_{m=1}^{\infty}$ is not complete, then there exists a nonzero $\gamma \in L^2(D)$ that $\widehat{\gamma}(m) = 0$ for all $m\in \bbN$, then using the fact that $\{\theta_m\}_{m\ge 1}$ are the eigenfunctions of the Gram kernel $\cG$ in~\eqref{EQ: GRAM KERNEL}, note $\cG(x,y) = \cG(y,x)$, it is self-adjoint, then by Hilbert-Schmidt theorem,  
\begin{equation}
   \int_D \cG(x, y) \gamma(y) dy = 0,
\end{equation}
and we differentiate the above equation 4 times on both sides, which leads to $\gamma^{(4)}(x) = 0$ which means $\gamma$ should be a cubic polynomial with boundary conditions $\gamma(1) = \gamma'(1) = \gamma''(-1) = \gamma'''(-1) = 0$, which means $\gamma\equiv 0$, it is a contradiction with our assumption of $\|\gamma\|_D\neq 0$.
\end{proof}

Next, we show the eigenfunctions are \emph{almost} Fourier modes.
\begin{theorem}
\label{lem:B4}
The following statements hold
\begin{enumerate}
\item 
$(k+\frac{1}{4})\pi <c_{2k+1} < (k+\frac{1}{4})\pi + e^{-2c_{2k+1}}$ and $(k+\frac{3}{4})\pi> c_{2k+2} > (k+\frac{3}{4})\pi  -e^{-2c_{2k+2}}$.
  \item $A_k, B_k = \cO(e^{-c_k})$ and $C_k, D_k = \cO(1)$. 
    \item $\|\theta_k\|_{L^{\infty}(D)} = \cO(1)$,  $\|\theta_k'\|_{L^{\infty}(D)} = \cO(k)$ and $\|\theta_k''\|_{L^{\infty}(D)} = \cO(k^2)$. 
    \item $\|\theta_{2k+1}(x) - \cos(c_{2k+1} x)\|_{L^2(D)} = \cO(k^{-1/2})$ and $\|\theta_{2k+2}(x) - \sin(c_{2k+2} x)\|_{L^2(D)} = \cO(k^{-1/2})$.
\end{enumerate}
\end{theorem}
\begin{proof}
    For $c_{2k+1}$, from the relation $\tan(c_{2k+1}) = \coth(c_{2k+1}) = 1 + 2/ (e^{2c_{2k+1}} - 1)$, we obtain that 
    \begin{equation*}
      c_{2k+1} - (k+\frac{1}{4})\pi <   \tan(c_{2k+1} - (k+\frac{1}{4})\pi) = \frac{\tan(c_{2k+1}) - 1}{1 + \tan(c_{2k+1})} = e^{-2 c_{2k+1} }.
    \end{equation*}
    The first inequality uses $x < \tan x$ for $x\in (0, \frac{\pi}{4})$. The result for $c_{2k+2}$ follows a similar derivation.
    
We only prove for $\theta_{2k+1}$, the proof is similar for $\theta_{2k+2}$. 
    \begin{equation}\label{EQ: NORM}
    \begin{aligned}
        &\phantom{\mspace{2mu}=\;} \int_{D} \left(-\frac{\cos(c_{2k+1})}{\sinh(c_{2k+1})} \sinh(c_{2k+1} x)  +  \cos(c_{2k+1} x)\right)^2 dx \\&=  \int_D   \left(\frac{\cos(c_{2k+1})}{\sinh(c_{2k+1})} \sinh(c_{2k+1} x) \right)^2 dx + \int_D \cos^2(c_{2k+1} x) dx \\&\quad - 2\int_D  \frac{\cos(c_{2k+1})}{\sinh(c_{2k+1})} \sinh(c_{2k+1} x)\cos(c_{2k+1} x) dx  \\
        &=  \left(\frac{\cos(c_{2k+1})}{\sinh(c_{2k+1})}\right)^2 \left[\frac{\sinh(2c_{2k+1} )}{2c_{2k+1}} - 1\right] + \left[\frac{\sin(2 c_{2k+1})}{2c_{2k+1}} + 1\right] \\
        &= \frac{\cos^2(c_{2k+1})\coth(c_{2k+1})}{c_{2k+1}} + \frac{\sin(2c_{2k+1})}{2c_{2k+1}} + 1 - \left(\frac{\cos(c_{2k+1})}{\sinh(c_{2k+1})}\right)^2 
        =1 + \cO(k^{-1}).
    \end{aligned}
    \end{equation}
    Therefore $C_{2k+1} = 1 + \cO(k^{-1})$ and $|B_{2k+1}|\le \sinh(c_{2k+1})^{-1} |C_{2k+1}|=\cO(e^{-c_{2k+1}})$. The norm $\|\theta_{2k+1}\|_{L^{\infty}(D)} \le 2 C_{2k+1} = \cO(1)$ and 
\begin{equation*}
    \theta_{2k+1}'(x) = C_{2k+1}c_{2k+1}\left(-\frac{\cos(c_{2k+1})}{\sinh(c_{2k+1})}\cosh(c_{2k+1} x) - \sin(c_{2k+1} x)\right),
\end{equation*}
which means $\|\theta_{2k+1}'\|_{L^{\infty(D)}} = \cO(k)$ and similarly,  $\|\theta_{2k+1}''\|_{L^{\infty(D)}} = \cO(k^2)$. We also can see from~\eqref{EQ: NORM} that  
    \begin{equation*}
        \int_D |\theta_{2k+1}(x) - C_{2k+1}\cos(c_{2k+1} x)|^2 dx = \cO(k^{-1}),
    \end{equation*}
which means $\theta_k \sim C_{2k+1}\cos(c_{2k+1} x) + \cO(k^{-{1/2}}) = \cos(c_{2k+1} x) + \cO(k^{-1/2})$. 
\end{proof}

\section{Improved bounds for learning dynamics}
\label{sec:improved}
\subsection{Part I}\label{sec:part1}
In the first step, we provide an improved estimate for $$\int_{D} \partial_b w(b', t) \mu_2(b', t) \theta_m'(b') db' = \frac{1}{n}\sum_{i=1}^n \partial_b w(b_i(t), t) |a_i(t)|^2 \theta'_m(b_i(t)).  $$ 
Here we slightly abuse the notation and treat $\partial_b w$ as zero outside $D$. Define the piecewise linear continuous function $\cA(\cdot, t) \in C(D)$ that $\cA(b_i(t), t) = |a_i(t)|^2$, 
then we have the following equation using integration by parts,
\begin{equation*}
\begin{aligned}
    &\phantom{\mspace{2mu}=\;} \frac{1}{n}\sum_{i=1}^n \partial_b w(b_i(t), t) |a_i(t)|^2 \theta_m'(b_i(t)) - \int_D \partial_b w(b, t) \cA(b, t)  \theta_m'(b) db \\
     &= -\int_D \texttt{disc}(b, t) \partial_b \left[ \partial_b w(b, t) \cA(b, t)  \theta_m'(b) \right] db, 
\end{aligned}
\end{equation*}
where $\texttt{disc}$ is the discrepancy function 
\begin{equation*}
    \texttt{disc}(b, t) = \frac{1}{n} \sum_{i=1}^n \one_{[-1, b)}(b_i(t)) - b
\end{equation*}
and $\one_{S}$ denotes the characteristic function on the set $S$.
\begin{lemma}\label{LEM: VAR}
     Let $\cV(t)$ be the total variation of $\partial_b w(b, t) \cA(b, t) \theta_m'(b)$, then there exists an absolute constant $C > 0$ that $$\cV(t) \le C V(\cA) m^2 \|h(\cdot, t) - f(\cdot)\|_{L^2(D)}, $$
     where $V(\cA)$ is the total variation of $\cA$: 
     \begin{equation}\label{eq:TV}
         V(\cA) = a_1^2(t) + \sum_{i=1}^{n-1} |a_i^2(t) - a_{i+1}^2(t)| + a_n^2(t).
     \end{equation}
\end{lemma}
\begin{proof}
The result comes from the fact $\partial^2_b w(b, t)=h(b,t)-f(b) $, $|\theta^{''}_m(b)|=O(m^2)$,
and bounded variation functions form a Banach algebra.
\end{proof}
\begin{lemma}\label{LEM: BOUND 2}
    Suppose $\sup_{1\le i\le n}|a_i(t)|^2$ is uniformly bounded by $M$, the total variation of $\cA$ is bounded by $M'$, and the initial biases are equispaced distributed, then there exists a constant $C > 0$ such that 
    \begin{equation*}
        \left|\int_{D} \partial_b w(b', t) \mu_2(b', t) \theta_m'(b') db' \right| \le  C \|h(\cdot, 0) - f(\cdot)\|_{L^2(D)}\left(   1  +   \left( \frac{2}{n} + Kt\right) m^2 \right).
    \end{equation*}
The constant $K = \sqrt{2M}  \|h(x, 0) - f(x)\|_{L^2(D)}$. 
\end{lemma}
\begin{proof}
Let $L=\|\theta_m\|_{C(D)}=\cO(1)$, we apply integration by parts and there exists a constant $C' > 0$ such that 
\begin{equation}\label{EQ: FOUR EST}
\begin{aligned}
   &\phantom{\mspace{2mu}=\;} \left|\int_D \partial_b w(b, t) \cA(b, t) \theta'_m(b) db \right| \\&\le \left| \int_{D} \theta_m(b) \partial_b\left(\partial_b w(b, t)\cA(b, t) \right) db \right| + \left|\theta_m(b) \partial_b w(b, t)\cA(b, t) \Big|_{\partial D}\right|\\
    &\le L\left[\int_{D} |\partial_b^2 w(b, t) \cA(b, t)  | db  + \int_{D} |\partial_b w(b, t)||\partial_b \cA(b, t) | db + M\|\partial_b w(\cdot,t)\|_{C(D)}\right]  \\
  &\le L \left[\|\partial_b^2 w(\cdot, t) \|_{L^2(D)} \|\cA(\cdot, t)\|_{L^2(D)} + \|\partial_b w(\cdot,t)\|_{C(D)} V(\cA) +  M\|\partial_b w(\cdot,t)\|_{C(D)}\right] \\
  &\le C'\|h(\cdot, t) - f(\cdot)\|_{L^2(D)} (M+M').
\end{aligned}
\end{equation}

Therefore we have a new estimate:
\begin{equation*}
\begin{aligned}
      \left|\int_D \partial_b w(b, t) \mu_2(b, t) \theta'_m(b) db \right| &\le C' \|h(\cdot, t) - f(\cdot)\|_{L^2(D)} (M+M') \\&\quad +  \left|\int_D \texttt{disc}(b, t) \partial_b \left[ \partial_b w(b, t) \cA(b, t) \theta_m'(b) \right] db \right|.
\end{aligned}
\end{equation*}
Notice that this bound will be better than the constant bound in Theorem~\ref{THM: SLOW DECAY} if the second term on the right-hand side is relatively small. Next, we characterize the discrepancy function $\texttt{disc}(b, t)$.
\begin{equation*}
\begin{aligned}
    \left|\texttt{disc}(b, t) - \texttt{disc}(b, 0) \right|&= \left|\frac{1}{n} \sum_{i=1}^n \one_{[-1, b)}(b_i(t)) - b - \left(\frac{1}{n} \sum_{i=1}^n \one_{[-1, b)}(b_i(0)) - b \right) \right| \\
    &= \frac{1}{n} \left|\sum_{i=1}^n \left(\one_{[-1, b)}(b_i(t)) - \one_{[-1, b)}(b_i(0))\right) \right|. 
\end{aligned}
\end{equation*}
Because 
\begin{equation}\label{EQ: FINITE SPEED}
   \begin{split}
       \left| \frac{d}{dt} b_i(t) \right| 
       & \le |a_i(t)|\int_D |h(x,t) - f(x)| |\partial_b \psi(x,b_i(t))| dx 
       \\ & \le K\coloneqq C'''\sqrt{M}  \|h(x, 0) - f(x)\|_{L^2(D)},
   \end{split}
\end{equation}
and within time $t$, the maximum distance of propagation is $K t$ for each bias. Therefore, 
\begin{equation*}
\begin{aligned}
        \frac{1}{n} \left|\sum_{i=1}^n \left(\one_{[-1, b)}(b_i(t)) - \one_{[-1, b)}(b_i(0))\right) \right| &\le \frac{1}{n}\left|\sum_{i=1}^n \left(\one_{[-1+Kt, b-Kt)}(b_i(0)) - \one_{[-1, b)}(b_i(0))\right) \right|  \\
        &= \frac{1}{n} \sum_{i=1}^n \one_{[-1, -1+Kt)\cup [b - Kt, b)}(b_i(0)) \le Kt + \frac{1}{n}.
\end{aligned}
\end{equation*}
Therefore, using $|\texttt{disc}(b, 0)|\le\frac{1}{n}$ for equispaced biases, 
\begin{equation*}
\begin{aligned}
     \left|\int_D \partial_b w(b, t) \mu_2(b, t) \theta'_m(b) db \right|&\le C' \|h(x, t) - f(x)\|_{L^2(D)} (M+M') +  \left( \frac{2}{n} + Kt\right) \cV(t) \\
    &\le  \|h(x, 0) - f\|_{L^2(D)} (M+M') \left(   C' +   \left( \frac{2}{n} +   Kt\right) C'' m^2 \right).
\end{aligned}
\end{equation*}
\end{proof}
\begin{remark}
    If the bound of $V(\cA)$ can be as large as $\cO(n M)$ at the worst case, then we may return to the Theorem~\ref{THM: SLOW DECAY}. In fact, if $V(\cA)$ becomes $\cO(m)$, we may also have to return to the Theorem~\ref{THM: SLOW DECAY}.
\end{remark}

\subsection{Part II}
In this part, we try to optimize the bound for $\cJ_m(t) - \cJ_m(0)$: 
\begin{equation*}
   \cJ_m(t) - \cJ_m(0)=  \theta_m'(-1) h(-1, t) = \theta'_m(-1) \sum_{i=1}^{s-1} a_i(t) \chi(b_i(t))(-1 - b_i(t)).
\end{equation*}
The main result is the following estimate.
\begin{theorem}\label{THM: J}
Assume that $\sup_{1\le i=1\le n} |a_i(t)|^2$ is uniformly bounded by $M > 0$ and the biases $\{b_i(0)\}_{i=1}^n$ are initially equispaced on $D$, then there exists a constant $C > 0$ that $$| \cJ_m(t) - \cJ_m(0)| \le C m \sqrt{M} K^2 n t^2.$$
\end{theorem}
\begin{proof}
    The biases are propagating with finite speed that $|\frac{d}{dt} b_i(t)|\le K\coloneqq C'''\sqrt{M}\|h(\cdot, 0) - f(\cdot)\|_{L^2(D)}$ (see~\eqref{EQ: FINITE SPEED}).  If the initial biases are \emph{equispaced} distributed on $D$, then $s(t) - 1\le \frac{Kn t}{2}$. For each $1\le i\le s(t) - 1$, the bias $b_i(t)$ satisfies 
\begin{equation*}
    |b_i(t) - b_i(0)|\le Kt,\quad\text{and}\quad  |b_i(0) - (-1)|\le Kt.
\end{equation*}
Therefore 
\begin{equation*}
    \begin{aligned}
    \sum_{i=1}^{s(t)-1}\chi(b_i(t))|-1 - b_i(t)| &\le  \sum_{i=1}^{s(t)-1} |-1 - (b_i(0) - Kt)| \\
    &\le \sum_{i=1}^{s(t)-1}|-1 - b_i(0)| + Kt(s(t) - 1) \\
    &\le K^2 n t^2. 
    \end{aligned}
\end{equation*}
Therefore $  |\cJ_m - \cJ_m(0)|\le |\theta_m'(-1)| \sqrt{M} K^2 n t^2\le C m \sqrt{M} K^2 n t^2$.
\end{proof}

\subsection{Part III}
Now we can prove the following theorem with the improved bounds. In the following, we assume $n\ge m$. In other words, it only makes sense to study the learning dynamics for those frequency modes which can be resolved by the grid resolution corresponding to the network width.  
\begin{theorem}\label{THM: SLOW DECAY 2}
        Suppose $\sup_{1\le i\le n}|a_i(t)|^2$ is uniformly bounded by $M$, the total variation of $\cA$~\eqref{eq:TV} is bounded by $M'$, and the initial biases are equally spaced. Let $n \ge m^4$ be sufficiently large, then it takes at least $\cO(\frac{m^4 |\widehat{w}(m, 0)|}{n})$ to reduce the initial error in {generalized} Fourier mode $|\widehat{w}(m, t)| \le \frac{1}{2} |\widehat{w}(m, 0)|$. Especially when $|\widehat{w}(m, 0)|>c'm^{-2}$, the half-reduction time is at least $\cO(\frac{m^2 c'}{n})$.
\end{theorem}
\begin{proof}
    Using Lemma~\ref{LEM: BOUND 2}, there exists a constant $C'' > 0$ that 
 \begin{equation*}
        \left|\int_{D} \partial_b w(b', t) \mu_2(b', t) \theta_m'(b') db' \right| \le  C''\left(   1  +   \left( \frac{1}{n} + t\right) m^2 \right).
    \end{equation*}
    We only consider the case that $n\ge m$, otherwise we return to Theorem~\ref{THM: SLOW DECAY}. 
    Recall that in Theorem~\ref{THM: SLOW DECAY} that $|\widehat{w\mu_0}(m, t)|$ is uniformly bounded (see~\eqref{EQ: W MU0}) and $|\cH_m(t) - \cH_m(0)| \le \frac{1}{2}|\theta_m(-1)|\sqrt{M} K n t$ (see~\eqref{EQ: HJ1} and~\eqref{EQ: HJ2}). Combine the estimate in Theorem~\ref{THM: J} and follow the same process as~\eqref{EQ: INEQ}, we can find a constant $C''' > 0$ that 
\begin{equation}\label{EQ: W DECAY 2}
\begin{aligned}
    |\widehat{w}(m, t)| - |\widehat{w}(m, 0)| 
    \le -C'''\frac{n}{p_m}  \left(\left( 1 + \frac{m^2}{n}\right)t + (m+m^2) t^2\right).
\end{aligned}
\end{equation}
    We solve an upper bound for the half-reduction time $\tau$ from the quadratic equation 
    \begin{equation}\label{EQ: QUAD}
        \frac{1}{2}|\widehat{w}(m, 0)| =  \frac{n}{p_m}C'''\left(\left( 1 + \frac{m^2}{n}\right)\tau + (m+m^2) \tau^2\right) .
    \end{equation}
    The solution satisfies
    \begin{equation}
        \tau \ge \frac{p_m}{2C'''n}  \frac{|\widehat{w}(m, 0)|}{\sqrt{(1+m^2/n)^2 + 2(m+m^2) p_m |\widehat{w}(m, 0)|/(C''' n)}}.
    \end{equation}
    Let's use the notations $A \prec B$ and $A\succ B$ to represent the relation 
    $A = \cO(B)$ and $B = \cO(A)$, respective.
    We find the following regimes: 
    \begin{enumerate}
        \item $n \succ m^2$ and $|\widehat{w}(m, 0)| \succ \frac{n}{m^6}$, then $\tau = \cO\left(m\sqrt{\frac{\widehat{w}(m, 0)}{n}}\right)$.
        \item $n \succ m^2$ and $|\widehat{w}(m, 0)| \prec \frac{n}{m^6}$, then $\tau = \cO\left(\frac{m^4}{n}|\widehat{w}(m, 0)|\right)$.
        \item $m \prec n \prec m^2$ and $|\widehat{w}(m, 0)| \succ \frac{1}{nm^2}$,  then $\tau = \cO\left(m\sqrt{\frac{\widehat{w}(m, 0)}{n}}\right)$.
        \item $m \prec n \prec m^2$ and $|\widehat{w}(m, 0)| \prec \frac{1}{nm^2}$, then $\tau = \cO( m^2 |\widehat{w}(m, 0)|)$.
    \end{enumerate}
    In the second case, when $n \succ m^4$ is sufficiently large and $|\widehat{w}(m, 0)| > c'm^{-2}$, the half-reduction time becomes $\cO(\frac{c'm^2}{n})$, which is a better bound than the one in Theorem~\ref{THM: SLOW DECAY}.
\end{proof}

\begin{remark}\label{rem: fix bias}
If the biases $b_i$ are equally spaced and fixed, the learning dynamics become the gradient flow for the least square problem. Using a standard Fourier basis one gets a simpler version of ~\eqref{EQ: NEW PDE} without the boundary terms and the last term involving $\theta_m'$. 
{{Since $w(b,t)$ is $H^2$, one gets $\widehat{\mu_0w}(m,t)\le C ( \frac{1}{m^2} + m \text{disc}(\{b_i(0)\}_{i=1}^n) ) $, for $n\succ m^3$ and equispaced $\{b_i\}_{i=1}^n$ that $\text{disc}(\{b_i(0)\}_{i=1}^n) = \cO(\frac{1}{n})$, it takes $\cO(m^4)$ steps to reduce the initial error in mode $m$ by half. }}
Hence the full learning dynamics, i.e., involving the bias, while requiring more computation cost in each step, may speed up the convergence. See Fig~\ref{fig: fix-bias} for an example.
\end{remark}

\subsection{Numerical experiments}
In the following, we perform numerical experiments to demonstrate the scaling laws with a different total variation of $|a_i(t)|^2$. The objective function is 
$$f(x) = \sin(k \pi x)$$
with $k$ chosen from selected high frequencies. We set the number of neurons $n = k^{\beta}$, $\beta\in \{2, 3, 4\}$ for a selected frequency $k$. The learning rate is selected as $n^{-1}$. This set up is regime 1 above, where $\hat{w}(k,0)=k^{-2}$, and hence the number of iterations should be of order $\cO(n\tau)=\cO(\sqrt{n})$.

The initialization of biases $\{b_i(0)\}_{i=1}^n$ are equispaced and ordered ascendingly. The weights $\{a_i(0)\}_{i=1}^n$ are initialized in the following ways.
\begin{enumerate}
    \item[(A)] $a_i(0) = \frac{1}{2} (-1)^i$. 
    \item[(B)] $a_i(0) = \frac{1}{2}\cos(i)$.  
\end{enumerate}
We record the dynamics of the network (denoted by $h_k$) at exactly frequency $k$ through the projection
\begin{equation*}
    E_k(t) = \left|\int_{D} (h_k(x, t) - f(x) ) f(x) dx \right|.
\end{equation*}
As we will see in the experiments, the total variation in $\cA$ is slowly varying in time, so we can summarize the \emph{a priori} theoretical lower bounds and the experimental results of the number of epochs in the following tables~\ref{tab:my_label 1} and~\ref{tab:my_label 2}. The initialization (A) has a relatively small total variation, and the experiments agree with the theoretical bounds quite well. However, similar results are observed for initialization (B), which has a relatively large total variation during the training, see Fig~\ref{fig: case a-4},~\ref{fig: case a-3},~\ref{fig: case a-2} for initialization (A) and Fig~\ref{fig: case b-4},~\ref{fig: case b-3},~\ref{fig: case b-2} for initialization (B). One possible explanation is the overestimate using the total variation in~\eqref{EQ: FOUR EST} which might be unnecessary. As we mentioned in Remark~\ref{rem: fix bias}, we record the training dynamics with \emph{fixed} biases, initialization (A), and choose $n = k^4$, see Fig~\ref{fig: fix-bias}. The result matches the argument in Remark~\ref{rem: fix bias}. For comparison purposes, we additionally demonstrate an example with \texttt{Adam} optimizer, which shows a similar scaling relation as the \texttt{GD} optimizer at the initial training stage, see Fig~\ref{fig: fix-bias}. 

All these tests show a consistent phenomenon, the higher the frequency, the slower the learning dynamics.


\begin{table*}
    \centering
    \caption{Theoretical lower bounds.}
    \begin{tabular}{ccccc}
    \toprule
        & $TV(\cA)$ & $\beta=2$ & $\beta=3$ & $\beta=4$  \\
         \midrule
       Init A & $\cO(1)$ & $\cO(k)$  & $\cO(k^{1.5})$  &   $\cO(k^2$) \\
         \midrule
       Init B & $\cO(n)$ & $\cO(k)$  & $\cO(k)$  &  $\cO(k)$  \\ 
       \bottomrule 
    \end{tabular}
    \label{tab:my_label 1}
\end{table*}
\begin{table*}
    \centering
    \caption{Experimental fitted results.}
    \begin{tabular}{ccccc}
    \toprule
        & $TV(\cA)$ & $\beta=2$ & $\beta=3$ & $\beta=4$  \\
            \midrule 
       Init A & $\cO(1)$ & $\cO(k^{1.61})$  & $\cO(k^{1.82})$  &   $\cO(k^{1.87}$) \\
         \midrule
       Init B & $\cO(n)$ &  $\cO(k^{1.59})$ &  $\cO(k^{1.75})$ &   $\cO(k^{1.90}) $ \\ 
       \bottomrule
    \end{tabular}
    \label{tab:my_label 2}
\end{table*}

\begin{figure}
    \centering
    \includegraphics[height=0.23\textwidth]{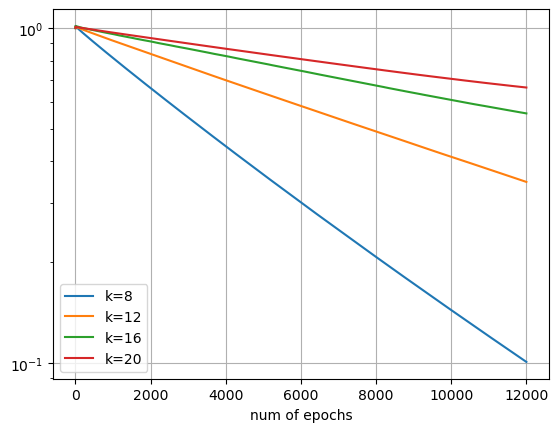}
    \hfill
    \includegraphics[height=0.23\textwidth]{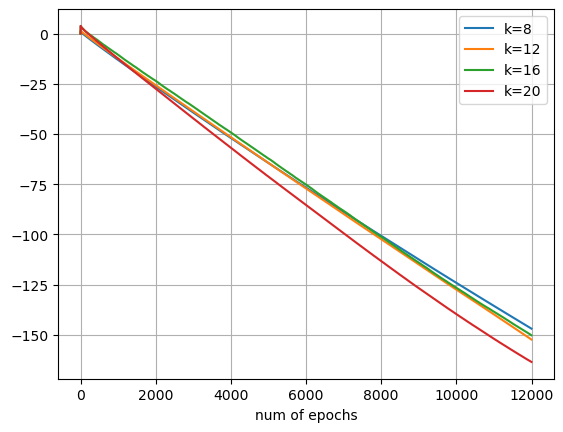}
    \hfill
    \includegraphics[height=0.23\textwidth]{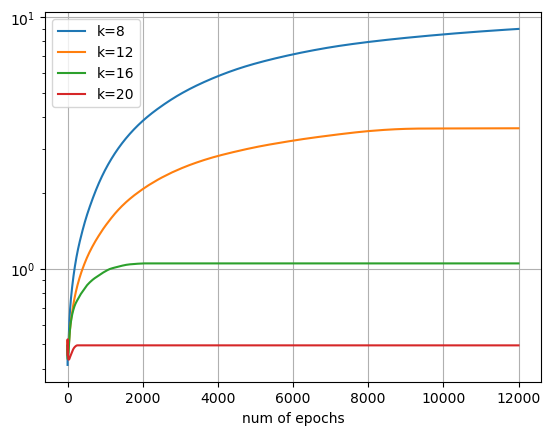}
    \caption{Experiment for initialization (A) and $\beta=4$. Left: the graphs of $E_k(t)$. Middle: the graphs of $k^{1.87}\ln(E_k(t))$. Right: the total variations of $\cA$.}
    \label{fig: case a-4}
\end{figure}

\begin{figure}
    \centering
    \includegraphics[height=0.23\textwidth]{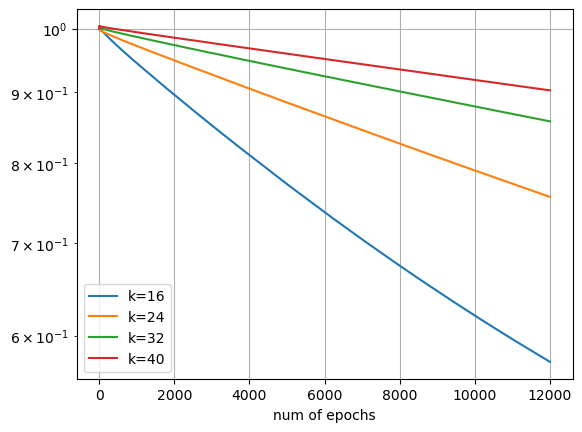}\hfill
    \includegraphics[height=0.23\textwidth]{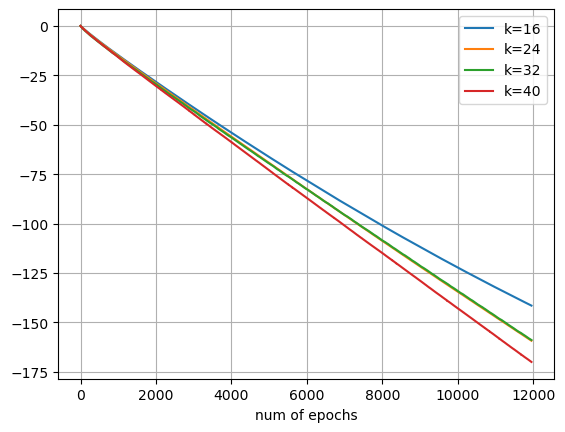}\hfill
    \includegraphics[height=0.23\textwidth]{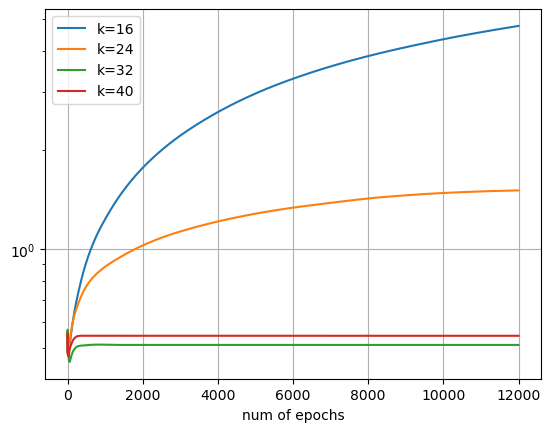}
    \caption{Experiment for initialization (A) and $\beta=3$. Left: the graphs of $E_k(t)$. Middle: the graphs of $k^{1.82}\ln(E_k(t))$. Right: the total variations of $\cA$.}
    \label{fig: case a-3}
\end{figure}
\begin{figure}
    \centering
    \includegraphics[height=0.23\textwidth]{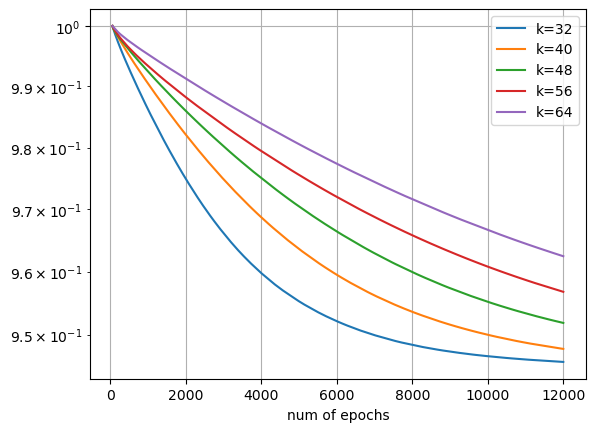}\hfill
    \includegraphics[height=0.23\textwidth]{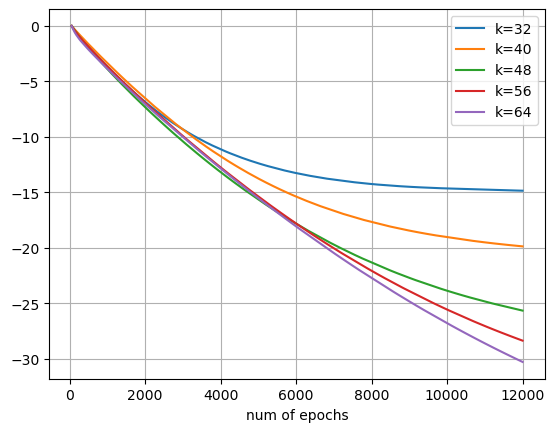}\hfill
    \includegraphics[height=0.23\textwidth]{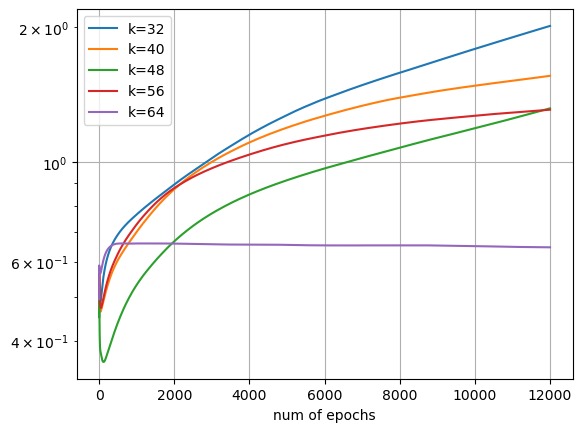}
    \caption{Experiment for initialization (A) and $\beta=2$. Left: the graphs of $E_k(t)$. Middle: the graphs of $k^{1.61}\ln(E_k(t))$ (fitting first 2000 epochs only). Right: the total variations of $\cA$.}
    \label{fig: case a-2}
\end{figure}



\begin{figure}
    \centering
    \includegraphics[height=0.23\textwidth]{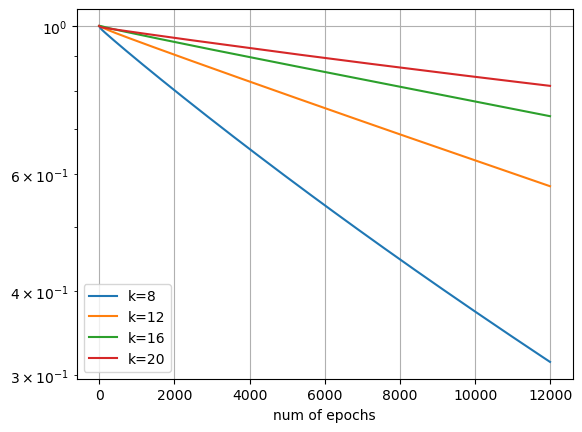}\hfill
    \includegraphics[height=0.23\textwidth]{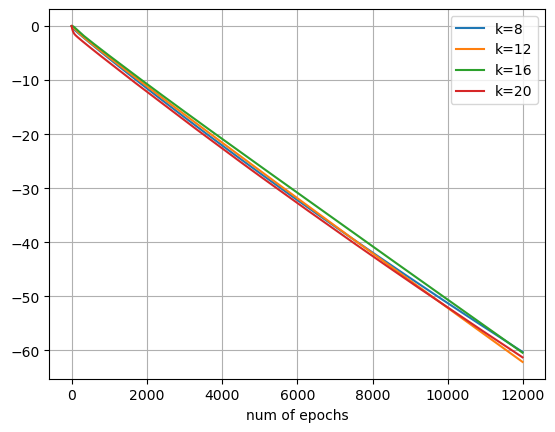}\hfill
    \includegraphics[height=0.23\textwidth]{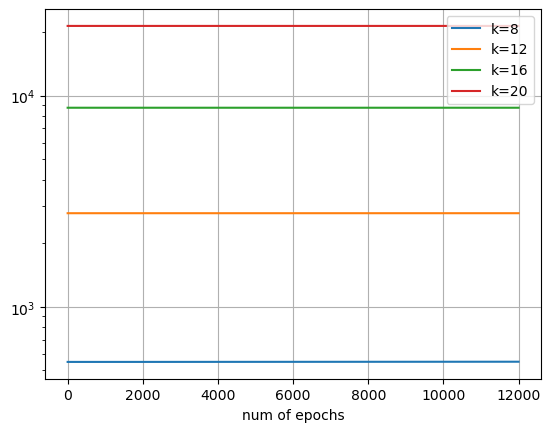}
    \caption{Experiment for initialization (B) and $\beta=4$. Left: the graphs of $E_k(t)$. Middle: the graphs of $k^{1.90}\ln(E_k(t))$ (fitting first 2000 epochs only). Right: the total variations of $\cA$.}
    \label{fig: case b-4}
\end{figure}

\begin{figure}
    \centering
    \includegraphics[height=0.23\textwidth]{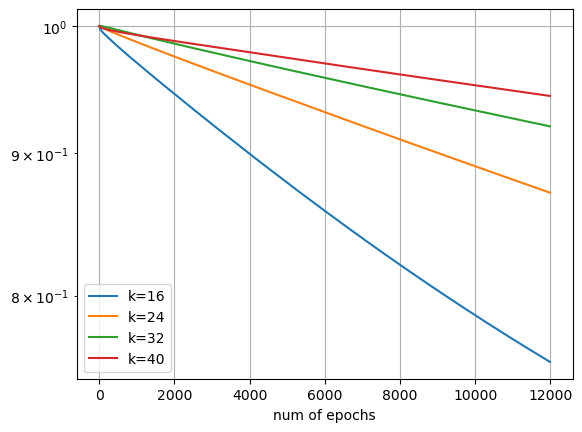}\hfill
    \includegraphics[height=0.23\textwidth]{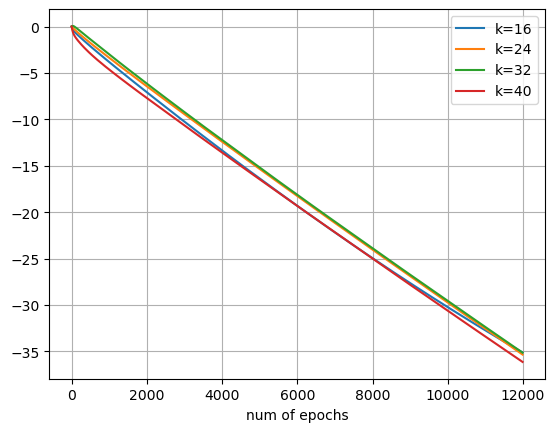}\hfill
    \includegraphics[height=0.23\textwidth]{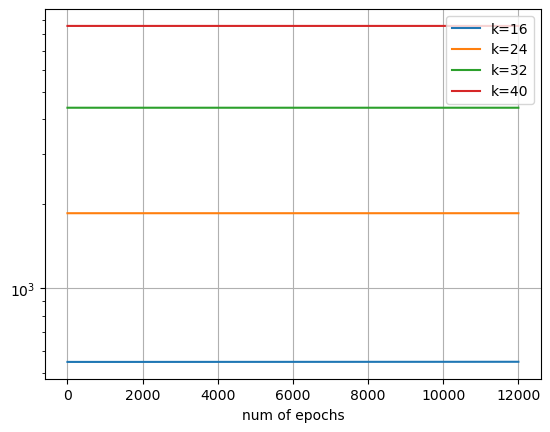}
    \caption{Experiment for initialization (B) and $\beta=3$. Left: the graphs of $E_k(t)$. Middle: the graphs of $k^{1.75}\ln(E_k(t))$. Right: the total variations of $\cA$.}
    \label{fig: case b-3}
\end{figure}

\begin{figure}
    \centering
    \includegraphics[height=0.23\textwidth]{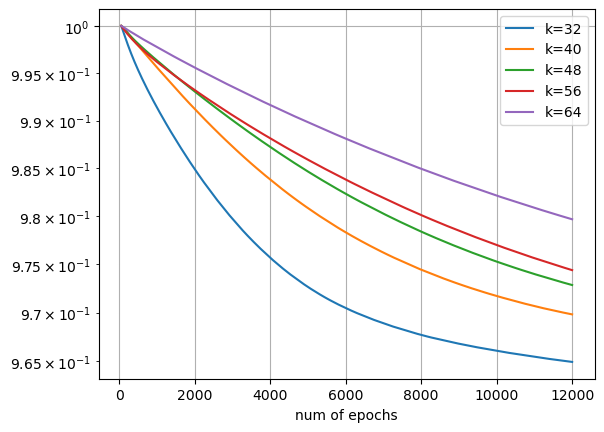}\hfill
    \includegraphics[height=0.23\textwidth]{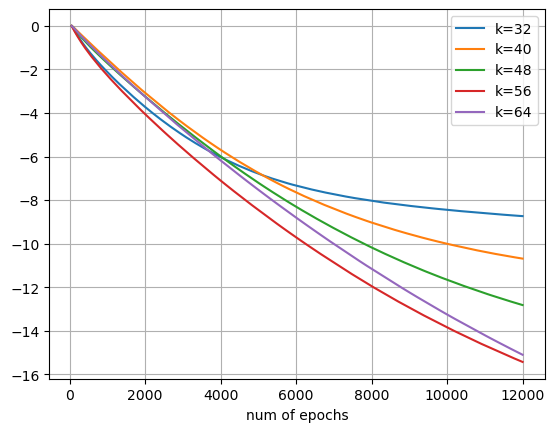}\hfill
    \includegraphics[height=0.23\textwidth]{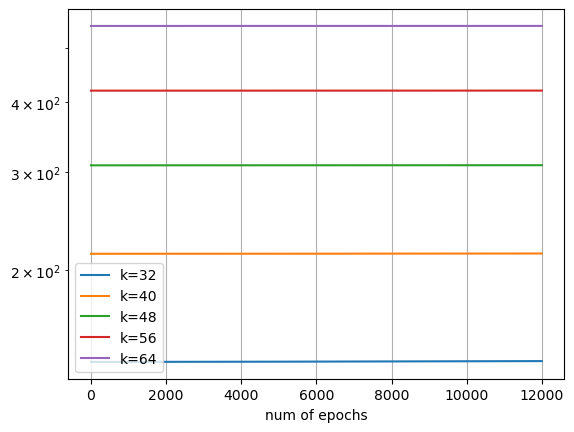}
    \caption{Experiment for initialization (B) and $\beta=2$. Left: the graphs of $E_k(t)$. Middle: the graphs of $k^{1.59}\ln(E_k(t))$ (fitting first 2000 epochs only). Right: the total variations of $\cA$.}
    \label{fig: case b-2}
\end{figure}

\begin{figure}
    \centering
\includegraphics[height=0.23\textwidth]{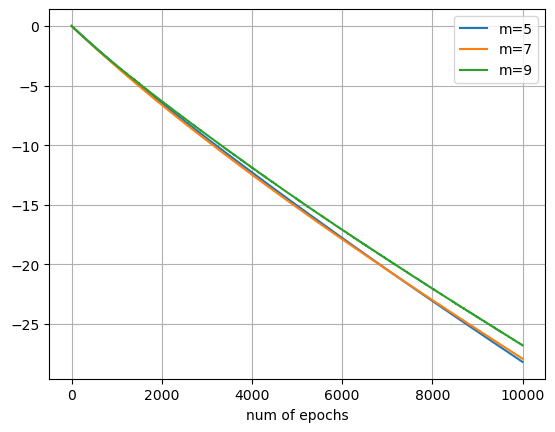}\hfill
\includegraphics[height=0.23\textwidth]{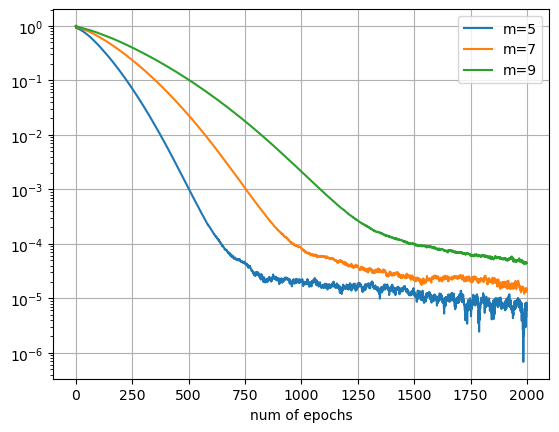}\hfill
\includegraphics[height=0.23\textwidth]{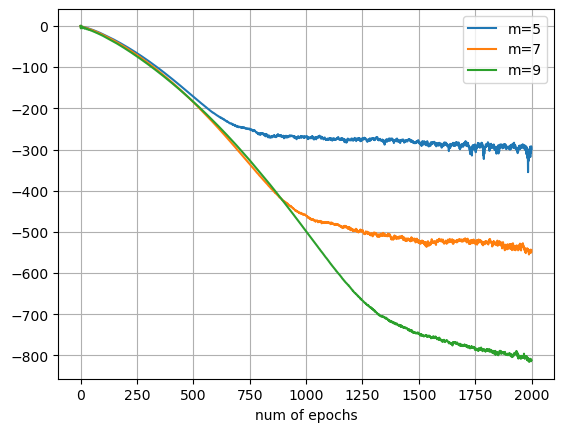}
\caption{Additional experiment for initialization (A) and $\beta=4$. Left: Graphs of $k^{3.5}\ln(E_k(t))$ for $k=5, 7, 9$ with \emph{fixed} equispaced biases, trained by \texttt{GD}. Middle: Graphs of $E_k(t)$ for $k=5,7,9$, trained by \texttt{Adam}. Right: Graphs of $k^{2}\ln E_k(t)$ for $k=5,7,9$, trained by \texttt{Adam}.}
\label{fig: fix-bias}
\end{figure}

\subsection{Further remarks}
\subsubsection{Initial distribution of biases}
If the initial biases are uniformly distributed instead of equispaced, the previous estimates need to be modified. In particular, the upper bound of $s(t) - 1$ will become $\frac{Knt}{2} + \cO_p(\sqrt{Knt})$ using the Chebyshev inequality. The discrepancy of $\{b_i(0)\}_{i=1}^n$ will be also updated to $\cO_p(n^{-1/2})$. Then we have the following modified estimates:
\begin{equation*}
\begin{aligned}
    |\cH_m(t) - \cH_m(0)| &\le \frac{1}{2}|\theta_m(-1)|\sqrt{M} (nKt + \cO_p(\sqrt{nKt})),\\
    |\cJ_m(t) - \cJ_m(0)| &\le |\theta'_m(-1)| \sqrt{M} K t (Knt+\cO_p(\sqrt{Knt})),\\
    \left|\int_{D} \partial_b w(b', t) \mu_2(b', t) \theta_m'(b') db' \right| &\le C'' \left(1 +\left(t + \cO_p\left(\frac{1}{\sqrt{n}}\right) \right)m^2\right),
\end{aligned}
\end{equation*}
and~\eqref{EQ: W DECAY 2} becomes 
\begin{equation*}
    |\widehat{w}(m, t)| - |\widehat{w}(m, 0)| \;\mathop{\le}_{p}\; -C''' \frac{n}{p_m}  \left(\left( 1 + \frac{m^2}{\sqrt{n}}\right)t + (m+m^2) t^2 + (1+mt)\sqrt{\frac{t}{n}}\right).  
\end{equation*}
In particular, when $n\succ m^4$ and $|\widehat{w}(m, 0)| > c'm^{-2}$, the half-reduction time is still the same $\cO_p(m^2 c'/n)$ as Theorem~\ref{THM: SLOW DECAY 2}. A similar probabilistic estimate can be derived for initial biases sampled from a continuous probability density function which is bounded from below and above by positive constants.

\subsubsection{Activation function}
    The regularity of the activation function plays a crucial role in the analysis. In general, using a smoother activation function, which leads to a faster spectrum decay of the corresponding Gram matrix, will take an even longer time to eliminate higher frequencies. For instance, if the activation function is chosen as $\frac{1}{p!}\sigma^p(x)$, $p\ge 1$, where $\sigma$ is the \texttt{ReLU} activation function, then a similar analysis will show that if $n\succ m^{p+3}$, then under the assumption of Theorem~\ref{THM: SLOW DECAY 2}, the \emph{half-reduction} time of ``frequency'' $m$ is at least $\cO(\frac{m^{2p+2}}{n} |\widehat{w}(m, 0)|)$ in the Theorem~\ref{THM: SLOW DECAY 2}. This also implies that using a shallow neuron network with smooth activation functions will be even worse for learning and approximating high-frequency information by minimizing the $L^2$ error or mean-squared error (\texttt{MSE}). 

    In the following, we perform a simple numerical experiment to validate our conclusion. We set the objective function as the Fourier mode $f(x) = \sin(m\pi x)$ on $D$, $m\in\bbN$ and use the activation functions 
    $\texttt{ReLU}^p(x)\coloneqq \frac{1}{p!}\sigma^p(x)$, $p=1,2$ to train the shallow neuron network $h_{m,p}(x, t)$ to approximate $f(x)$, respectively.  The number of neurons $n = 10^4$. We record the the \emph{error} of the Fourier mode $$E_{m,p}(t)\coloneqq \left|\int_{D} (h(x, t) - f(x)) f(x) dx \right|$$ 
    at each iteration. The errors are shown in Figure~\ref{FIG: RATIO} for $m=5, 7, 9$. We can observe that the decay rates of $\texttt{ReLU}$ and $\texttt{ReLU}^2$ are $\cO(m^{-2})$ and $\cO(m^{-3})$ roughly.
\begin{figure}
\centering
\includegraphics[height=0.18\textwidth]
{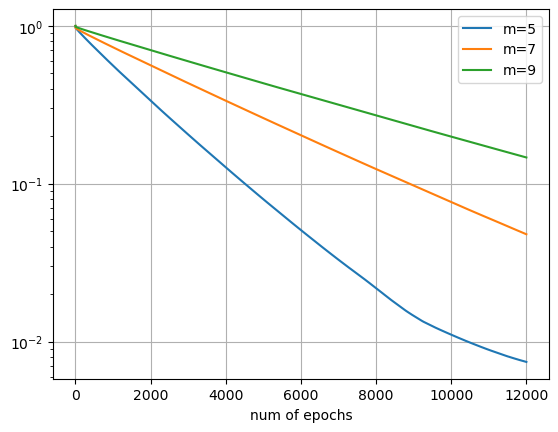}\hfill
\includegraphics[height=0.18\textwidth]{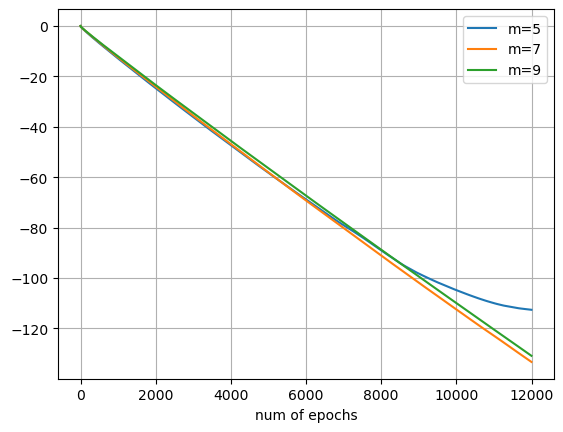}\hfill
\includegraphics[height=0.18\textwidth]{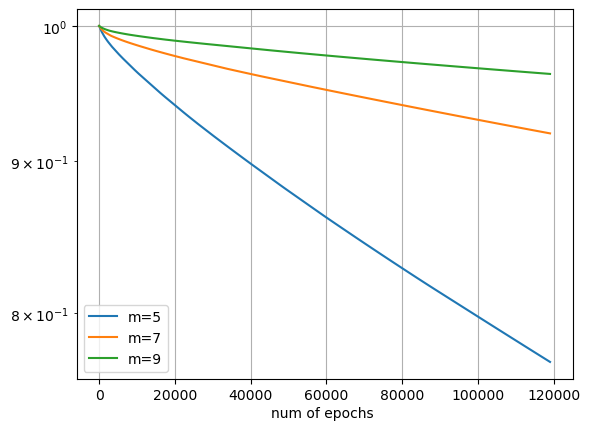}\hfill
\includegraphics[height=0.18\textwidth]{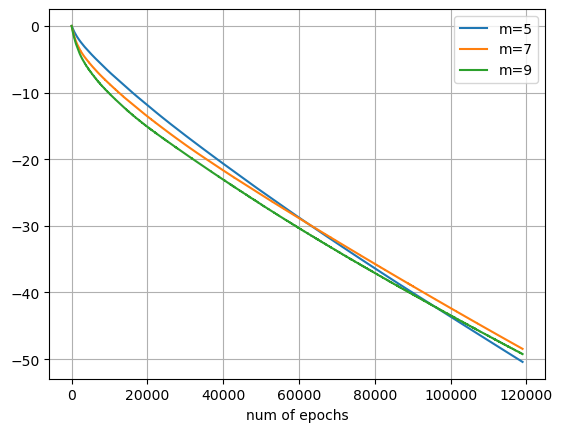}
\caption{The comparison of $\texttt{ReLU}$, $\texttt{ReLU}^2$ for the approximation to $f(x) = \sin(m\pi x)$. From left to right, the figures are $E_{m,1}(t)$, $m^{1.95}\ln(E_{m,1}(t))$, $E_{m,2}(t)$, $m^{3.25}\ln(E_{m,2}(t))$.}
\label{FIG: RATIO}
\end{figure}
\subsubsection{Boundedness of weights}
One may notice that the requirement of weights $\sup_{i\ge 1} |a_i(t)|^2 \le M$ for all $t > 0$ can be relaxed to $0\le t \le \tau$, where $\tau$ denotes the lower bound of half-reduction time in Theorem~\ref{THM: SLOW DECAY} or Theorem~\ref{THM: SLOW DECAY 2}. When $n \succ m$ in Theorem~\ref{THM: SLOW DECAY} or $n \succ m^2$ in Theorem~\ref{THM: SLOW DECAY 2}, such requirement can be relaxed to only the initial condition $\sup_{i\ge 1} |a_i(0)|^2 \le M$.

\section{Rashomon set for bounded activation function}
\label{SEC: RASH BOUND}
 In this section, we characterize the Rashomon set for a general bounded activation function instead of \texttt{ReLU}. We consider a more general setting of the parameter space: $a_i$ are mean-zero i.i.d sub-Gaussian random variables that 
\begin{equation*}
   \bbP[ | a_i | > t] < 2e^{-m t^2}, \quad i\in [n]
\end{equation*} 
for some $m > 0$. Then we have the following estimate for the Rashomon set by following a similar idea for the proof of Theorem~\ref{THM: RASH}.
 \begin{theorem}
     Assuming the same network structure as Theorem~\ref{THM: RASH} and $\sigma$ as a bounded activation function that 
     \begin{equation}\label{EQ: ACTIV}
        -1\le \sigma(t) \le 1, \quad \sigma'(t) > 0, \quad\forall t\in\bbR, 
     \end{equation}
     then the Rashomon set's measure
    \begin{equation*}
     \bbP[\|h(\bmx) - f(\bmx)\|_{L^2(D)} \le \eps \|f\|_{L^2(D)}] \le 2\exp\left(-\frac{C n(1-\eps)^2\|f\|^4_{L^2(D)}}{4\kappa^2 \ell^2 \theta^2}\right),
\end{equation*} 
where $\theta = \|a_i\|_{\psi_2}$ is the \emph{Orlicz} norm, $\ell = \tn{diam}(D)$, and $\kappa$ denotes 
\begin{equation*}
    \kappa \coloneqq \sup_{(\bmw,t)\in \bbS^{d-1}\times\bbR} \left|\int_{\{ \bmx\cdot \bmw = t, \bmx\in D \} } f(\bmx) d H_{d-1}\bmx \right|. 
\end{equation*}
 \end{theorem}
 \begin{proof}
     We denote $r_i \coloneqq |\bmw_i|$ and $\ell = \text{diam}(D)$, then let 
     \begin{equation*}
         X_i\coloneqq a_i\int_{D} f(\bmx) \sigma(\bmw_i\cdot \bmx - b_i) d\bmx = \frac{a_i}{r_i} \int_{-\ell r_i}^{\ell r_i} \sigma(s - b_i) \int_{\{\bmx\cdot \bmw_i = s\}} f(\bmx) dH_{d-1}(\bmx) ds.
     \end{equation*}
     Then $\aver{h,f}=\frac{1}{n}\sum_{i=1}^n X_i$ and 
     \begin{equation*}
         \bbP\left[ \|h - f\|^2_{L^2(D)} \le\eps^2 \|f\|_{L^2(D)}^2  \right] \le \bbP\left[ (1-\eps)\|f\|_{L^2(D)}^2 \le \frac{1}{n}\sum_{i=1}^n X_i\right].
      \end{equation*}
      Since the random variable $a_i$ is sub-Gaussian, then  $X_i$ is also sub-Gaussian by $|X_i|\le 2 a_i \ell \kappa$. One can apply the Hoeffding's inequality (see Theorem 2.6.2~\cite{vershynin2018high}) that 
      \begin{equation*}
          \bbP\left[ (1 - \eps)\|f\|_{L^2(D)}^2 \le \frac{1}{n}\sum_{i=1}^n X_i\right] \le 2\exp\left(\frac{-C n (1-\eps)^2\|f\|_{L^2(D)}^4}{4\kappa^2 \ell^2 \theta^2}\right).
      \end{equation*}
      for certain absolute constant $C$.
 \end{proof}
 The constant $\kappa$ stands for the largest possible average of $f$ on every hyperplane $\{\bmx\cdot \bmw = t\}$, $t\in\bbR$. When $f(\bmx)$ is oscillatory in all directions, the constant $\kappa$ becomes small. More intuitively speaking, an activation function of the form~\eqref{EQ: ACTIV} can not feel oscillations in $f$, i.e., $\aver{f, \sigma}$ is small due to cancellation. According to the heuristic argument in~\cite{semenova2022existence}, it also implies that it is relatively difficult to find a shallow neural network with activation function~\eqref{EQ: ACTIV} that can approximate highly oscillatory functions well. In other words, the optimal set of parameters only occupies an extremely small measure of the parameter space for highly oscillatory functions.

\section{Further discussions}\label{sec:app-discussion}
\subsection{General case of Gram matrix of two-layer $\texttt{ReLU}$ networks in one dimension}
When the two-layer \texttt{ReLU} network in one dimension is 
\begin{equation*}
    f(x) = c + \sum_{i=1}^n a_i \sigma(w_i x - b_i),\quad x\in D \coloneqq [-1,1],
\end{equation*}
where $w_i\in \{+1, -1\}$ obeys the Bernoulli distribution with $p=\frac{1}{2}$. Then the corresponding Gram matrix has the following block structure of continuous kernels 
 \begin{equation*}
 \begin{aligned}
   \calG  \coloneqq  \begin{pmatrix}
      \calG^{++}(x, y) &  \calG^{+-}(x, y)  \\  \calG^{+-}(x, y)  & \calG^{--}(x, y) 
     \end{pmatrix}
 \end{aligned}
 \end{equation*}
 where the sub-kernels are
 \begin{equation*}
 \begin{aligned}
     \calG^{++} (x, y) &= \frac{1}{24}(2 - x - y - |x - y|)^2 (2 - x - y + 2 |x-y|),\\
     \calG^{--}(x, y) &= \frac{1}{24}(2 + x+y-|x-y|)^2 (2 + x + y + 2 |x-y|),\\
     \calG^{+-}(x, y) &= \frac{1}{48} \left[|x - y| - (x-y)\right]^3,\\
     \calG^{-+}(x, y) &= \frac{1}{48} \left[|x - y| + (x-y)\right]^3. 
 \end{aligned}
 \end{equation*}
Note the kernels $\calG^{+-}(x, y) = \calG^{-+}(-x, -y)$ and $\calG^{++}(x, y) = \calG^{--}(-x, -y)$, suppose $(g_k^{+}, g_k^{-})$ is an eigenfunction of $\calG$ for eigenvalue $\lambda_k$, we can obtain: 
\begin{itemize}
    \item If $  \phi_k(x) = g_k^{+}(x)+ g_k^{-}(-x) \not\equiv 0$, then it is an eigenfunction of $\calG_{\phi} = \calG^{++}(x, y) + \calG^{+-}(x,-y) $ for eigenvalue $\lambda_k$. 
\item If $  \psi_k(x) = g_k^{+}(x) - g_k^{-}(-x)\not\equiv 0$, then it is an eigenfunction of $\calG_{\psi} = \calG^{++}(x, y) - \calG^{+-}(x,-y) $ for eigenvalue $\lambda_k$.
\end{itemize}

 The kernel $\calG_{\phi}\in C^2(D\times D)$ is 
 \begin{equation*}
     \calG_{\phi}(x, y) = \frac{1}{12}(|x-y|^3 + |x+y|^3 + 4 + 12 xy - 6(x+y) - 6xy(x+y)).
 \end{equation*}
  Based on the above observation, it is straightforward to derive the following theorem.
 \begin{theorem}
 If the two kernels $\calG_{\phi}$ and $\calG_{\psi}$ do not allow common eigenvalues, then $(g_k^{+}, g_k^{-})$ is an eigenfuntion of $\calG$, then they satisfy either $g_k^{+}(x) = g_k^-(x)$ or $g_k^{+}(x) = g_k^{-}(-x)$.
 \end{theorem}
 \begin{remark}
     The kernel $K_{\alpha} = |x-y|^{\alpha}$ has been studied in~\cite{bogoya2012eigenvalues} for $\alpha=1$, where the eigenvalues have a leading positive term and all of the rest eigenvalues are negative and decay as $\frac{c}{(2k+1)^2}$. Consider the eigenvalue for $|x - y|^3$, we need to find 
     \begin{equation*}
         \lambda h(x) = \int_{-1}^x (x - y)^3 h(y) dy + \int_{x}^1 (y- x)^{3} h(y) d y
     \end{equation*}
     by differentiating the above equation 4 times, we get $f^{(4)} = \frac{ 12 }{\lambda }h(x)$, let $\omega^4 = \frac{12}{\lambda}$, then the solution consists of the basis 
     \begin{equation*}
         \sum_{k=0}^3 A_k \exp(\omega e^{\frac{2\pi i  k}{4}} x ).
     \end{equation*} 
   Thus solving the eigensystem is equivalent to solving a $4\times 4$ matrix $\det(M) = 0$ for $\omega$. Similar arguments hold for the Hankel kernel $|x+y|^3$, they share the same basis. The exact values are quite expensive to compute. 
 \end{remark}
   Now we apply the same idea for $\calG_{\phi}$, by the same differentiation technique used in deriving~\eqref{eq:ODE}, we arrive at the same form:
   \begin{equation*}
       \phi_k(x) =     
         \sum_{l=0}^3 c_l \exp(\omega e^{\frac{2\pi i  l}{4}} x ),
    \end{equation*}   
    where $\omega^4 = \frac{2}{\lambda}$, here $\lambda > 0$, thus we choose $\omega\in\bbR^{+}$ and the basis are more explicit:
    \begin{equation*}
               \phi_k(x) = c_0 \cosh(\omega x) + c_1 \sinh(\omega x) + c_2 \cos(\omega x) + c_3\sin(\omega x).
    \end{equation*}
    The eigenvalues $\lambda_k$ can be computed in a similar way as in Theorem~\ref{lem:B4} and there are constants $c_1, c_2 > 0$ that $c_1 k^{-4} \le \lambda_k \le c_2  k^{-4}$.

\subsection{Leaky $\texttt{ReLU}$ activation function}
For the leaky $\texttt{ReLU}$ activation function with parameter $\alpha\in (0, 1)$, $ \sigma_{\alpha}(x) = \sigma(x) - \alpha\sigma(-x)$, we can derive the Gram matrix $G_{\alpha}$:
\begin{equation*}
\begin{aligned}
    G_{\alpha, ij} &= \int_{-1}^1 \sigma_{\alpha}(x - b_i) \sigma_{\alpha}(x - b_i)  d x \\&= \int_{-1}^1(\sigma(x-b_i) - \alpha\sigma(-x + b_i) ) (\sigma(x-b_j) - \alpha\sigma(-x + b_j) ) dx  \\
    &=\calG(b_i, b_j) + \alpha^2 \calG(-b_i, -b_j) -\frac{\alpha}{6}|b_i - b_j|^3 .
\end{aligned}
\end{equation*}
Then we can derive the following estimate for the eigenvalue for $G_{\alpha}$. Let the kernel $\calG_{\alpha}(x, y) \coloneqq \calG(x, y) + \alpha^2 \calG(-x, -y) -\frac{\alpha}{6}|x - y|^3 $.
\begin{theorem}
    Suppose $b_i$ are quasi-evenly spaced on $[-1, 1]$, $b_i = -1 + \frac{2(i -1)}{n}+ o\left(\frac{1}{n}\right)$. Let $\lambda_1 \ge \lambda_2 \ge \cdots \ge \lambda_n \ge 0$ be the eigenvalues of the Gram matrix $G_{\alpha}$ then $|\lambda_k - \frac{n}{2}\mu_{\alpha, k} |\le C$ for some constant $C = \calO(1)$, where $\mu_{\alpha, k} = \calO(\frac{(\alpha-1)^2}{k^4})$ is the $k$-th eigenvalue of $\calG_{\alpha}$.
\end{theorem}
\begin{proof}
Let the kernel $\calG_{\alpha}(x, y) \coloneqq \calG(x, y) + \alpha^2 \calG(-x, -y) -\frac{\alpha}{6}|x - y|^3 $, then using the same differentiation technique in deriving~\eqref{eq:ODE}, if $\psi_{\alpha, k}$ is an eigenfunction of $\calG_{\alpha}$ for the eigenvalue $\mu_{k,\alpha}$, we have
\begin{equation*}
    \psi_{\alpha, k}^{(4)} = \frac{(\alpha-1)^2}{\mu_k} \psi_{\alpha, k}.
\end{equation*}
Let $w_{\alpha, k} = \sqrt{|1-\alpha|} \mu_k^{-\frac{1}{4}}$, then equivalently we obtain the following equation for $w_{\alpha, k}$:
\begin{equation*}
   \begin{aligned}
    & \left(P_{0,\alpha}(w_{\alpha, k}) + P_{1,\alpha}(w_{\alpha, k})\cos(2 w_{\alpha, k}) + P_{2,\alpha}(w_{\alpha, k}) \sin(2w_{\alpha, k})\right)\\&+\tanh(w_{\alpha,k}) \left(Q_{0,\alpha}(w_{\alpha, k})  + Q_{1,\alpha}(w_{\alpha, k})\cos(2 w_{\alpha, k}) + Q_{2,\alpha}(w_{\alpha, k}) \sin(2w_{\alpha, k}) \right) \\
    &+ \tanh^2(w_{\alpha,k}) \left(R_{0,\alpha}( w_{\alpha, k})  + R_{1,\alpha}( w_{\alpha, k})\cos(2 w_{\alpha, k}) + R_{2,\alpha}( w_{\alpha, k}) \sin(2w_{\alpha, k}) \right) = 0,
   \end{aligned}
   \end{equation*}
   where $P_{i,\alpha}, Q_{i,\alpha}, R_{i,\alpha}$, $i=0,1,2$ are polynomials of $w_{\alpha, k}$ of degree $\le 4$.  Set $A_{\alpha}(x) = (- 36 \alpha^4  + 42\alpha^5 - 12\alpha^6) x^2$ and $B_{\alpha}(x) = (8\alpha^4 - 8\alpha^5 + 2\alpha^6) x^4$, then 
   \begin{equation*}
       \begin{aligned}
          P_{0,\alpha} (x) &= \frac{3}{2} + 6\alpha^2 - 6\alpha^3 + \frac{3}{2}\alpha^4 +A_{\alpha}(x)  + B_{\alpha}(x), \\
          P_{1,\alpha}(x) &= P_{0,\alpha}(x) - 3,\\
          P_{2,\alpha}(x) &= -3\alpha^2(\alpha^2-3\alpha+3) x + 2\alpha^2(\alpha-2)(2-9\alpha^2+6\alpha^3)x^3,
       \end{aligned}
   \end{equation*}
    \begin{equation*}
       \begin{aligned}
          Q_{0,\alpha}(x) &= 2\alpha^2 x (-18 + 15\alpha - 42\alpha^2 - 57\alpha^3 + 18\alpha^4 + (12 \alpha^2-14\alpha^3 + 4\alpha^4) x^2), \\
          Q_{1,\alpha}(x) &= 2\alpha^2 x(9 - 6\alpha + 45\alpha^2 - 18\alpha^4 + (8 + 4\alpha + 24\alpha^2 - 28\alpha^3 + 8\alpha^4) x^2),\\
          Q_{2,\alpha} (x) & = -12\alpha^2 x^2 (-4 + 3\alpha - 10\alpha^2 + 13 \alpha^2 + 4\alpha^4),
       \end{aligned}
   \end{equation*}
    \begin{equation*}
       \begin{aligned}
          R_{0,\alpha}(x) &= \frac{1}{2}\left[ 3 - 24\alpha^2 + 12 \alpha^3 + 111\alpha^4 - 144 \alpha^5 + 48\alpha^6 \right] -A_{\alpha}(x) + B_{\alpha}(x), \\
          R_{1,\alpha}(x) &= \frac{1}{2}\left[-3+48\alpha^2 - 36\alpha^2-105\alpha^4 + 144\alpha^5 -48\alpha^6\right] - A_{\alpha}(x) + B_{\alpha}(x),\\
          R_{2,\alpha}(x) &= \alpha^2(27-21\alpha-87\alpha^2+114\alpha^3 - 36\alpha^4)x + \alpha^2(8-4\alpha -12\alpha^2 + 14\alpha^3 - 4\alpha^4) x^2.
       \end{aligned}
   \end{equation*}
   We can rewrite the equation as 
   \begin{equation*}
       Z_{0,\alpha}(w_{\alpha, k}) + Z_{1,\alpha}(w_{\alpha, k})\cos(2 w_{\alpha, k}) + Z_{2, \alpha}(w_{\alpha, k}) \sin(2 w_{\alpha, k}) = 0,
   \end{equation*}
   where $Z_{0,\alpha} = P_{0,\alpha} + \tanh(w_{\alpha, k}) Q_{0, \alpha} + \tanh^2(w_{\alpha, k}) R_{0, \alpha}$, $Z_{1,\alpha} = P_{1,\alpha} + \tanh(w_{\alpha, k}) Q_{1, \alpha} + \tanh^2(w_{\alpha, k}) R_{1, \alpha}$ and $Z_{2,\alpha} = P_{2,\alpha} + \tanh(w_{\alpha, k}) Q_{2, \alpha} + \tanh^2(w_{\alpha, k}) R_{2, \alpha}$. It is not hard to show that as $w_{\alpha, k} > 0$, we have $0< 1-\tanh(w_{\alpha, k}) \le 2e^{-2w_{\alpha, k}}$, and there exists a constant $c > 0$ that $\forall x > c$, 
   \begin{equation*}
       Z_{0,\alpha}(x) + Z_{1,\alpha}(x) > 0 ,\quad \text{ and }\quad  Z_{0,\alpha}(x) - Z_{1,\alpha}(x) < 0   
   \end{equation*}
   by computing the sign of leading power in $x$, 
   which implies that there exist roots on the intervals $[n\pi, (n+\frac{1}{2})\pi]$ and $[(n+\frac{1}{2})\pi, (n+1)\pi]$, respectively for sufficiently large $n$.
\end{proof}
The following corollary can be derived by using the Corollary~\ref{COR: EIGEN 1D}. It shows that the $k$th eigenvalue grows as $\calO((\alpha - 1)^2 k^{-4})$ when $k$ is sufficiently large.
\begin{corollary}
    When $\{b_i\}_{i=1}^n$ are i.i.d uniformly distributed on $[-1,1]$, then with probability $1-p$ that 
    \begin{equation*}
    \begin{aligned}
          |\lambda_k - \frac{n}{2}\mu_{\alpha, k} |= \begin{cases}
             \cO\left(n^{\frac{5}{8}}i^{-3} \sqrt{\log \frac{n}{p}}\right), &  i < n^{\frac{7}{8}},\\
             \cO\left(n^{-2}  \sqrt{\log\frac{n}{p}}\right), & n^{\frac{7}{8}} \le  i \le n,
         \end{cases}       
    \end{aligned}
    \end{equation*}
    for certain constant $C > 0$, where $\mu_{\alpha, k} = \calO(\frac{(\alpha-1)^2}{k^4})$  is the $k$-th eigenvalue of $\calG_{\alpha}$.
\end{corollary}
The leading eigenvalue $\mu_{\alpha, 1}$ can be estimated from above using the Hilbert-Schmidt norm of $\calG_{\alpha}$ and also using the Perron-Frobenious theorem~\cite{frobenius1912matrizen} (or Krein-Rutman theorem for positive compact operators), one can estimate both $\mu_{\alpha, 1}$ from below. 
\begin{corollary}
    For any $\alpha\in (0, 1)$, the leading eigenvalue $\mu_{\alpha, 1} \in [0.941, 2.754]$.
\end{corollary}
\begin{proof}
First, we compute that
\begin{equation*}
\begin{aligned}
  v(x) &\coloneqq  \int_{-1}^1 \calG_{\alpha}(x, y) dy \\&= \frac{(x-1)^2(x^2 + 6x + 17) - 2\alpha(1 + 6x^2 +  x^4) + \alpha^2 (x+1)^2(x^2 - 6x + 17)}{24},
\end{aligned}
\end{equation*}
 then the leading eigenvalue satisfies 
\begin{equation*}
  \mu_{\alpha, 1} \ge \frac{\|v\|_{L^2[-1, 1]}}{\sqrt{2}} \ge 2\sqrt{\frac{(728 - 323\alpha + 450\alpha^2 - 323 \alpha^3 + 728\alpha^4)}{2835}}.
\end{equation*}
Minimize the right-hand side, we find that $\mu_{\alpha, 1} \ge 0.941$, $\forall \alpha\in [0, 1]$, the minimum is achieved around $\alpha=0.351$. For the upper bound, we compute the Hilbert-Schmidt norm
\begin{equation*}
\begin{aligned}
   \sqrt{ \int_{-1}^1 \int_{-1}^1 |\calG_{\alpha}(x, y)|^2 dx dy } 
   &=\sqrt{\beta(\Upsilon\cdot (1, \alpha, \alpha^2, \dots, \alpha^8))} \le \frac{32}{3\sqrt{15}} \approx 2.754, \quad \alpha\in [0, 1],
\end{aligned}
\end{equation*}
where $\beta = \frac{512}{212837625}$ and $$\Upsilon=(1370738, -172283, 394834, -98757, 164086, -98757, 394834, -172283, 1370738)\in \bbR^9, $$
and the maximum is achieved at $\alpha = 1$.
\end{proof}

We should observe that the leading eigenvalue $\mu_{\alpha, 1}$ actually is uniformly bounded from below if $\alpha\in(0,1)$. Therefore the decay of eigenvalues is even worse for $\alpha$ close to $1$. This somewhat is straightforward since $\alpha\sim 1$ means the loss of nonlinearity and the eigenvalues collapse to zeros except the leading one.

\subsection{Analytic activation functions}
For analytic activation functions such as $\texttt{Tanh}$ or $\texttt{Sigmoid}$, the Gram matrix is formed by 
\begin{equation*}
    \bmG_{i,j} = \int_{D} \sigma(\bmw_i \cdot \bmx - b_i) \sigma(\bmw_j \cdot \bmx - b_j) d\bmx .
\end{equation*}
Particularly, if the weights $|\bmw_i|\le A < \infty$, then the kernel function can be viewed as 
\begin{equation*}
    \calG(\bmx, \bmy) = \int_{D} \sigma(\bmx\cdot \bmz) \sigma(\bmy \cdot \bmz) d \bmw  ,\quad \bmx, \bmy \in [-A, A]\times [-1, 1],
\end{equation*}
where $\bmz \coloneqq (\bmw, -1)\in \bbR^{d+1}$. Since the kernel is analytic in both $\bmx$ and $\bmy$, the eigenvalues of the kernel are decaying faster than any polynomial rate~\cite{reade1983eigenvalues,reade1984eigenvalues}. Two examples in one dimension are provided in Figure~\ref{fig:spectrum:1D:smooth:act}.

\begin{figure}
    \centering	
        \begin{subfigure}[c]{0.420833\linewidth}
    \centering            \includegraphics[width=0.90985\textwidth]{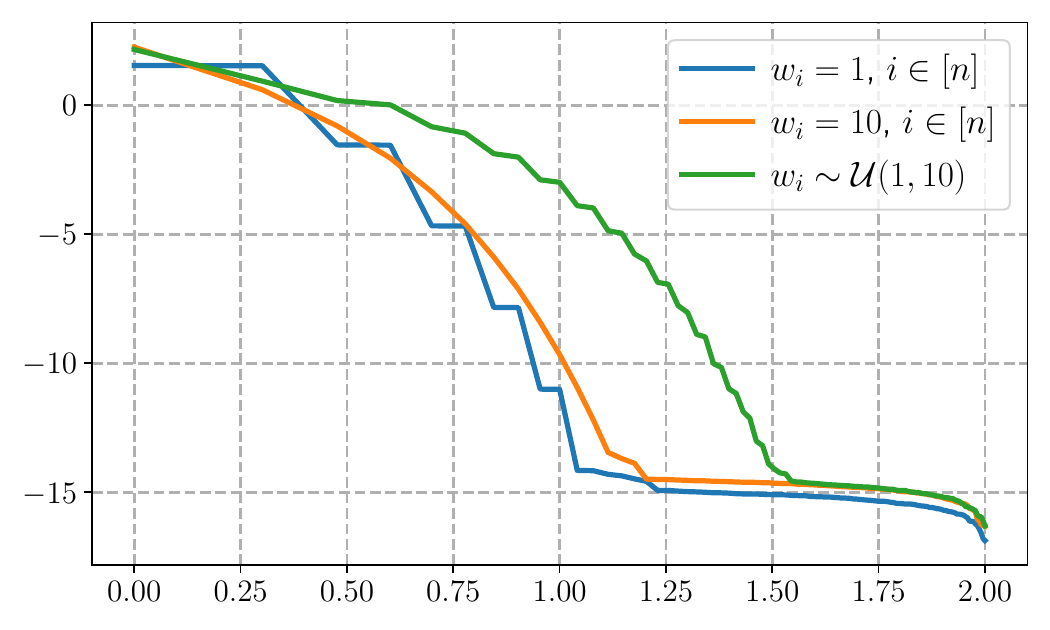}
    \subcaption{\texttt{Tanh}: $\sigma(x)=(e^{x}-e^{-x})/(e^{x}+e^{-x})$.}
    \end{subfigure}
    \hfill
    \begin{subfigure}[c]{0.420833\linewidth}
    \centering            \includegraphics[width=0.90985\textwidth]{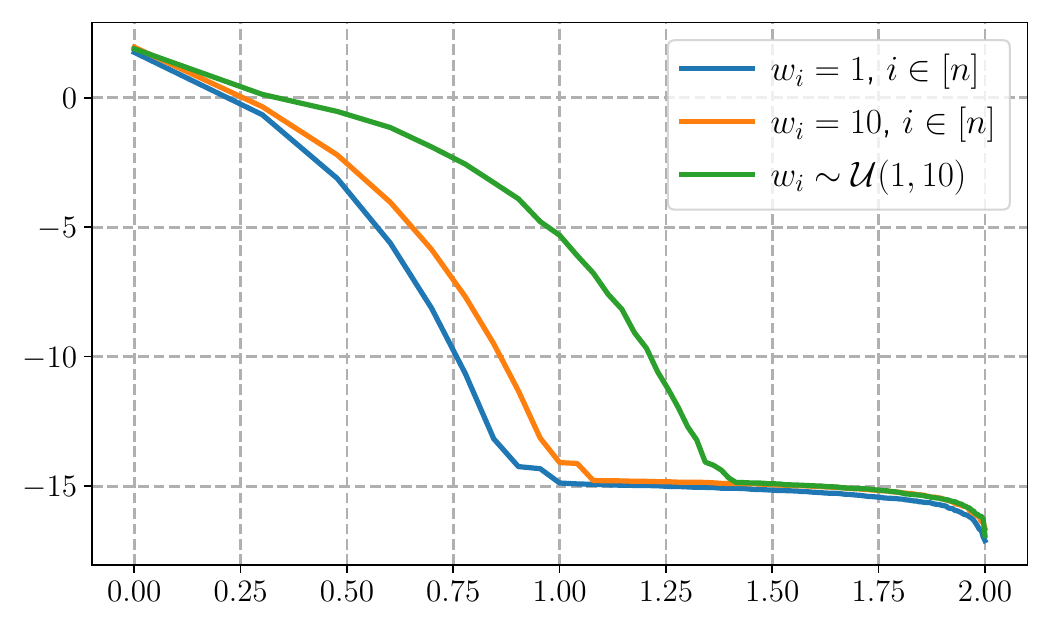}
    \subcaption{\texttt{Sigmoid}: $\sigma(x)=1/(1+e^{-x})$.}
    \end{subfigure}\hfill  
    \caption{Illustrations of the spectrum of Gram matrices in the one-dimensional case with $n=100$ for \texttt{Tanh} and \texttt{Sigmoid} activation functions. 
    The $x$-axis and $y$-axis correspond to $\log_{10}k$ and $\log_{10}\lambda_k$, respectively, for $k\in [n]$.
    Here, $(b_i)_{i=1}^{n}$ is evenly spaced in the interval $[-1,1]$ and $(w_i)_{i=1}^{n}$  is chosen from one of three cases: $w_i=1$ for all $i$, $w_i=10$ for all $i$, or $w_i$ randomly sampled from a uniform distribution $\calU(1,10)$.}
    	\label{fig:spectrum:1D:smooth:act}
\end{figure}

\end{document}

%% file: preamble.tex
\usepackage{lmodern}
\let\mathcal\relax
\DeclareMathAlphabet{\mathcal}{OMS}{cmsy}{m}{n}

\usepackage{float}
\usepackage{graphicx}
\graphicspath{{./figures/}}
\makeatletter
\renewcommand*\l@subsubsection[2]{{}{}{}} %
\makeatother
\usepackage{subcaption}

\usepackage[english]{babel}        

\makeatletter
\renewcommand\section{%
  \@startsection{section}{1}{\z@}%
    {-3.5ex \@plus -1ex \@minus -.2ex}%
    {2.3ex \@plus.2ex}%
    {\normalfont\Large\bfseries}%
}

\renewcommand\subsection{%
  \@startsection{subsection}{2}{\z@}%
    {-3.25ex \@plus -1ex \@minus -.2ex}%
    {1.5ex \@plus .2ex}%
    {\normalfont\large\bfseries}%
}

\renewcommand\subsubsection{%
  \@startsection{subsubsection}{3}{\z@}%
    {-3.25ex \@plus -1ex \@minus -.2ex}%
    {1ex \@plus .2ex}%
    {\normalfont\normalsize\bfseries}%
}

\AtBeginDocument{%
\newenvironment{proof}[1][\proofname]{\par
    \pushQED{\qed}%
    \normalfont \topsep6\p@\@plus6\p@\relax
    \trivlist
    \item\relax
    {\itshape
    #1\@addpunct{.}}\enskip
    \hspace\labelsep\ignorespaces
    }{%
    \popQED\endtrivlist\@endpefalse
    }
}%
\makeatother

\usepackage{amsmath}
\numberwithin{equation}{section}
\makeatletter
\renewcommand\@biblabel[1]{[#1]\hfill\hspace*{6pt}}
\renewcommand{\figurename}{Figure}
\def\ps@opening
  {%
      \def\@oddfoot{{%
            }}%
      \def\@evenfoot{{%
            }}%
  \let\@evenhead\relax
  \let\@oddhead\relax}

\def\authorandsep{\ifnum\arabic{myauthcount@totc}=\arabic{myauthcount}\ifnum\arabic{myauthcount@totc}=1\else 
\fi\else\fi}

\def\author#1{\gdef\@author{#1}}
\def\@author{}
\def\author{\@@author}%
\renewcommand{\@@author}[2][]{%
  \g@addto@macro\@author{%
    \refstepcounter{myauthcount}%
      \hspace*{0.001pt}{\authorandsep#2\ifx#1\@empty\else\textsuperscript{#1}\fi}\par\par
      \vspace*{3pt}
      }
                        }%

\def\@maketitle{%
  \let\footnote\thanks
  \clearemptydoublepage
    \checkoddpage\ifcpoddpage\setlength{\aboveskipchk}{-28.6pt}\else\setlength{\aboveskipchk}{-28.6pt}\fi
  \vspace*{\aboveskipchk}%
  \vspace{\dropfromtop}%
    \vspace*{11pt}\par%
  \hbox to \textwidth{%
  \parbox[t]{\titlepagewd}{%
    \ifx\@subtitle\@empty%
    \else%
    {\sffamilyfontcn\fontsize{14}{21}\selectfont\raggedright \@subtitle \par}%
    \vspace{7.5\p@}%
    \fi%
        {\fontsize{12bp}{14pt}\selectfont\bfseries\leftskip0pt plus1fill\rightskip0pt plus1fill {\@title} \par}
    \vspace{12\p@}%
        {\fontsize{10bp}{13.5}\selectfont
        \leftskip0pt plus1fill\rightskip0pt plus1fill \@author \par}%
    \vspace{4\p@}
    \vspace{-4\p@}
        {\fontsize{10bp}{12}\selectfont\leftskip0pt plus1fill\rightskip0pt plus1fill\@corresp \par}%
    \vspace{10\p@}
    \vspace*{10.5pt}
    \begingroup
    \parindent=0pt
    {\fontsize{9bp}{11}\selectfont\leftskip18pt\rightskip18pt\@abstract\par}
    \endgroup
    }%
  }
  \vspace{13.5\p@}%
  \ifx\@boxedtext\@empty\else\vspace*{12pt}\par\fi%
  \@boxedtext
  \vspace{12\p@ plus 6\p@ minus 6\p@}%
  \vspace{\extraspace}
  \vspace*{-18pt}
  }
\makeatother

\usepackage{hyperref}
\usepackage{bookmark} %
\hypersetup{
    colorlinks=true,
    breaklinks=true,
    plainpages=false,
    pdfpagelabels=true,
    unicode=true,
    bookmarksopen=true,
    bookmarksnumbered=true,
    bookmarksdepth=3,
    pdfpagelayout=SinglePage,
    pdffitwindow=true
}
\makeatletter

\def\fnum@table{Table~\thetable.}

\long\def\@tablecaption#1#2{%
  \begingroup%
  \centering
      \fontsize{10bp}{12pt}
      \selectfont%
      #1\enskip{#2\par}%
  \endgroup\vspace{\belowcaptionskip}}

\renewcommand\appendix{\setcounter{section}{0}%
  \setcounter{subsection}{0}%
  \def\thesection{\Alph{section}}
  \def\thesubsection{\Alph{section}.\arabic{subsection}}%
  \@addtoreset{equation}{section}%
  \gdef\theequation{\Alph{section}.\arabic{equation}}}
\makeatother

\usepackage{tabularx,multirow,array,booktabs}
\usepackage{colortbl}

\newtheorem{theorem}{Theorem}[section]
\newtheorem{lemma}[theorem]{Lemma}
\newtheorem{definition}[theorem]{Definition}

\newtheorem{remark}[theorem]{Remark} 
\newtheorem{corollary}[theorem]{Corollary}

\newcommand{\supp}{\operatorname{supp}}

\newcommand{\eps}{\varepsilon}

\DeclareMathOperator*{\argmin}{\arg\min}
\DeclareMathOperator*{\argmax}{\arg\max}

\usepackage{etoolbox}
\usepackage{mathtools}
\usepackage{dsfont} 
\usepackage{bm}
\newcommand{\myMF}[4]{\expandafter#1\csname#3#4\endcsname{{#2{#4}}}}
\newcommand{\myMFcmd}[4]{\expandafter#1\csname#3#4\endcsname{{#2{\csname#4\endcsname}}}}

\forcsvlist{\myMF{\newcommand}{\mathfrak}{frak}}{T,u,U,o,O,E,m}

\let\one\bbone
\forcsvlist{\myMF{\newcommand}{\mathbb}{bb}}{A,C,E,L,N,P,Q,R,S,Z}

\forcsvlist{\myMF{\newcommand}{\bm}{bm}}{a,b,c,d,e,f,g,h,i,j,k,l,m,n,o,p,q,r,s,t,u,v,w,x,y,z,A,B,C,D,E,F,G,H,I,J,K,L,M,N,O,P,Q,R,S,T,U,V,W,X,Y,Z}
\forcsvlist{\myMFcmd{\newcommand}{\bm}{bm}}{alpha,beta,gamma,delta,epsilon,zeta,eta,theta,iota,kappa,lambda,mu,nu,xi,omicron,pi,rho,sigma,tau,upsilon,phi,chi,psi,omega,Alpha,Beta,Gamma,Delta,Epsilon,Zeta,Eta,Theta,Iota,Kappa,Lambda,Mu,Nu,Xi,Omicron,Pi,Rho,Sigma,Tau,Upsilon,Phi,Chi,Psi,Omega}

\forcsvlist{\myMF{\newcommand}{\mathcal}{cal}}{A,B,C,D,E,F,G,H,I,J,K,L,M,N,O,P,Q,R,S,T,U,V,W,X,Y,Z}

\newcommand{\cA}{\mathcal A}

\newcommand{\cG}{\mathcal G} \newcommand{\cH}{\mathcal H}
 \newcommand{\cJ}{\mathcal J}
 \newcommand{\cL}{\mathcal L}
 
\newcommand{\cO}{\mathcal O}  
 \newcommand{\cR}{\mathcal R}
 
 \newcommand{\cV}{\mathcal V}

\forcsvlist{\myMF{\newcommand}{\mathscr}{scr}}{A,B,C,D,E,F,G,H,I,J,K,L,M,N,O,P,Q,R,S,T,U,V,W,X,Y,Z} 

\forcsvlist{\myMF{\newcommand}{\widehat}{hat}}{a,b,c,d,e,f,g,h,i,j,k,l,m,n,o,p,q,r,s,t,u,v,w,x,y,z,A,B,C,D,E,F,G,H,I,J,K,L,M,N,O,P,Q,R,S,T,U,V,W,X,Y,Z} 

\newcommand{\hatbm}[1]{\widehat{\bm{#1}}}
\forcsvlist{\myMF{\newcommand}{\hatbm}{hatbm}}{a,b,c,d,e,f,g,h,i,j,k,l,m,n,o,p,q,r,s,t,u,v,w,x,y,z,A,B,C,D,E,F,G,H,I,J,K,L,M,N,O,P,Q,R,S,T,U,V,W,X,Y,Z}

\forcsvlist{\myMF{\newcommand}{\widehtilde}{tilde}}{a,b,c,d,e,f,g,h,i,j,k,l,m,n,o,p,q,r,s,t,u,v,w,x,y,z,A,B,C,D,E,F,G,H,I,J,K,L,M,N,O,P,Q,R,S,T,U,V,W,X,Y,Z} 

\newcommand{\tildebm}[1]{\widetilde{\bm{#1}}}
\forcsvlist{\myMF{\newcommand}{\tildebm}{tildebm}}{a,b,c,d,e,f,g,h,i,j,k,l,m,n,o,p,q,r,s,t,u,v,w,x,y,z,A,B,C,D,E,F,G,H,I,J,K,L,M,N,O,P,Q,R,S,T,U,V,W,X,Y,Z}

\newcommand{\aver}[1]{\langle {#1} \rangle}

\let\tilde\widetilde
\let\hat\widehat
\let\epsilon\varepsilon
\let\eps\varepsilon
\let\subset\subseteq
\let\tn\textnormal

\let\cdots\customcdots
\let\dots\customcdots
\let\ldots\customcdots

\let\myforall\forall
\def\forall{{\myforall\, }}
\let\myexists\exists
\def\exists{{\myexists\, }}

\definecolor{mygray}{RGB}{230,230,230}
\definecolor{myorange}{HTML}{ff7f0e}